%% file: arxiv.tex
\documentclass{article}

\usepackage{microtype}
\usepackage{graphicx}
\usepackage{booktabs} 
\usepackage{amsmath,amssymb,manfnt} 
\usepackage{amsthm}          
\usepackage{enumitem}        
\usepackage{wrapfig}
\usepackage{subcaption}

\usepackage{adjustbox}
\usepackage[table]{xcolor}
\usepackage{makecell}
\usepackage{float}
\usepackage[labelfont=bf]{caption}

\usepackage[utf8]{inputenc}
\usepackage[T1]{fontenc}
\usepackage{lmodern}

\usepackage[preprint]{icml2025} 

\usepackage{amsmath}
\usepackage{amssymb}
\usepackage{mathtools}
\usepackage{amsthm}
\usepackage{soul}

\usepackage[capitalize,noabbrev]{cleveref}

\theoremstyle{plain}
\newtheorem{theorem}{Theorem}[section]
\newtheorem{proposition}[theorem]{Proposition}
\newtheorem{lemma}[theorem]{Lemma}

\theoremstyle{definition}
\newtheorem{definition}[theorem]{Definition}

\theoremstyle{remark}

\theoremstyle{definition}
\newtheorem*{proposition*}{Proposition}

\definecolor{LegendGreen}{HTML}{1B5E20} 
\definecolor{LegendBlue}{HTML}{1A237E}  
\definecolor{LegendRed}{HTML}{B71C1C}   

\newcommand{\tcr}[1]{{{\color{red}{#1}}}}
\newcommand{\tcb}[1]{{{\color{blue}{#1}}}}

\newcommand{\ind}[1]{\mathrm{ind}#1}

\newcommand{\ka}[1]{%
  \renewcommand{\arraystretch}{1.7}%
  \begin{tabular}[t]{@{}l@{}}#1\end{tabular}}

\newcommand{\grayrule}{%
  \arrayrulecolor{gray!55}\midrule\arrayrulecolor{black}}

\usepackage[textsize=tiny]{todonotes}
\icmltitlerunning{Intrinsic Task Symmetry Drives Generalization in Algorithmic Tasks}

\begin{document}

\twocolumn[
\icmltitle{Intrinsic Task Symmetry Drives Generalization in Algorithmic Tasks}

\begin{icmlauthorlist}
\icmlauthor{Hyeonbin Hwang}{kaist}
\icmlauthor{Yeachan Park}{sejong}
\end{icmlauthorlist}

\icmlaffiliation{kaist}{KAIST}
\icmlaffiliation{sejong}{Sejong University}
\icmlcorrespondingauthor{Yeachan Park}{ychpark@sejong.ac.kr}
\icmlkeywords{grokking}

\vskip 0.3in
]

\printAffiliationsAndNotice{} 

\begin{abstract} 
Grokking, the sudden transition from memorization to generalization, is characterized by the emergence of low-dimensional representations, yet the mechanism underlying this organization remains elusive. We propose that \textbf{intrinsic task symmetries} primarily drive grokking and shape the geometry of the model’s representation space. We identify a consistent three-stage training dynamic underlying grokking: (i) memorization, (ii) symmetry acquisition, and (iii) geometric organization. We show that generalization emerges during the symmetry acquisition phase, after which representations reorganize into a structured, task-aligned geometry. We validate this symmetry-driven account across diverse algorithmic domains, including algebraic, structural, and relational reasoning tasks. Building on these findings, we introduce a symmetry-based diagnostic that anticipates the onset of generalization and propose strategies to accelerate it. Together, our results establish intrinsic symmetry as the key factor enabling neural networks to move beyond memorization and achieve robust algorithmic reasoning.
\end{abstract}

\vspace{-1em}
\section{Introduction}

\textit{\textbf{What enables generalization in neural networks?}} This remains as the central open question in deep learning. State-of-the-art models show remarkable proficiency on tasks requiring intricate reasoning and abstraction, yet what they actually learn and
under what conditions those solutions extend beyond the
observed samples remains elusive. As a result, grokking~\citep{power2022grokkinggeneralizationoverfittingsmall} has emerged as a particularly revealing testbed for studying generalization, characterized by a delayed yet abrupt transition from memorization to generalization.

Prevailing explanations for grokking such as implicit biases~\citep{lyu2024dichotomy}, regularization norms~\citep{liu2023omnigrokgrokkingalgorithmicdata}, and circuit efficiency~\citep{varma2023explaininggrokkingcircuitefficiency} converge on the observation that generalization arises when representations collapse into low-dimensional, structured manifolds~\citep{nanda2023progress, zheng2024delays}. However, while \textit{optimization}-centric explanations like weight decay account for the \textit{tendency} toward compression, they alone cannot explain the \textit{specific form} of the resulting geometry. Even task-specific circuits in tasks like modular addition~\citep{nanda2023progress, zhong2023clockpizzastoriesmechanistic} cannot reliably provide a generic explanation as to \textbf{\textit{why}} the model chooses to solve the problem using such algorithm. 

In this paper, we propose that \textbf{intrinsic task symmetry is the key driver of generalization\footnote{Transition to perfect test accuracy, signifying the model's full recovery of the underlying algorithmic rule.}}, determining the specific shape of the learned geometry in algorithmic tasks. To validate this hypothesis, we extend our analysis beyond the standard modular arithmetic testbed to a diverse suite of algebraic, structural, and relational domains. 

By tracking the specific symmetries governing each of these tasks, we identify a consistent three-stage training dynamic: (i) Memorization, (ii) Symmetry Acquisition, and (iii) Geometric Organization. In doing so, our framework complements prior grokking phase characterizations~\citep{liu2022understandinggrokkingeffectivetheory, nanda2023progress} by providing a causal explanation for \textit{why} specific representations emerge from a data-centric view.

Specifically, our analysis reveals that generalization emerges alongside symmetry acquisition; we identify a \textit{soft threshold} of symmetry violation below which perfect test accuracy consistently emerges. This acquisition, coupled with low-rank constraints, leads to the specific geometric organization observed in the embedding space. Furthermore, we demonstrate the functional role of these findings by showing that symmetry- or geometry-prompting strategies can be used to accelerate the onset of generalization. In summary, our main contributions are:

\begin{figure}[h]
\begin{center}
\includegraphics[width=0.92\linewidth]{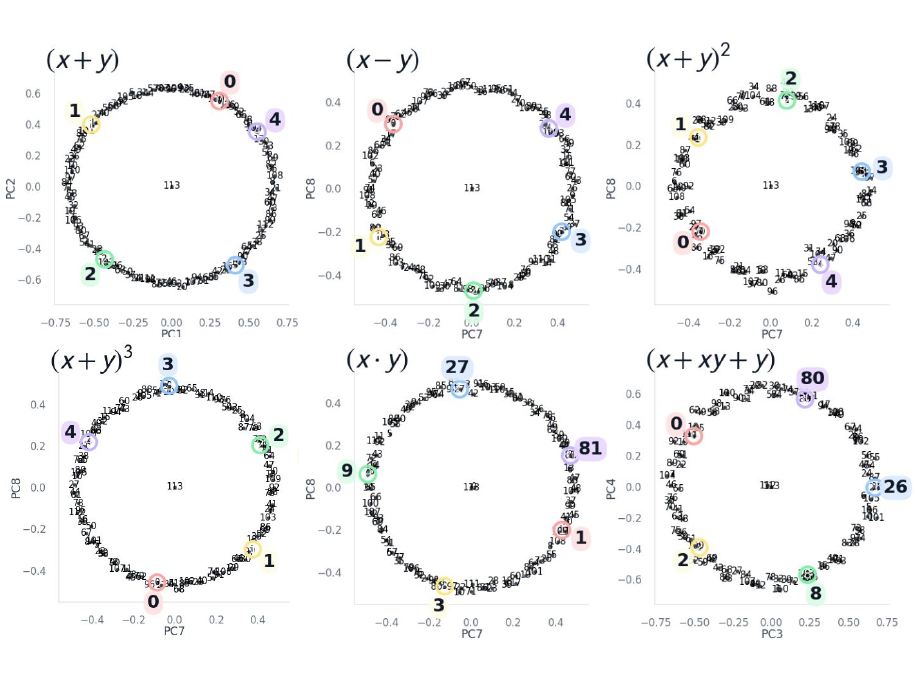}
\end{center}
\vspace{-0.9em}
\caption{Modular Arithmetics Embeddings}
\label{fig:modular_embedding}
\end{figure}

\begin{itemize}[leftmargin=1.5em, itemsep=0.2em]
\vspace{-0.7em}
\item \textbf{Intrinsic Task Symmetry Drives Generalization:} We establish intrinsic task symmetries as the primary driver of generalization in algorithmic tasks, and empirically investigate this mechanism across three domains: algebraic (modular arithmetic), structural (graph metric completion), and relational (comparison).

\item \textbf{A Three-Stage Framework for \textit{Grokking} Dynamics:} We propose a generic three-stage training dynamic for grokking: (i) \textit{Memorization}, (ii) \textit{Symmetry Acquisition}, and (iii) \textit{Geometric Organization}. We show that generalization emerges during the \textit{\textbf{Symmetry Acquisition}} phase, which can be tracked using a proposed criterion, and that intrinsic symmetry, together with low-rank representations, steers the geometric organization of the embedding space.

\item \textbf{Accelerating Grokking via Symmetry Alignment:} Leveraging these insights, we show that symmetry- and geometry-prompting strategies can accelerate grokking, providing evidence for a causal role of symmetry alignment in generalization.

\end{itemize}

\section{Related Works}

\textbf{Grokking}, which was first observed in modular arithmetic tasks~\citep{power2022grokkinggeneralizationoverfittingsmall}, has prompted extensive investigation into the mechanisms of delayed generalization. Prevailing explanations range from theoretical analyses of simple models~\citep{gromov2023grokkingmodulararithmetic,vzunkovivc2024grokking,lyu2024dichotomy,mohamadi2024why,gu2025progressmeasurestheoreticalinsights,levi2024grokking,nam2025solveLayerwiseLinear} to optimization-centric views that cite small weight norms~\citep{liu2023omnigrokgrokkingalgorithmicdata}, low-rank biases~\citep{junior2025grokking}, initialization regimes~\citep{kumar2024grokkingtransitionlazyrich}, circuit competition dynamics~\citep{merrill2023talecircuitsgrokkingcompetition, varma2023explaininggrokkingcircuitefficiency, huang2024unifiedviewgrokkingdouble}, or numerical instability~\cite{prieto2025grokkingAtNumerical}.

\begin{figure}[h]
\vspace{0.8em}
\begin{center}
    \includegraphics[width=0.98\linewidth]{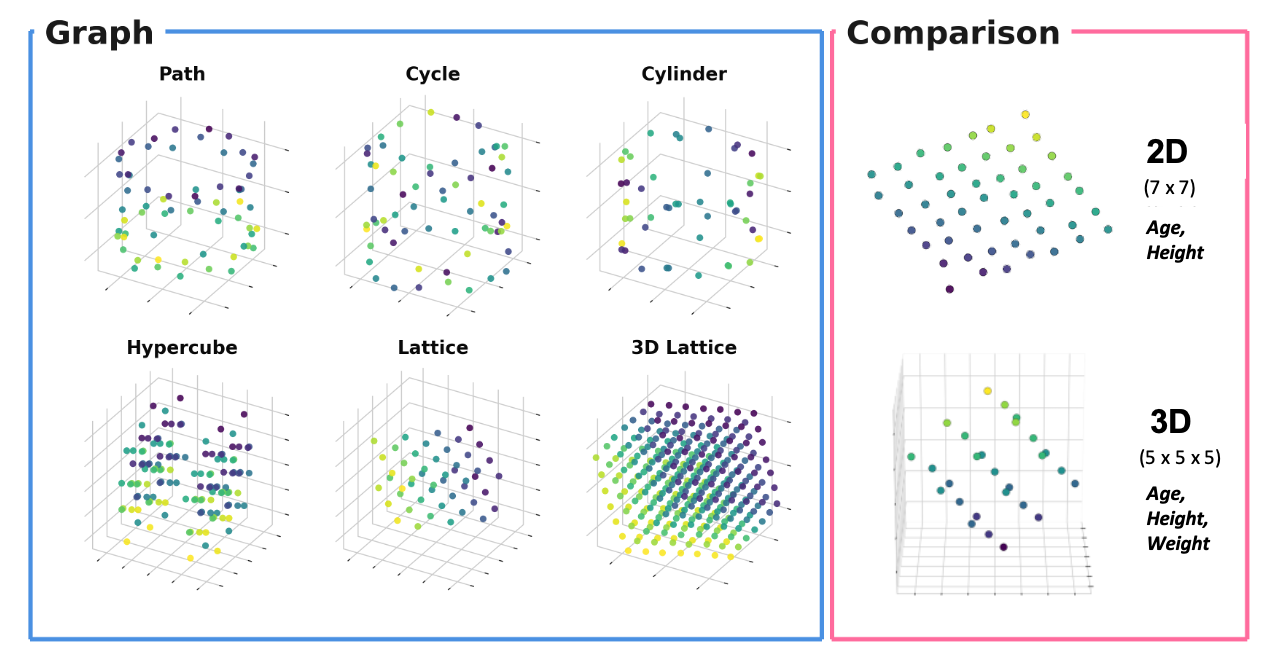}
\end{center}
\vspace{-0.5em}
\caption{Graph Metric Completion and Comparison Embeddings. }
\vspace{-1.3em}
\label{fig:graph_embedding}
\end{figure}

From a mechanistic perspective, many works clarified how models generalize, tying grokking to the emergence of highly structured representations. For example, \citet{liu2022understandinggrokkingeffectivetheory} and \citet{zheng2024delays} argue that a sharp reorganization of representational geometry acts as the primary signal of generalization. In the specific case of modular arithmetics, this geometry manifests as circular or helical structures in the embedding space, often interpreted through Fourier analysis~\citep{nanda2023progress, zhong2023clockpizzastoriesmechanistic}. Notably, these geometric mechanisms also appear in pre-trained language models~\citep{zhou2024pretrainedlargelanguagemodels, kantamneni2025languagemodelsusetrigonometry}.

While the emergence of these geometric representations is well-documented, the question of \textit{why} they take these specific forms remains largely unanswered, often attributed generically to weight decay.~\citep{nanda2023progress, zhu2024criticaldatasizelanguage} On the other, a more fundamental perspective is emerging, by paying particular attention on training-test distribution shift~\citep{liu2023omnigrokgrokkingalgorithmicdata, wang2024grokkedtransformersimplicitreasoners, gu2025progressmeasurestheoreticalinsights}. Notably,~\citet{chang2025characterizingpatternmatchinglimits} formalizes this as functional equivalence for predicting coverage in compositional generalization. There are also theoretical studies that link the learning of invariance to generalization bounds~\citep{xu2020neuralnetworksreasonabout, benton2020learninginvariancesneuralnetworks}. Our work advances this discourse by proposing that across diverse algorithmic domains, intrinsic task symmetry is the underlying invariance acting as the causal force that drives generalization. 

Lastly, parallel to mechanistic inquiries, other works focus on accelerating the onset of generalization. These include transfer learning from pretrained or weaker models~\citep{furuta2024towards, xu2025let, park2024accelerationgrokkinglearningarithmetic}, optimization techniques like \textit{GrokFast} that amplify slow-varying gradients~\citep{lee2024grokfastacceleratedgrokkingamplifying}, and the enforcement of inductive biases such as commutativity in modular arithmetics~\citep{tan2023understanding, park2024accelerationgrokkinglearningarithmetic}.

\section{General Setup}
\label{sec:setup}

\subsection{Modular Arithmetic}
\label{setup:modular}
 Following~\citet{power2022grokkinggeneralizationoverfittingsmall}, we first evaluate generalization on algorithmic tasks over the ring of integers modulo $p$, denoted as $\mathcal{X} = \mathbb{Z}_p = \{0, 1, \dots, p-1\}$. We consider a suite of six tasks ranging from simple linear groups to higher-order polynomials. Formally, given inputs $x, y \in \mathbb{Z}_p$, the network must predict the target $z$ defined by:
\begingroup
\setlength{\abovedisplayskip}{6pt}
\setlength{\belowdisplayskip}{2pt}
\setlength{\jot}{2pt}
\[
\begin{aligned}
T_1 &= (x+y)      & T_2 &= (x-y)      & T_3 &= (x+y)^2 \\
T_4 &= (x+y)^3    & T_5 &= x\times y  & T_6 &= (x+xy+y)
\end{aligned}
\]
\endgroup

\subsection{Graph Metric Completion}
\label{setup:graph}

\begin{wrapfigure}{r}{0.48\linewidth}
    \centering
    \fbox{
        \includegraphics[width=0.90\linewidth]{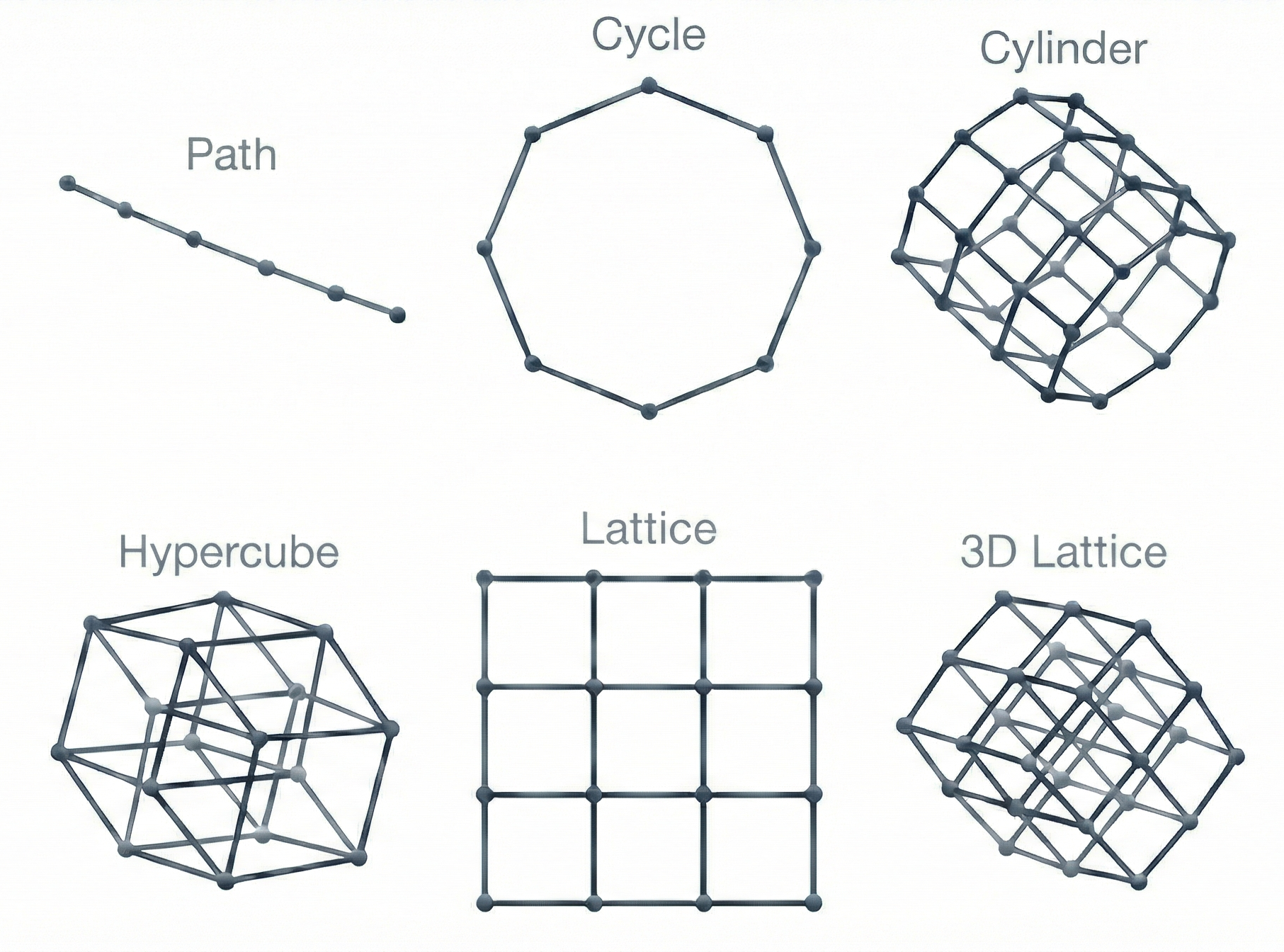}
    }
    \vspace{-10pt}
    \caption{Type of Graphs.}
    \vspace{-5pt}
    \label{fig:graph_gallery}
\end{wrapfigure}

While modular arithmetic has served as a popular testbed for grokking, these algebraic operations can be viewed geometrically as movements on a specific graph topology: the cycle graph $C_p$ of the ring $\mathbb{Z}_p$. In this light, we extend the setup from simple cycles to arbitrary graph structures with explicit geometries. 

We introduce the \textbf{Graph Metric Completion}\footnote{This task is the discrete analogue of the \textit{Euclidean Distance Matrix Completion} problem.} task, which aims to recover the global topological structure of the graph by inferring the full shortest-path distance matrix $D$. Specifically, given two vertices $u, v \in V$ in a graph $G=(V, E)$, the network is tasked with predicting the geodesic distance $d(u, v)$. In our experiments, we consider six distinct geometries: Path, Cycle, Cylinder, Hypercube, Lattice, 3D Lattice (Figure~\ref{fig:graph_gallery}).

\subsection{Comparison}
\label{subsec:comparision_task}

Now, we strip away the metric scaffolding provided in the graph task and investigate whether geometry still emerges under strictly \textit{local} supervision. We adopt the \textbf{Comparison Task}, where the model must infer global structure solely from binary pairwise relations ($x \succ y$). To probe the limits of geometric representation, we focus on transitive regimes and evaluate performance in both 2D and 3D attribute spaces.

\textbf{Problem setup.} We instantiate the comparison task in a form analogous to a small knowledge graph, following prior works~\citep{wang2024grokkedtransformersimplicitreasoners, allenzhu2024physicslanguagemodels31}. Each entity represents a person characterized by latent scalar attributes, with dimensionality two or three depending on the regime. Supervision is provided exclusively through relative order relations between pairs of entities, without access to absolute attribute values.

Let $\mathcal{X}$ denote the entity vocabulary and $\mathcal{R}$ the relation vocabulary:
\begingroup
\setlength{\abovedisplayskip}{6pt}
\setlength{\belowdisplayskip}{6pt}
\begin{align*}
\mathcal{X} &= \{\, e_i = (a_i, h_i, w_i) \mid a_i, h_i, w_i \in \{1, 2, \dots, m\} \,\}, \\
\mathcal{R} &= \{\, r^{x}_{\bowtie} \mid x \in \{\text{age}, \text{height}, \text{weight}\},\
\bowtie \in \{<, =, >\} \,\}.
\end{align*}

\endgroup

Then, each training example is a triple $(e_i, r, e_j)$ with label $y \in \{0,1\}$, indicating whether the relation $r$ holds between $e_i$ and $e_j$. In natural language, this corresponds to sentences such as "John is older than Mary" or "Alice is taller than Bob".

\begin{figure}[H]
    \vspace{0.5em}
    \centering
    \includegraphics[height=0.28\textheight]{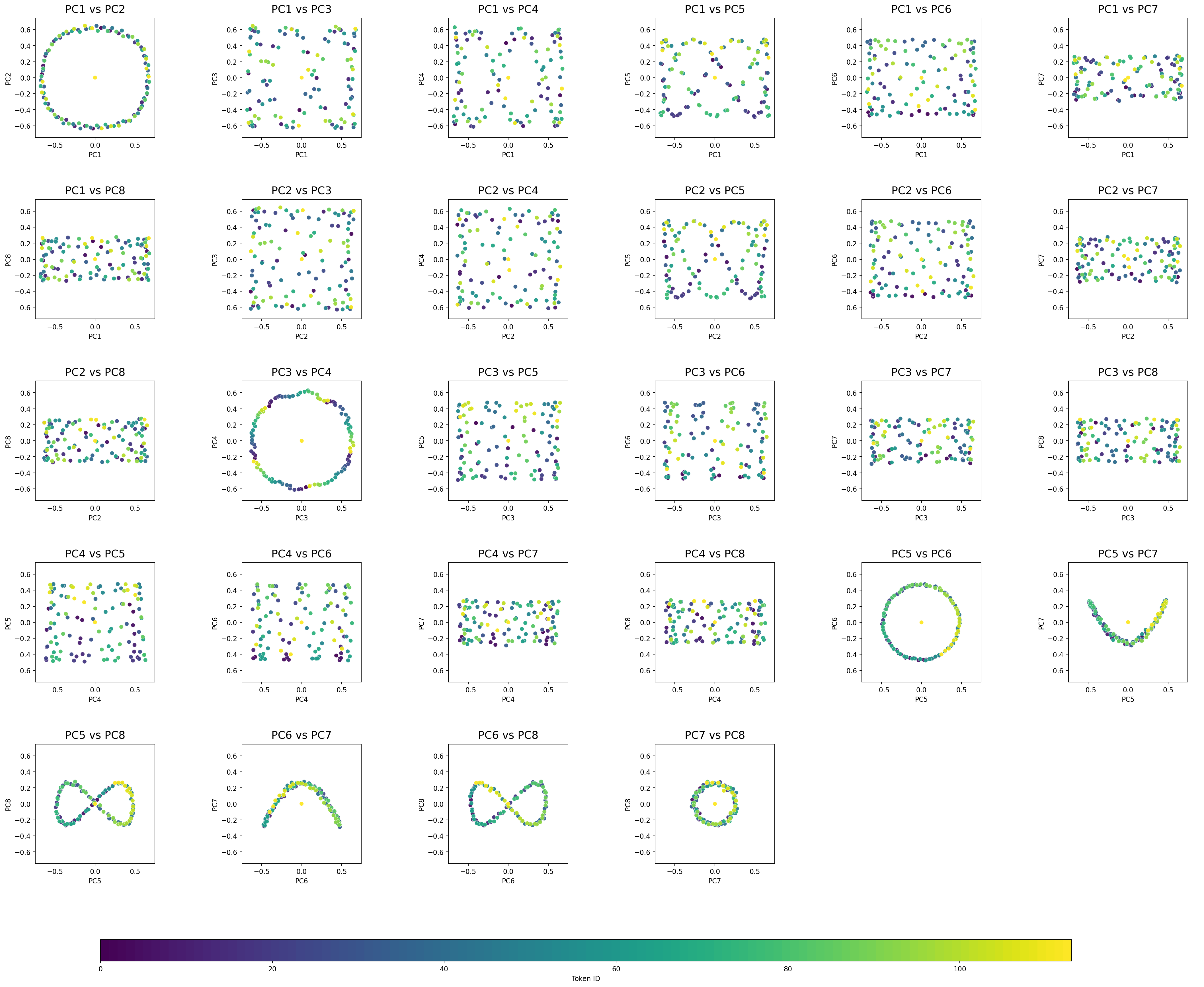}
    \caption{Full PCA visualization of modular addition embeddings.}
    \vspace{-1em}
\end{figure}

\section{Motivation: Structured Geometry in Embedding Space}
\label{sec:motivation}



Across the tasks introduced above, we analyze learned token embeddings using principal component analysis (PCA) to examine their representations. Projecting the leading principal components reveals an interesting pattern: after models achieve generalization, their embeddings consistently display a structured, low-dimensional geometric organization.

For modular arithmetic tasks, this geometric structure is particularly clear. As shown in Figure~\ref{fig:modular_embedding}, the leading two principal components form an approximately circular pattern, consistent with prior observations in modular addition \cite{nanda2023progress, kantamneni2025languagemodelsusetrigonometry}. Extending to higher-dimensional PC subspaces reveals a helical organization, where each PC pair forms a circular projection with a systematic phase shift. We include a more detailed analysis of this structure in Appendix~\ref{app:sym_modular}.

Similar pattern also appears in graph completion and comparison tasks. The learned embeddings recover geometric structures that reflect the relational or ordering properties of the task, as illustrated in Figure~\ref{fig:graph_embedding}. These observations suggest that task-aligned geometric structure is a generic property of embeddings that enable algorithmic generalization. We provide additional PCA visualizations across tasks in Appendix~\ref{Appendix:C}.

Now, this connection between embedding geometry and generalization has also been observed in prior work~\citep{liu2022understandinggrokkingeffectivetheory}. However, while such geometric structure helps characterize how generalization manifests in learned representations, it does not by itself explain why these structures arise. This raises a central question: \textbf{\textit{what drives the emergence of structured embedding geometry?}} In Section~\ref{sec:sym_driven_geo}, we argue that intrinsic task symmetries provide a principled mechanism underlying this phenomenon.

\begin{figure*}[h]
\begin{center}
\includegraphics[width=0.98\linewidth]{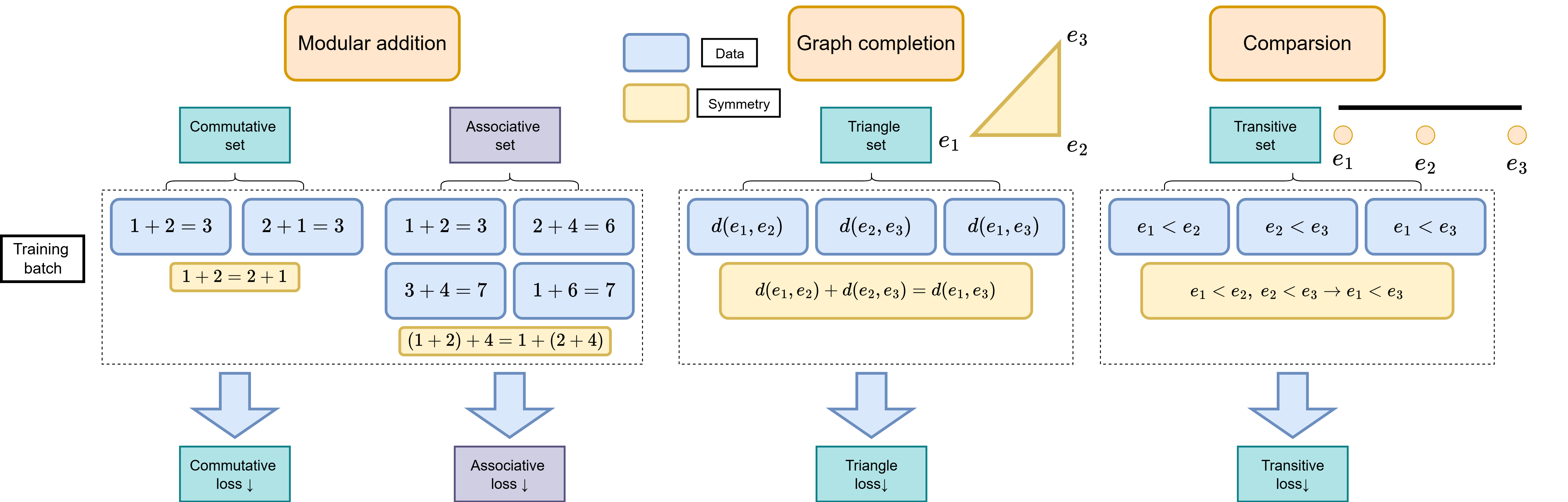}
\end{center}
\caption{Symmetric patterns within a dataset lead the network to learn intrinsic symmetry.}
\label{fig:symm}
\end{figure*}

\section{Intrinsic Symmetries in Algorithmic Tasks}\label{sec:hidden_sym}

To address the question raised above, we turn to a structural property shared across these algorithmic tasks: \emph{intrinsic task symmetry}. We argue that such symmetries impose fundamental constraints on the underlying rule governing the data, and that acquiring these invariances is a necessary condition for true algorithmic generalization. 

To this end, we first identify the intrinsic symmetries present in each domain (illustrated in \cref{fig:symm}) and formalize the invariance structure that the model would implicitly internalize in order to generalize beyond the training distribution.

\paragraph{Moodular Arithmetics}
For modular addition, there are two symmetries, commutativity and associativity:  
\begingroup
\setlength{\abovedisplayskip}{6pt}
\setlength{\belowdisplayskip}{6pt}
\begin{align*}
\text{Commutativity: } &\; x+y = y+x, \\[3pt]
\text{Associativity: } &\; x+(y+z) = (x+y)+z.
\end{align*}
\endgroup
These two properties are basic properties in arithmetic operations. We also note that other operations have similar symmetries which we present in Table~\ref{tab:sym_modular}. 

\begin{figure*}[h]
\begin{center}
\includegraphics[width=0.95\linewidth]{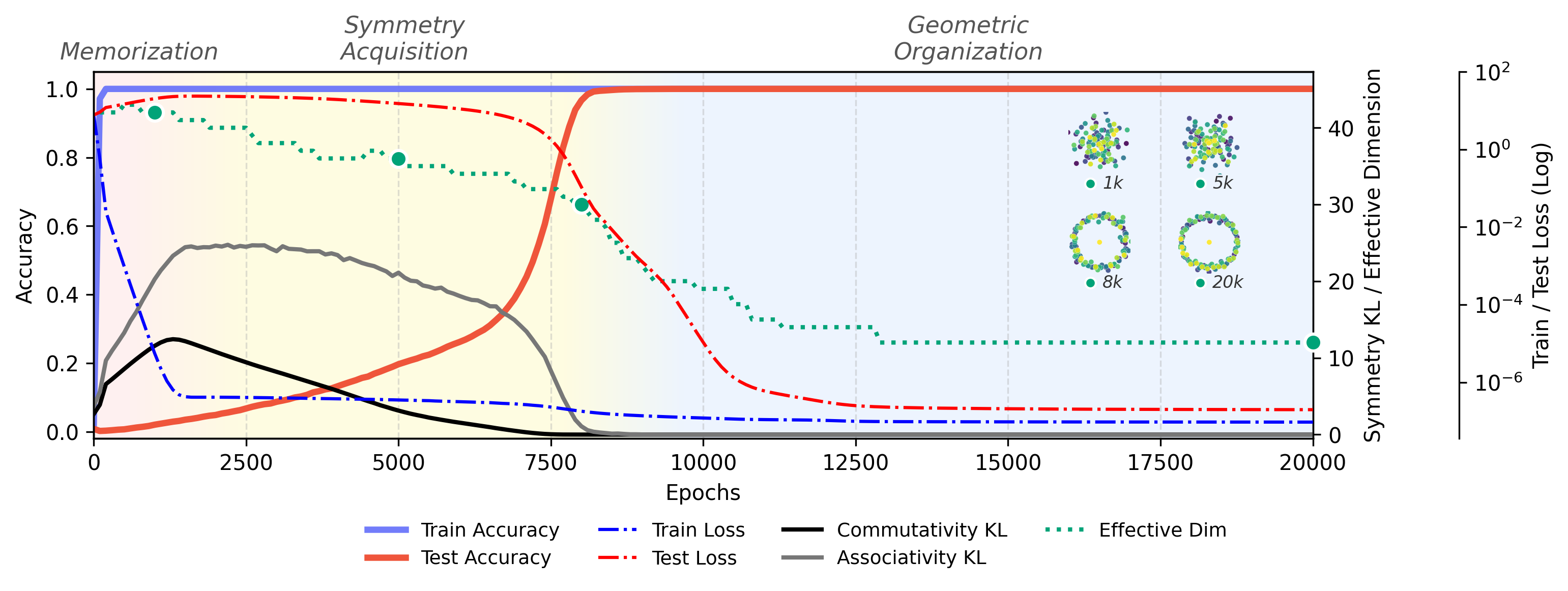}
\end{center}
\vspace{-1.0em}
\caption{Three-stage dynamic of \textit{grokking} for modular addition task. \textbf{(1) Memorization}: The model first \textit{memorizes} as it quickly reaches perfect train accuracy. \textbf{(2) Symmetry Acquisition} Then with the symmetry criterion decreasing to certain threshold, perfect test accuracy is reached. \textbf{(3) Geometric Organization} Finally, as the representation collapses into lower-dimension, embedding forms a geometric organization.}
\label{fig:three_stage}
\end{figure*}

\paragraph{Graph Metric Completion}
In graph metric completion task, there exists the symmetry of the triangle equality. Let $e_1,e_2,e_3$ be vertices of the graph. If $e_2$ is placed on the shortest path between $e_1$ and $e_3$, then the following triangle equality holds:
\[ \text{Triangular symmetry:  } d(e_1,e_2) + d(e_2,e_3) = d(e_1,e_3). \]

\paragraph{Comparison}
In comparison task, there exists the transitivity symmetry. Let $e_1,e_2,e_3$ be token entities in the comparison tasks. Suppose  we have $e_1 < e_2$ and $e_2 < e_3$ in the training batch. Then by transitivity, we also have $e_1 < e_3$ and we call it transitive symmetry. 
\[
\text{Transitivity :  } e_1 < e_2, \quad e_2 < e_3 \quad \to\quad e_1 < e_3.
\]

\begin{table}[t]
\fontsize{8.5}{10}\selectfont
\centering
\caption{Symmetries in various modular operations. \\ $x^+$ denotes $x^+=x+1$.}
\setlength{\tabcolsep}{2pt} 
\renewcommand{\arraystretch}{0.8} 
\begin{adjustbox}{max width=2.0\textwidth}
\begin{tabular}{@{} ccc @{}}
\toprule
\textbf{Task} & \textbf{commutativity proxy} & \textbf{associativity proxy} \\
\midrule
$x+y$
 & \ka{$x + y = y +x$ 
  } 
 & $(x+y)+z = x+(y+z)$\\
 \grayrule
$x-y$
 & \ka{$x-(-y) =  y - (-x)$  } 
 & \ka{ 
 $(x-y)-z $\\$= x-(y-(-z))$ }\\
 \grayrule
$(x+ y)^{\alpha}$
 & \ka{$ (x+ y)^{\alpha} = (y+ x)^{\alpha}  $ 
  }
 & \ka{ $ ( (x+ y)+z)^{\alpha}$ \\= $( x+ (y+z))^{\alpha} $} \\
\grayrule
$x \times y$
 & \ka{$x \times y=y \times x $ \\ }
 & \ka{$(x\times y)\times z$ \\ $=x\times (y\times z) $} \\
\grayrule
\ka{ $x^{}+xy+y^{}$ }
 & \ka{ $x^+\times y^+-1$ \\ $ = y^+ \times x^+-1$  }  & \ka{ $\left(x^+\times y^+ \right)\times  z^+-1$ \\ $ = x^+ \times \left( y^+ +z^+ \right)-1$  } 
\\
\bottomrule
\end{tabular}
\end{adjustbox}
\label{tab:sym_modular}
\vspace{-1.0em}
\end{table}

\section{Symmetry-Driven Generalization}
\label{sec:sym_driven_gen}

Having identified the intrinsic symmetries underlying each task, we now examine their functional role during training.

\subsection{Three-Stage Training Dynamics}

A common intuition is that once the network perfectly fits the training data, there is little left to learn and training effectively terminates. However, when intrinsic symmetries are present, training can continue beyond memorization: the model can reduce symmetry violations while maintaining near-zero training loss.

For example, in modular addition, if the training batch contains commutative pairs or associative tuples, the network is exposed to algebraic symmetries of addition. By repeatedly encountering symmetry patterns, the network can implicitly learn these symmetries beyond memorizing individual samples. Analogous symmetry constraints arise across the remaining domains (Figure~\ref{fig:symm}), suggesting a shared mechanism underlying generalization.

Based on these observations, we propose that optimization for \textit{grokking} progresses through three distinct phases: (i) data memorization, (ii) intrinsic symmetry acquisition, and (iii) geometric organization. Figure~\ref{fig:three_stage} illustrates this three-stage training dynamic.

During the data memorization stage, the network fits the observed training samples by minimizing the training loss. Even after achieving the perfect training accuracy, this process may continue by reducing remaining training loss without yielding generalization. 

Following sufficient memorization, training enters the symmetry acquisition stage. As symmetry violations decrease, quantified by the KL divergence between predictions on symmetry-related input pairs,\footnote{See Appendix~\ref{Appendix:training_dynamics} for details.} the network internalizes the intrinsic task symmetries, leading to improved performance on unseen data. Once these constraints are sufficiently satisfied, generalization emerges.

With symmetry constraints enforced, training transitions into the geometric organization phase. At this stage, weight decay exerts the dominant influence on optimization, favoring low-norm solutions among symmetry-consistent representations. This induces low-rank, low-dimensional structure in the embedding space, resulting in the organized geometries observed in Section~\ref{sec:motivation}.\footnote{Although geometric organization can arise near the completion of symmetry acquisition, in most cases it becomes visually apparent only after the effective embedding dimensionality has sufficiently decreased.}

Training dynamics for additional algorithmic tasks are provided in Appendix~\ref{Appendix:training_dynamics}. While the exact timing and sharpness of phase transitions vary across domains, we consistently observe that generalization aligns with symmetry acquisition rather than continued memorization, supporting the view that symmetry learning is mechanistically distinct from overfitting.

\begin{figure*}\
    \centering
    \hspace*{0.01\linewidth}
    \includegraphics[width=0.85\linewidth]{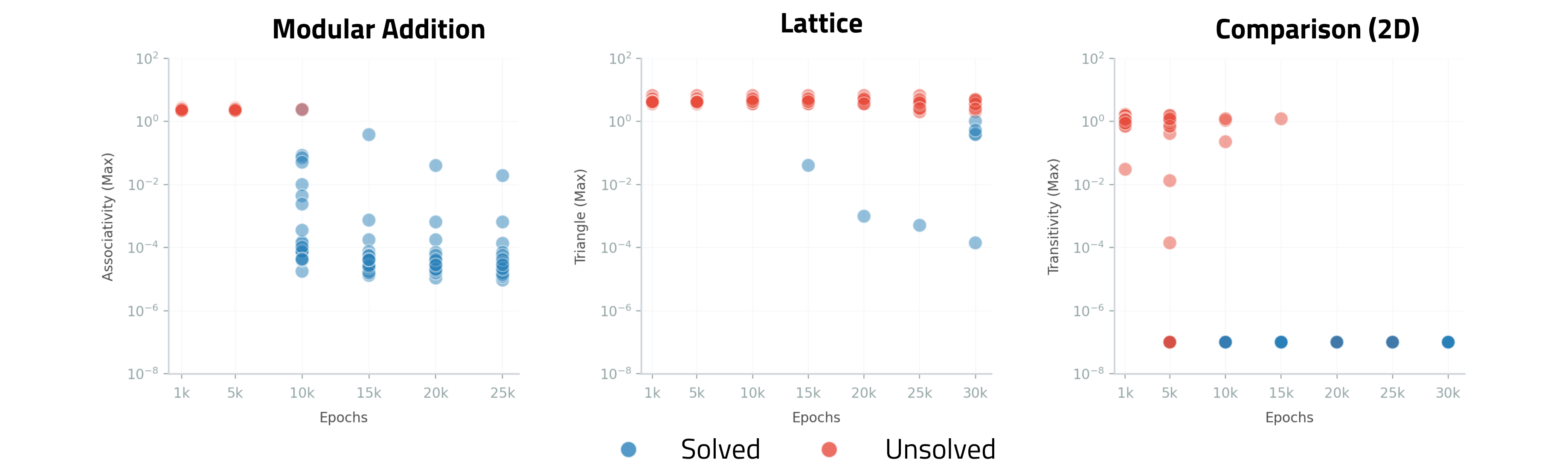}
    \small  
    \vspace{0.4em}
\caption{Symmetry violation metric vs. test accuracy across different seeds. 
A threshold in symmetry violation separates \tcb{\textbf{Solved}} from \tcr{\textbf{Unsolved}} runs, supporting symmetry alignment as a predictor of generalization.}
\label{fig:sym_metric}
\end{figure*}

\subsection{Symmetry-Based Criterion}

Our three-stage framework suggests a concrete, testable prediction: if symmetry acquisition drives generalization, then reductions in symmetry violation should reliably coincide with the onset of test accuracy. We therefore examine whether the symmetry metric defined above can serve as a predictive criterion for generalization.

Figure~\ref{fig:sym_metric} summarizes results across 20 independent training runs per task. Each point corresponds to a single run evaluated at a specific epoch (x-axis), with the y-axis showing the corresponding symmetry violation metric on a logarithmic scale. Points are colored according to whether the run ultimately achieves perfect test accuracy (blue, \textit{solved}) or yet fails to generalize (red, \textit{unsolved}). 

Across tasks, we observe a consistent threshold relationship: once a run’s symmetry violation drops below a characteristic level, it reliably reaches perfect test accuracy, while runs that remain above this level do not. Although the transition is not perfectly sharp, the symmetry metric separates solved from unsolved checkpoints across seeds, supporting symmetry alignment as a predictive criterion.

\begin{figure}
    \centering
    \fbox{%
        \includegraphics[width=0.95\linewidth]{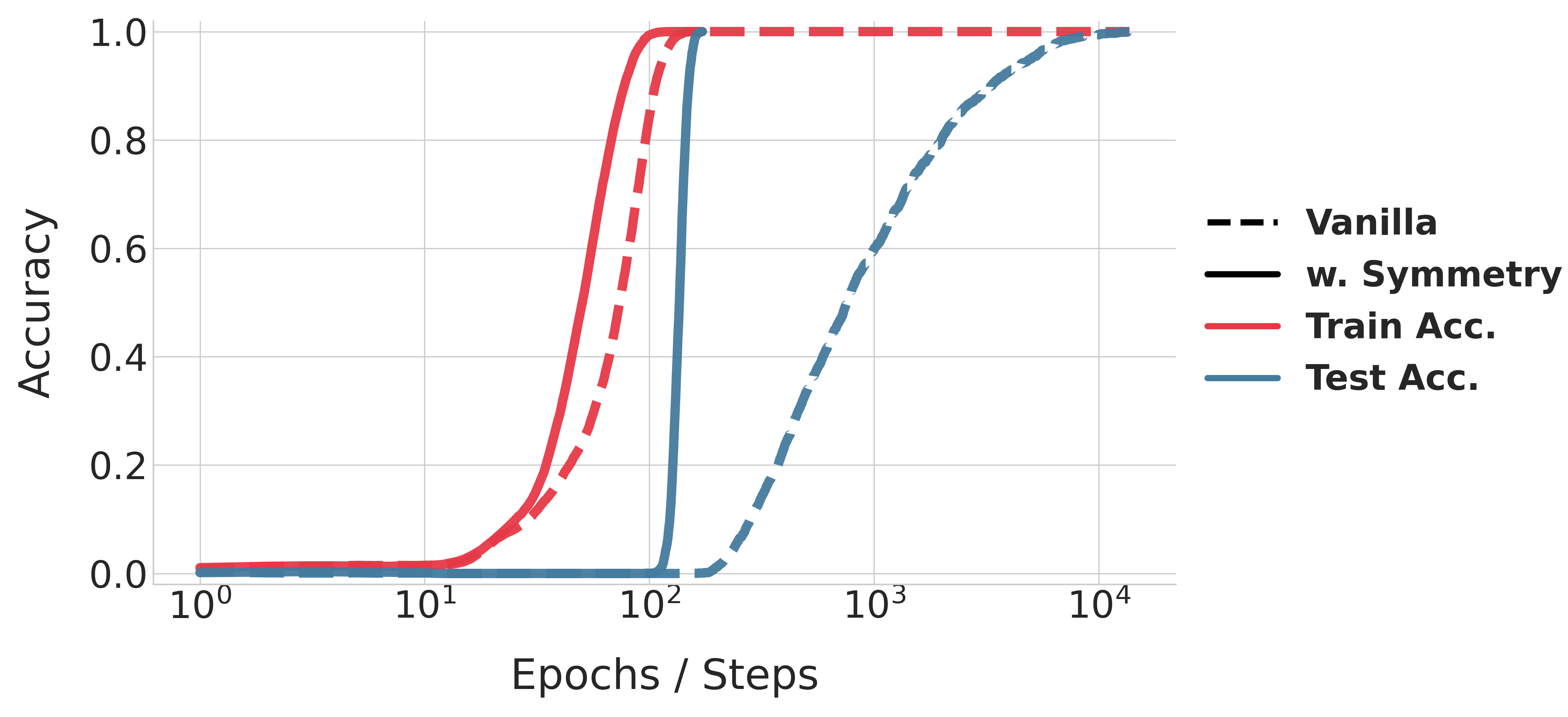}
    }
    \small  
    \caption{We train a two-layer MLP on modular addition using an 80\% training fraction, \textbf{\textit{without}} weight decay.}
    \label{fig:wo_wd}
\end{figure}

\subsection{Discussion}

\textbf{Role of weight decay.} One may argue that optimization biases, such as weight decay, are the primary drivers of generalization, with intrinsic symmetries playing only a secondary role. In particular,~\citet{nanda2023progress} decompose grokking into memorization, circuit formation, and cleanup, attributing circuit formation largely to weight decay selecting low-norm solutions. To examine this possibility, we conduct an additional experiment on modular addition \textbf{\textit{without}} weight decay.  

As shown in \cref{fig:wo_wd}, even in the absence of explicit regularization, the model reliably converges to perfect test accuracy \textbf{\textit{(dashed line)}}, as observed by previous works.~\citep{kumar2024grokkingtransitionlazyrich, prieto2025grokkingAtNumerical}. (We defer discussion of the bold line to Section~\ref{subsec:accl_disc}.) This indicates that weight decay is not a necessary condition for generalization. From our perspective, weight decay plays a supplementary role, primarily stabilizing and simplifying representations during the geometric organization stage. 

\section{Symmetry-Driven Geometric Organization}
\label{sec:sym_driven_geo}

If symmetry acquisition governs the onset of generalization, a natural next question is: 
\emph{what structural consequences does symmetry impose on the learned representation?}
Previously,  we showed that reductions in symmetry violation reliably coincide with generalization. We now examine how satisfying intrinsic task symmetries shapes the geometry of the embedding space. 
\bigskip
\bigskip

Intrinsic symmetries impose algebraic constraints on representations. 
Once internalized, these constraints reduce the effective degrees of freedom of the model’s solution space. 
In algorithmic tasks governed by strict rules, the admissible representations are therefore confined to a structured, lower-dimensional subset.

This perspective aligns with the manifold hypothesis \citep{fefferman2016testing}, which posits that meaningful representations lie on low-dimensional manifolds. 
In our setting, intrinsic symmetry provides the structural mechanism enforcing this dimensional restriction. 
Under mild optimization bias such as weight decay or implicit regularization \cite{barrett2021implicit}, the network preferentially selects low-norm solutions within this constrained space, giving rise to organized geometric structure.

For example, in modular arithmetic, whose intrinsic symmetries are commutativity and associativity, the learned embeddings exhibit structure consistent with an abelian group representation. 
Empirically, this appears as a closed helical geometry in principal component space, as shown in \cref{sec:motivation}. 
We now provide a theoretical account of why such symmetry necessarily induces this geometric structure.\footnote{Detailed proofs are provided in Appendix~\ref{app:theory}.}

\subsection{Theoretical observation }
\subsubsection{Modular arithmetic}

First, we formally verify that the symmetries induce the structured embedding geometry in modular addition task. We consider $\mathcal{X}=\mathbb{Z}_p$ is continuously embedded into the interval connecting endpoints : $\mathbb{Z}_p \subset \mathcal{M} =  [0, p] / \{ 0\sim p\}$. Note that the modular addition in $\mathbb{Z}_p$ is preserved in $\mathcal{M}$. If $\mathcal{M}$ has certain geometry, $\mathbb{Z}_p$ also follows the geometry. If $\mathcal{M}$ has commutativity and associativity, we can say that $\mathcal{M}$ has abelian group structure and $\mathcal{M}$ should be isomorphic to the closed helix in high-dimensional torus \footnote{By definition, $\mathcal{M}$ should have identity and invertible conditions to be an abelian group. Such properties are simple and easy for the network to learn, so we skip the conditions.}.

\begin{proposition}
\label{prop:abelian}
Let $\mathcal{M}$ be an one-dimensional compact abelian topological group and $\mathbb{Z}_p$ is continuously embedded in $ \mathcal{M}$. Then $\mathcal{M}$ is isomorphic to $S^1$ embedded in high-dimensional torus $\mathbb{T}^D$ (closed helix) 
\[ \Phi: \mathcal{M} \cong  S^1 \hookrightarrow \mathbb{T}^D.\]
Hence $\Phi(\mathbb{Z}_p) \hookrightarrow \Phi(\mathcal{M})$ is also confined in the closed helix. 
\end{proposition}
Using Proposition \ref{prop:abelian}, since $\mathcal{M}$ is compact abelian group (satisfying commutativity and associativity), $\mathcal{M}$ is isomorphic to the closed helix. Here, \textit{isomorphism} refers to a structure-preserving bijection (homeomorphism in topological sense). 
Consequently, we theoretically observe that if the network sufficiently learns the commutativity and associativity of modular addition, its embeddings naturally organize along a closed helical structure in the representation space. This observation further suggests that other modular operations exhibiting similar symmetries give rise to analogous geometric structures (see Appendix~\ref{app:sym_modular}).

\subsection{Graph metric completion and comparison tasks}

We also provide a partial theoretical analysis of the geometric organization observed in graph metric and comparison tasks. When the embedding space is assumed to be $\mathbb{R}^d$ or $\mathbb{T}^d$ for small $d \in \mathbb{N}$, it is possible to obtain partial explanations for the emergence of structured geometry.

Let $\mathcal{X} = \{v_1, \dots, v_n\}$ be a set of vertices of a graph with an embedding $\Phi : \mathcal{X} \to \mathbb{R}$ into a low-dimensional space. If $\Phi(\cdot)$ incorporates partial distance information (memorization of training data), intrinsic symmetries within the task, and low embedding dimension, then the image $\Phi(\mathcal{X})$ exhibits an organized geometric structure. The following proposition formalizes this intuition for the case of Path or Cycle graphs (Figure \ref{fig:graph_gallery}). If $\mathcal{X}$ is embedded into $\mathbb{R}$ preserving triangular symmetry, the resulting embedding recovers the graph structure.

\begin{proposition}\label{prop:theory_graph}
   Let $ \mathcal{X}=\{v_1 , \dots, v_n\}$ be a one-dimensional Path or Cycle graph (see Figure \ref{fig:graph_gallery}) with $n$ vertices ($n \ge 3$) and $\mathcal{Z} \in \{\mathbb{R}, \mathbb{T}\}$ be a embedding space. 
    Let $\Phi : G \to \mathcal{Z}$ be embedding with partial distance information
    \begin{align*}
        d(\Phi(v_{i}) ,\Phi(v_{i+1}))&=1, \; i=1,\dots n-1, \\
        d(\Phi(v_{1}) ,\Phi(v_{n}))&=1, \; \text{ if $\mathcal{X}$ is Cycle. } 
    \end{align*}
    If $\Phi(\cdot)$ preserve the triangular symmetry, then $\Phi(\mathcal{X})$ has its graph structure in embedding space  $\mathbb{R}$.
\end{proposition}
Since other graphs in Figure \ref{fig:graph_gallery} are considered as a composition of Path and Cycle, we can expect similar embedding structure.

An analogous result holds for the comparison task. The following proposition states that if $\mathcal{X}$ is embedded into $\mathbb{R}$ while preserving transitivity, the resulting embedding exhibits a grid structure.
\begin{proposition}\label{prop:theory_comparison}
    Let $\mathcal{X}=\{v_1 , \dots, v_n\}$ be a set of nodes in 1D comparison task (see \cref{subsec:comparision_task})  with $n \ge 3$ nodes. Let $\Phi : \mathcal{X} \to \mathbb{R}$ be embedding with partial relation information
    \[ \Phi(v_{i}) < \Phi(v_{i+1}), \; i=1,\dots n-1. \]
    If $\Phi(\cdot)$ preserve the transitivity, then $\Phi(\mathcal{X})$ has a grid structure in embedding space $\mathbb{R}$. 
\end{proposition}

\begin{figure*}\
    \centering
    \hspace*{0.01\linewidth}
    \includegraphics[width=0.92\linewidth]{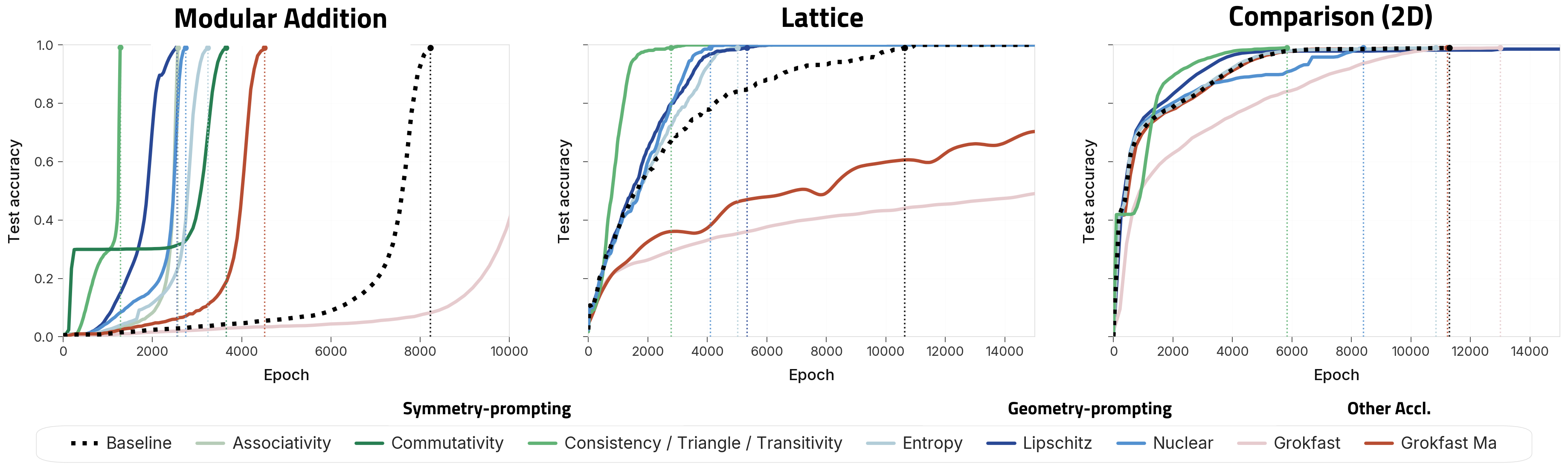}
    \small
\caption{Comparison of convergence speed on Modular Addition, Graph Completion (Lattice), and Comparsion (2D) task with
\textbf{\textcolor{LegendBlue}{Geometry-promoting priors}} (\textit{nuclear}, \textit{entropy}, \textit{Lipschitz});
\textbf{\textcolor{LegendGreen}{Symmetry-prompting losses}} (\textit{commutativity}/\textit{associativity}/etc);
\textbf{\textcolor{LegendRed}{Optimizer-based accelerants}} (\textit{GrokFast}).
Markers denote the first point where test accuracy reaches 100\%. \textit{Consistency} refers to using both Commutativity and Associativity.} 
\vspace{-1em}
\label{fig:faster}
\end{figure*}

Taken together, these observations suggest that geometric organization can be explained through the interplay of intrinsic task symmetries, sufficient memorization of partial information, and low-dimensional embedding constraints.

\section{Accelerating Generalization via Symmetry}
\label{sec:acceleration}
\subsection{Setup}
\label{sec:exp_setup}

Previously, we observed that generalization consistently coincides with the symmetry acquisition phase, and that geometric organization follows as representations contract to symmetry-consistent solutions.
These observations suggest that symmetry alignment plays a crucial role in determining the timing of the grokking transition.

To examine this mechanism more directly, we now intervene on training by introducing auxiliary objectives that explicitly promote intrinsic task symmetries. \textbf{If symmetry acquisition governs the onset of generalization, then systematically encouraging symmetry should shift the time-to-generalization across tasks and seeds.}
Such shifts would provide evidence consistent with a causal role of symmetry alignment in grokking.

Concretely, we consider (i) \emph{symmetry-prompting} losses that penalize violations of intrinsic task symmetries and (ii) \emph{geometry-promoting} priors that encourage low-rank structure in the learned representations.
The total objective is defined as
\[
\mathcal{L}_{\text{total}}
=
\mathcal{L}_{\text{CE}}
+
\alpha \|\theta\|^2
+
\lambda \mathcal{L}_{\text{aux}},
\]
where $\mathcal{L}_{\text{aux}}$ denotes the auxiliary term and $\alpha,\lambda$ are hyperparameters.
Full configuration details are provided in Appendix~\ref{app:config}.

\subsection{Symmetry-prompting loss}

We instantiate $\mathcal{L}_{\text{aux}}$ as a symmetry-prompting objective that directly penalizes the same symmetry violations measured in Section~\ref{sec:sym_driven_gen}.
Whereas previously symmetry violation served as a diagnostic metric, here we incorporate it into the training objective as an explicit regularizer.

For modular arithmetic, we enforce consistency under \textbf{commutativity} and \textbf{associativity} (and task-specific proxy identities for non-commutative operations; see Table~\ref{tab:sym_modular}).
For graph metric completion and comparison tasks, we instead enforce \textbf{triangle equality} and \textbf{transitivity}, respectively. Concretely, the loss encourages agreement between the model’s predictions on symmetry-related input tuples, thereby directly reducing symmetry violations during optimization.

\subsection{Geometry-prompting priors}

Previously, we observed that geometric organization follows symmetry acquisition and corresponds to a contraction of representations onto a low-dimensional manifold.
While symmetry alignment governs the onset of generalization, geometric organization stabilizes symmetry-consistent solutions by favoring low-rank structure.

To bias optimization toward this regime, we introduce geometry-promoting priors that encourage low-dimensional embeddings.
Specifically, we consider:
(i) \textbf{nuclear norm regularization}, which penalizes the sum of singular values and promotes low effective rank;
(ii) \textbf{entropy regularization}, which discourages diffuse, high-energy embeddings; and
(iii) \textbf{Lipschitz regularization}, which enforces local smoothness in representation space.

We note that these choices are consistent with prior findings linking low-rank structure to grokking dynamics~\citep{junior2025grokking} and entropy reduction to representation compression~\citep{demoss2025complexity}.

\subsection{Results} 

Figure~\ref{fig:faster} summarizes the effect of the proposed interventions on time-to-generalization across tasks.
Symmetry-prompting strategies consistently reduce the number of training steps required to reach perfect test accuracy relative to the baseline and others acceleration methods.
This effect holds across domains and seeds, indicating that directly minimizing symmetry violations systematically shifts the grokking transition. 

Geometry-promoting priors can also reduce time-to-generalization relative to the baseline. as noted in prior findings that low-rank structure facilitates representation compression~\citep{walker2025grokaligngeometriccharacterisationacceleration, demoss2025complexity, junior2025grokking}.
However, these improvements are smaller than those achieved through symmetry prompting, reinforcing the distinction between symmetry acquisition (\textit{generalization}) and geometric organization (\textit{stabilization}). In contrast, GrokFast~\citep{lee2024grokfastacceleratedgrokkingamplifying} exhibits task-dependent behavior, accelerating generalization in some settings but delaying convergence in others, suggesting sensitivity to task structure rather than alignment with intrinsic invariances. 

Overall, these results provide intervention-based evidence that symmetry alignment plays a causal role in enabling generalization in algorithmic tasks. Additional results are reported in Appendix~\ref{appendix:faster}.

\subsection{Discussion}
\label{subsec:accl_disc}

\textbf{Comparison with prior interpretations.} As previously noted, the model reliably converges to perfect test accuracy even \textbf{\textit{without}} weight decay (Figure~\ref{fig:wo_wd}). The addition of symmetry constraints (\textbf{bold line}) in this setting substantially accelerates generalization, reducing time-to-generalization by up to two orders of magnitude. 

We also examine \textit{initialization}~\citep{liu2023omnigrokgrokkingalgorithmicdata} and \textit{output scaling}~\citep{kumar2024grokkingtransitionlazyrich}, two factors previously reported to influence the emergence of grokking. As detailed in Appendix~\ref{Appendix:alternatives}, symmetry-enforcing objectives consistently accelerate generalization across a broad range of these regimes. These findings suggest that while scaling modulates the regime and timing of feature learning, intrinsic task symmetry provides the geometric \textbf{direction} that guides models beyond memorization toward generalization.

\section{Conclusion}
\label{sec:results}
We argue that intrinsic task symmetry is a primary driver of generalization in algorithmic reasoning. Across modular arithmetic, graph metric completion, and comparison tasks, we find that generalization arises when models move beyond memorizing training examples and instead acquire the task’s intrinsic symmetries. This observation supports a consistent three-stage dynamic of memorization, symmetry acquisition, and geometric organization, with the transition to perfect test accuracy aligning closely with symmetry acquisition. We further provide empirical and theoretical evidence that, once these symmetry constraints are internalized, mild optimization biases and low-rank pressures steer representations into the structured low-dimensional geometries observed in embedding space.

Building on this mechanism, we show that explicitly promoting symmetry can causally accelerate grokking. Symmetry-prompting objectives and geometry-promoting priors consistently reduce time-to-generalization across domains, with symmetry alignment providing the most reliable gains. Taken together, our results offer a principled explanation for why structured embedding geometries emerge after generalization, and position grokking as a process of symmetry acquisition underlying generalization in algorithmic tasks.

\section*{Acknowledgments}
We sincerely thank Juyoung Suk for his early contributions to the development of this work.

\newpage

\bibliography{arxiv}
\bibliographystyle{icml2025}

\include{appendix}

\include{appendix_theory}
\end{document}

%% file: appendix.tex
\newpage
\appendix
\onecolumn

\section{Training Dynamics of Grokking}
\label{Appendix:training_dynamics}

In this section, we present the figure for complete training dynamics of each task conducted in our experiments, visualizing the distinct three-stage progression: memorization, symmetry acquisition, and geometric organization.

\subsection{Modular Arithmetic Tasks}
For modular arithmetic, we test \textbf{Commutativity and Associativity} as our symmetric criterion, which can be computed from the model’s output logits  by comparing predictions along algebraically equivalent computation paths. For operations that do not naively satisfy these properties (e.g., subtraction), we employ a suitable \emph{proxy identity}, which is detailed in Table~\ref{tab:sym_modular}.

\paragraph{Notation.}
We denote by $p_\theta(\cdot \mid a,b)$ the model’s predicted distribution (softmax over logits) given 
inputs $(a,b)$, restricted to $\{0,\dots,p-1\}$. To quantify similarity between two distributions $p$ 
and $q$, we use the \emph{symmetric KL divergence}:
\[
\mathrm{KL}_{\text{sym}}(p,q) \;=\; \tfrac{1}{2}\Big(\mathrm{KL}(p\;\|\;q) + \mathrm{KL}(q\;\|\;p)\Big).    
\]

\newpage


\subsubsection{Addition: $x+y \pmod{p}$}

\begin{figure}[H]
    \centering
     \includegraphics[width=0.97\linewidth]{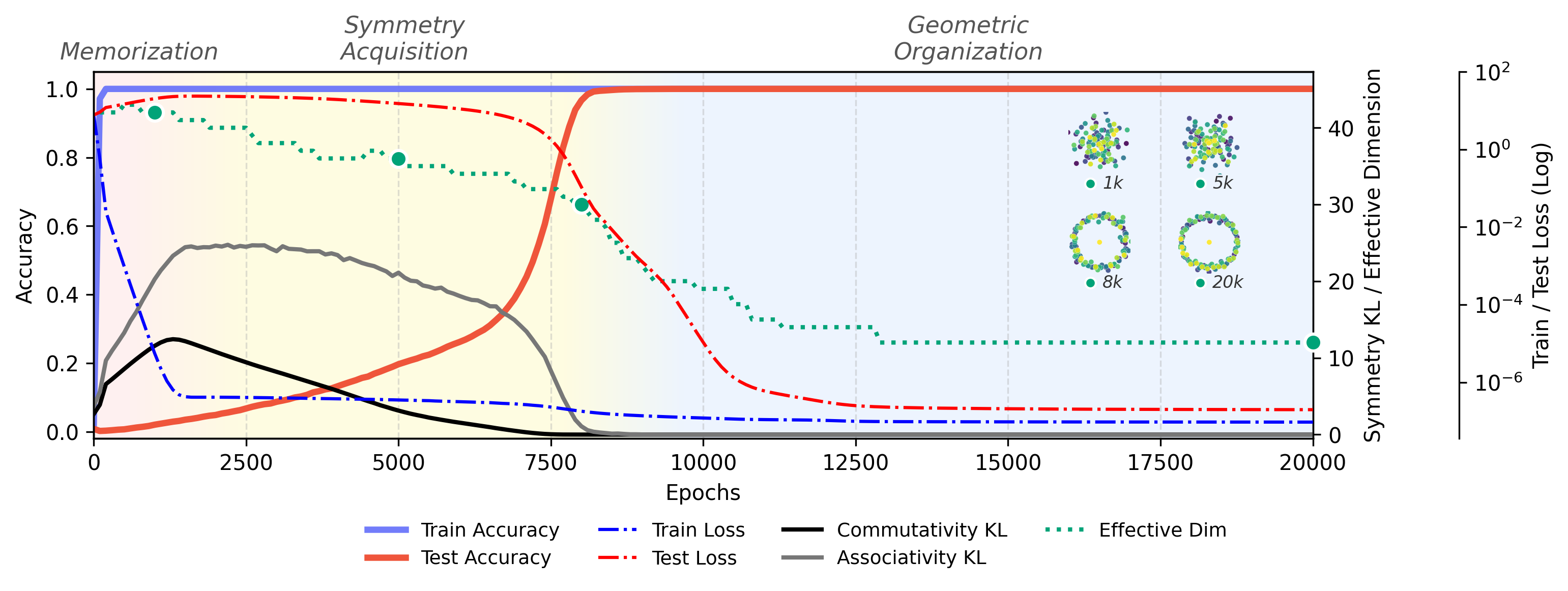}
    \caption{\textbf{Grokking Dynamics in Modular Addition.}}
\end{figure}


\subsubsection{Subtraction: $x-y \pmod{p}$}

\begin{figure}[H]
    \centering
     \includegraphics[width=0.97\linewidth]{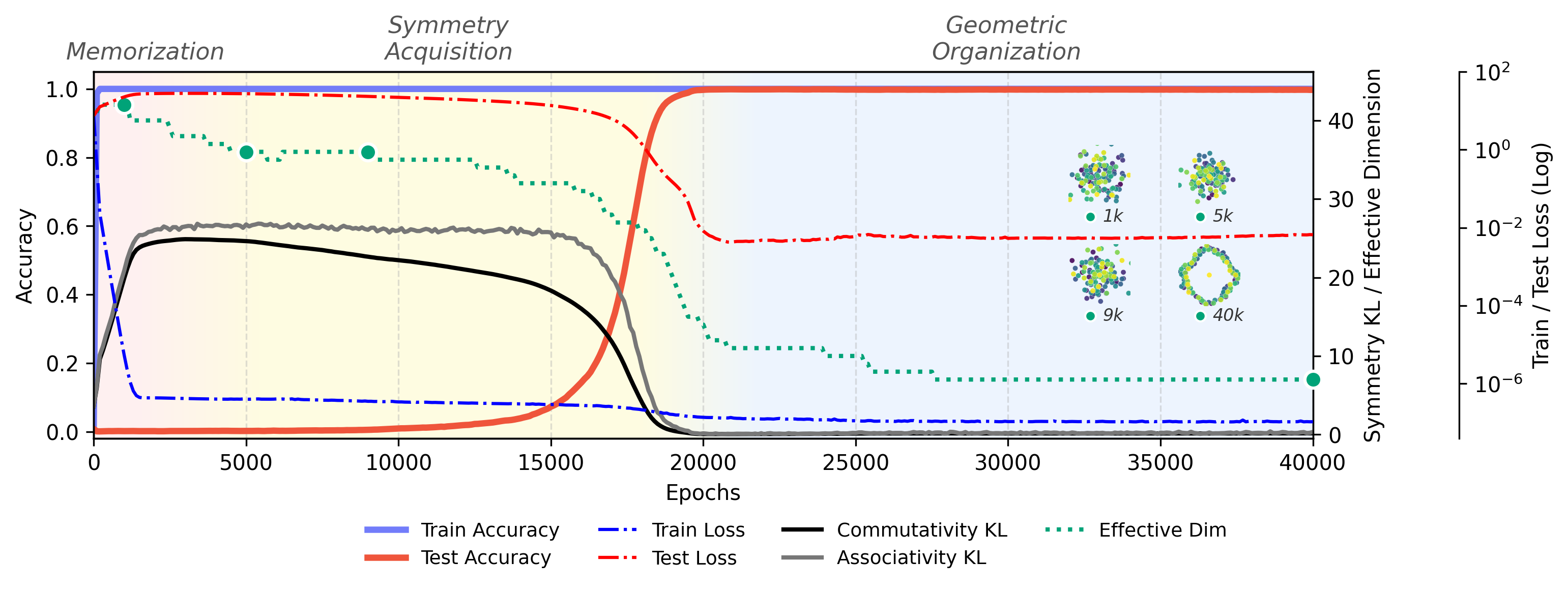}
    \caption{\textbf{Grokking Dynamics in Modular Subtraction.}}
\end{figure}

\newpage


\subsubsection{Squared sum: $(x+y)^2 \pmod{p}$}

\begin{figure}[H]
    \centering
     \includegraphics[width=0.97\linewidth]{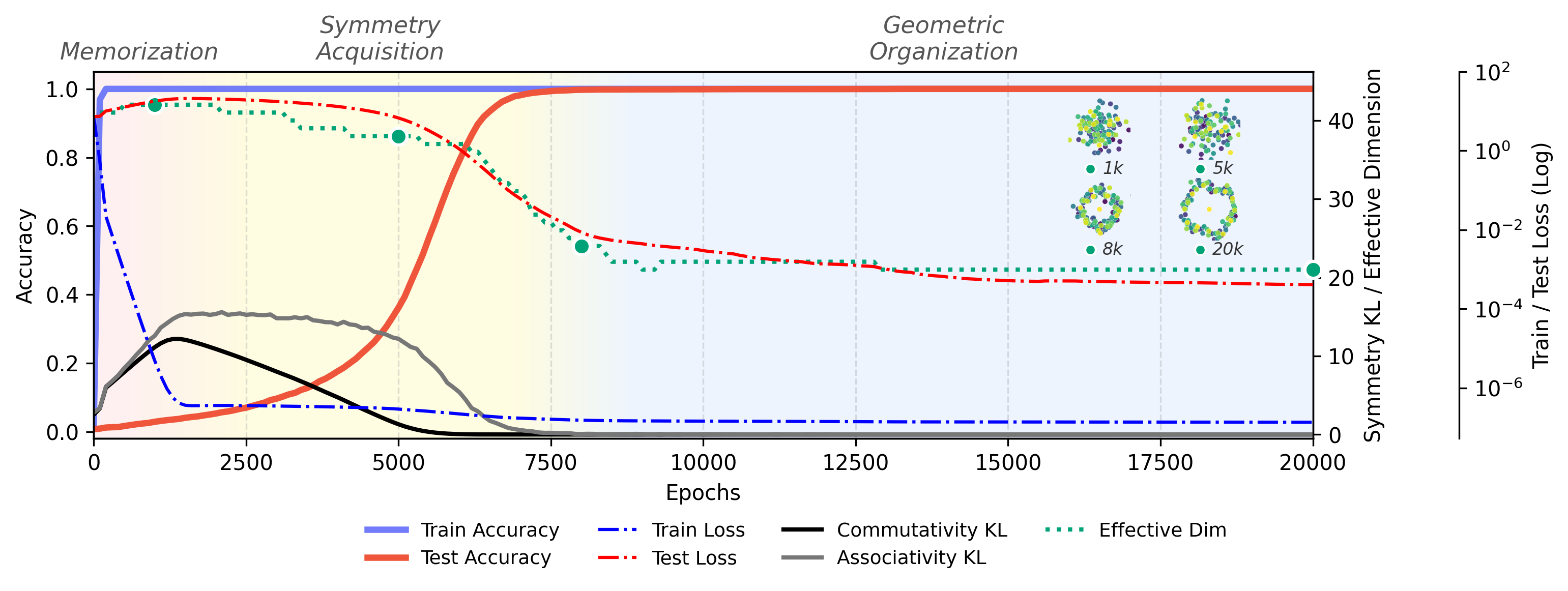}
    \caption{\textbf{Grokking Dynamics in Modular Squared Sum}}
\end{figure}


\subsubsection{Cubed sum: $(x+y)^3 \pmod{p}$}

\begin{figure}[H]
    \centering
     \includegraphics[width=0.97\linewidth]{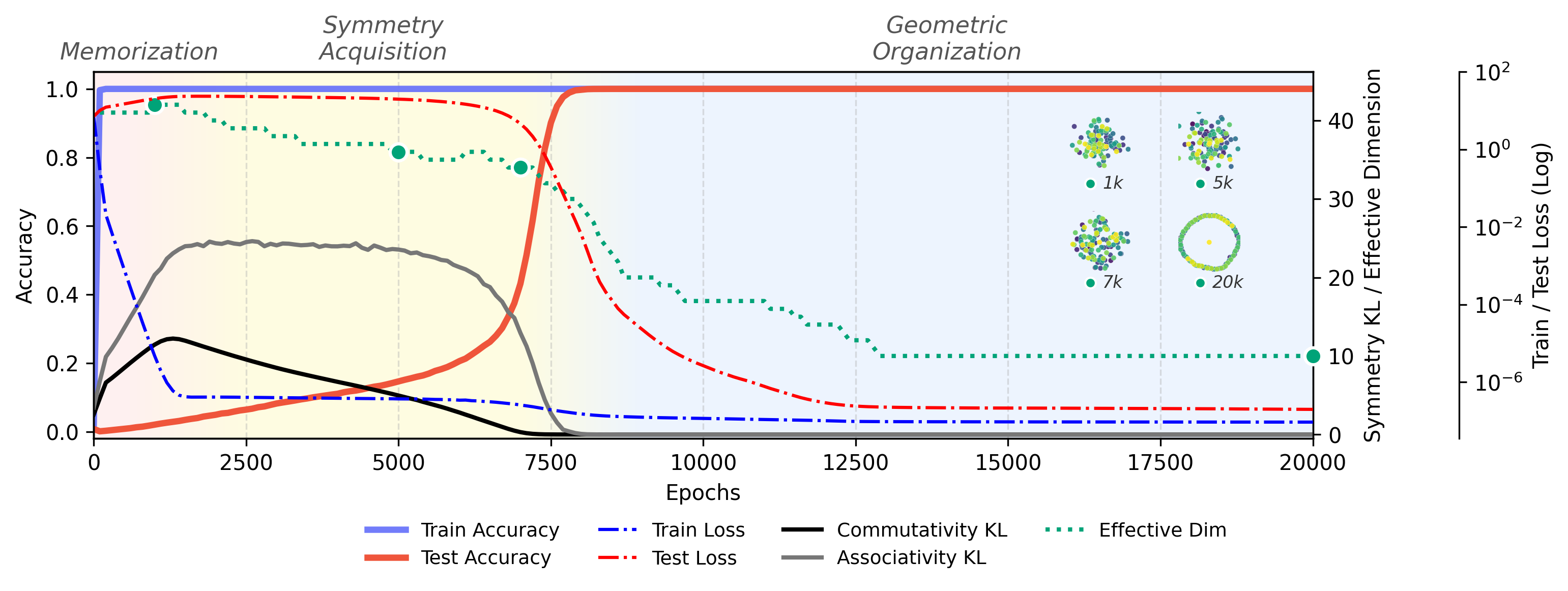}
    \caption{\textbf{Grokking Dynamics in Modular Cubed Sum}}
\end{figure}

\newpage

\subsubsection{Multiplication: $x \times  y \pmod{p}$}

\begin{figure}[H]
    \centering
     \includegraphics[width=0.97\linewidth]{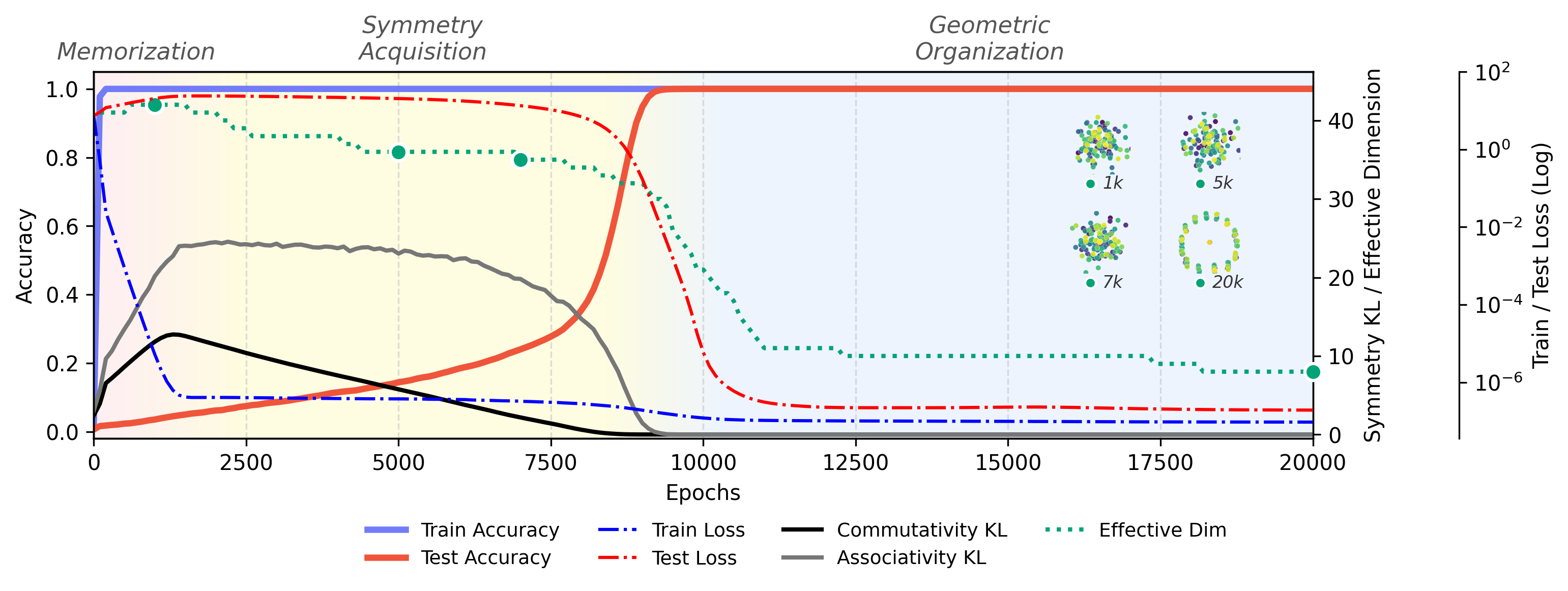}
    \caption{\textbf{Grokking Dynamics in Modular Multiplication}}
\end{figure}


\subsubsection{Affine--bilinear: $x + xy + y$}

\begin{figure}[H]
    \centering
     \includegraphics[width=0.97\linewidth]{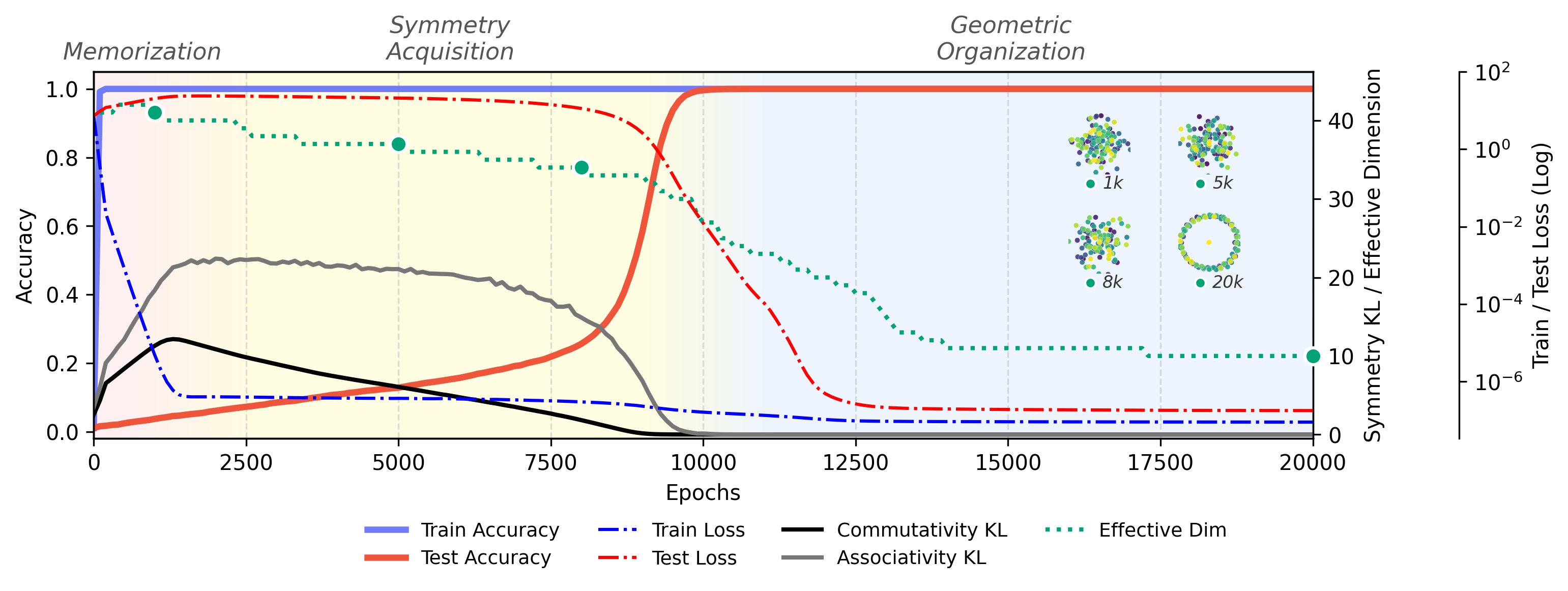}
    \caption{\textbf{Grokking Dynamics in Modular Affine--bilinear}}
\end{figure}












\newpage
\subsection{Geometric Progress Measure for Graph Metric Completion Tasks}
\label{Appendix:C}
For the graph metric completion task, we use the \textbf{Triangle Equality} property. This measure evaluates whether the model's learned representation respects the additive nature of distances along shortest paths. Formally, for any three nodes $u, v, w$, if $v$ lies on a geodesic (shortest path) between $u$ and $w$, the metric must satisfy: $d(u, w) = d(u, v) + d(v, w)$

To quantify this, we sample triplets $(a, b, c)$ such that $b$ is an intermediate node on a shortest path between $a$ and $c$. Then, we evaluate the consistency of the model's predictions by computing the convolution of the predicted probability distributions for $d(a, b)$ and $d(b, c)$, denoted as $P_{a \to b} * P_{b \to c}$. We then measure the Symmetric KL Divergence between this convolved distribution and the model's direct prediction for $d(a, c)$: $\mathcal{L}_{\text{tri}} = D_{\text{KL}}\left( (P_{a \to b} * P_{b \to c}) \parallel P_{a \to c} \right)$.

\subsubsection{Cycle}
\begin{figure}[H]
    \centering
     \includegraphics[width=0.97\linewidth]{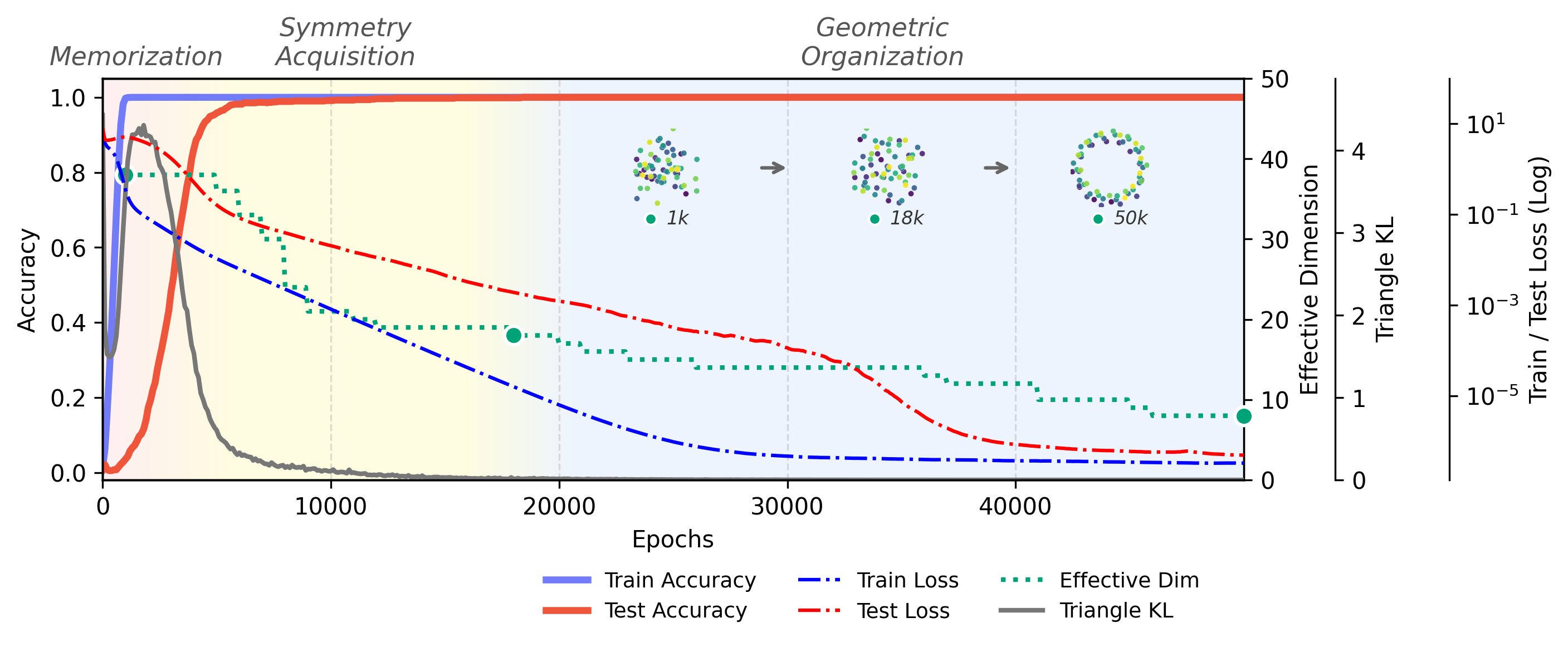}
    \caption{\textbf{Grokking Dynamics in Graph Metric Completion: Cycle}}
\end{figure}

\subsubsection{Path}
\begin{figure}[H]
    \centering
     \includegraphics[width=0.97\linewidth]{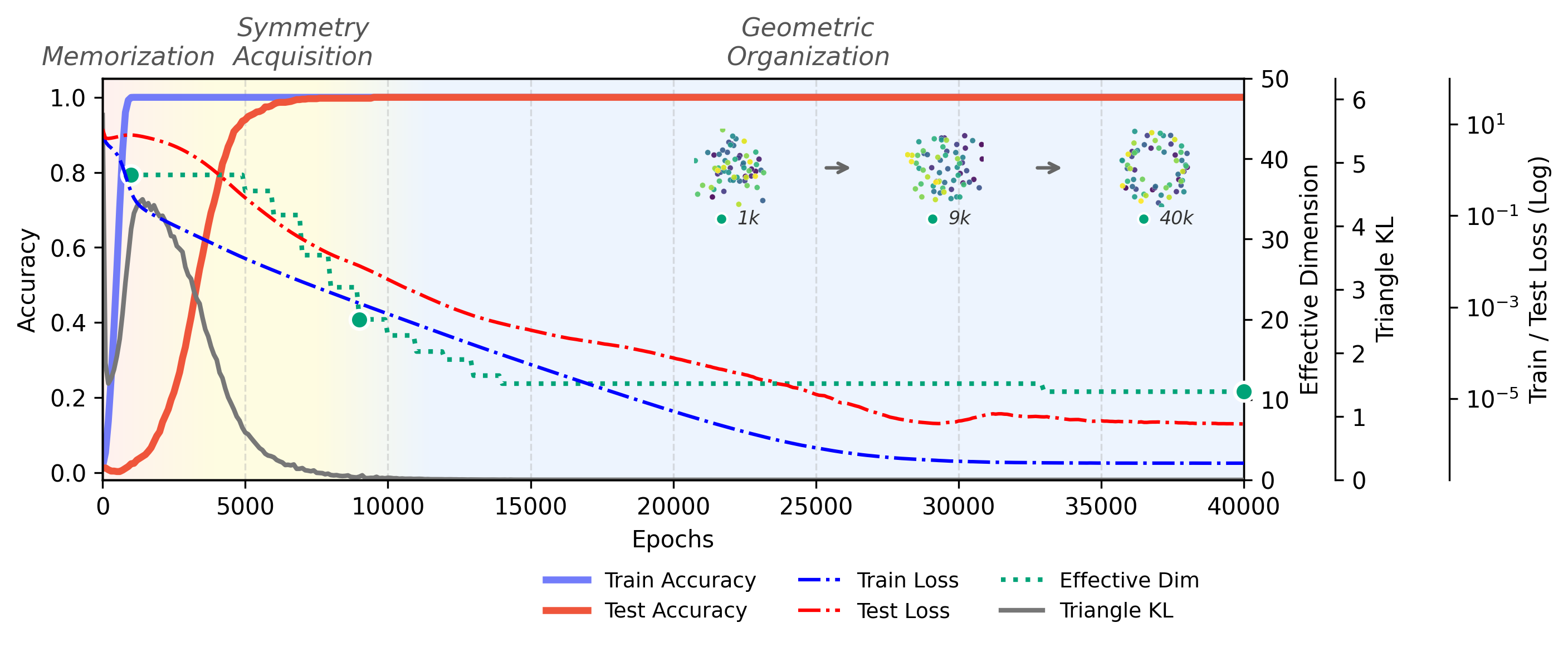}
    \caption{\textbf{Grokking Dynamics in Graph Metric Completion: Path}}
\end{figure}

\subsubsection{Cylinder}
\begin{figure}[H]
    \centering
     \includegraphics[width=0.97\linewidth]{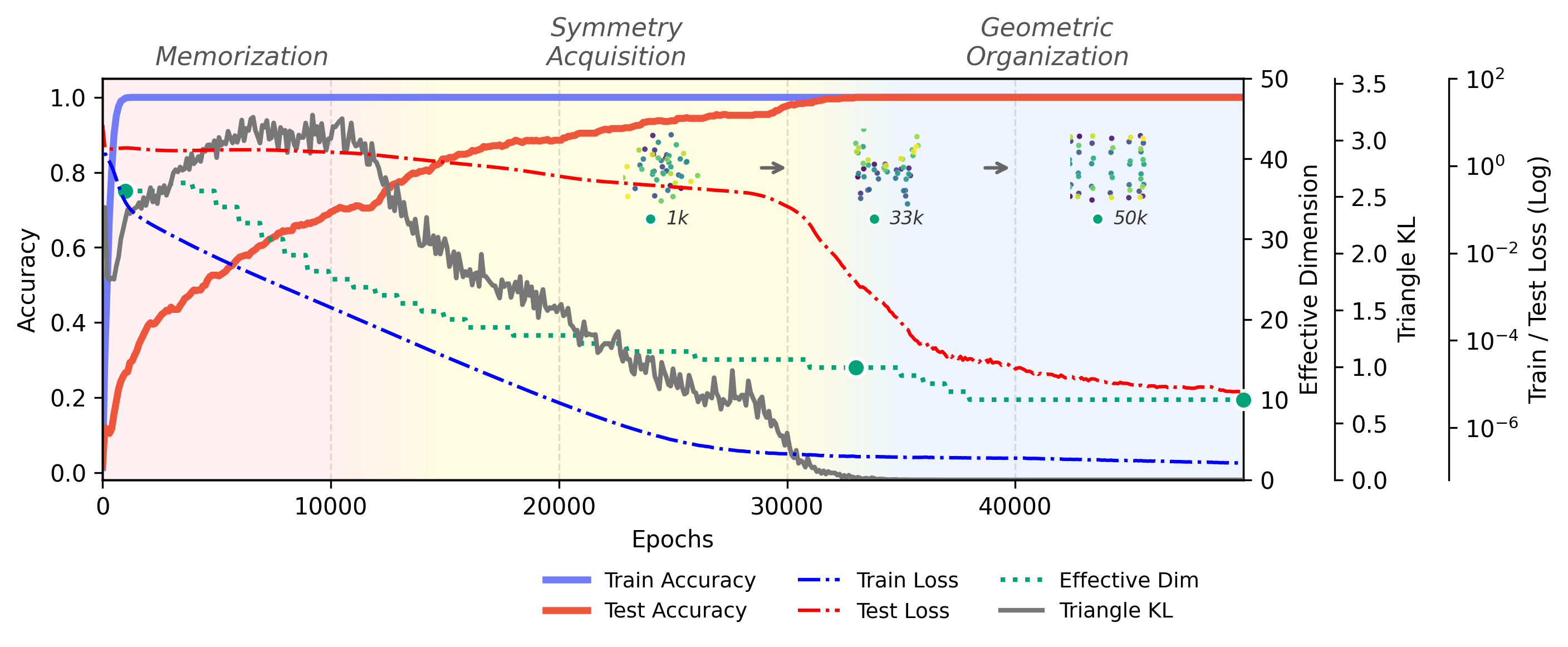}
    \caption{\textbf{Grokking Dynamics in Graph Metric Completion: Cylinder}}
\end{figure}

\subsubsection{Hypercube}
\begin{figure}[H]
    \centering
     \includegraphics[width=0.97\linewidth]{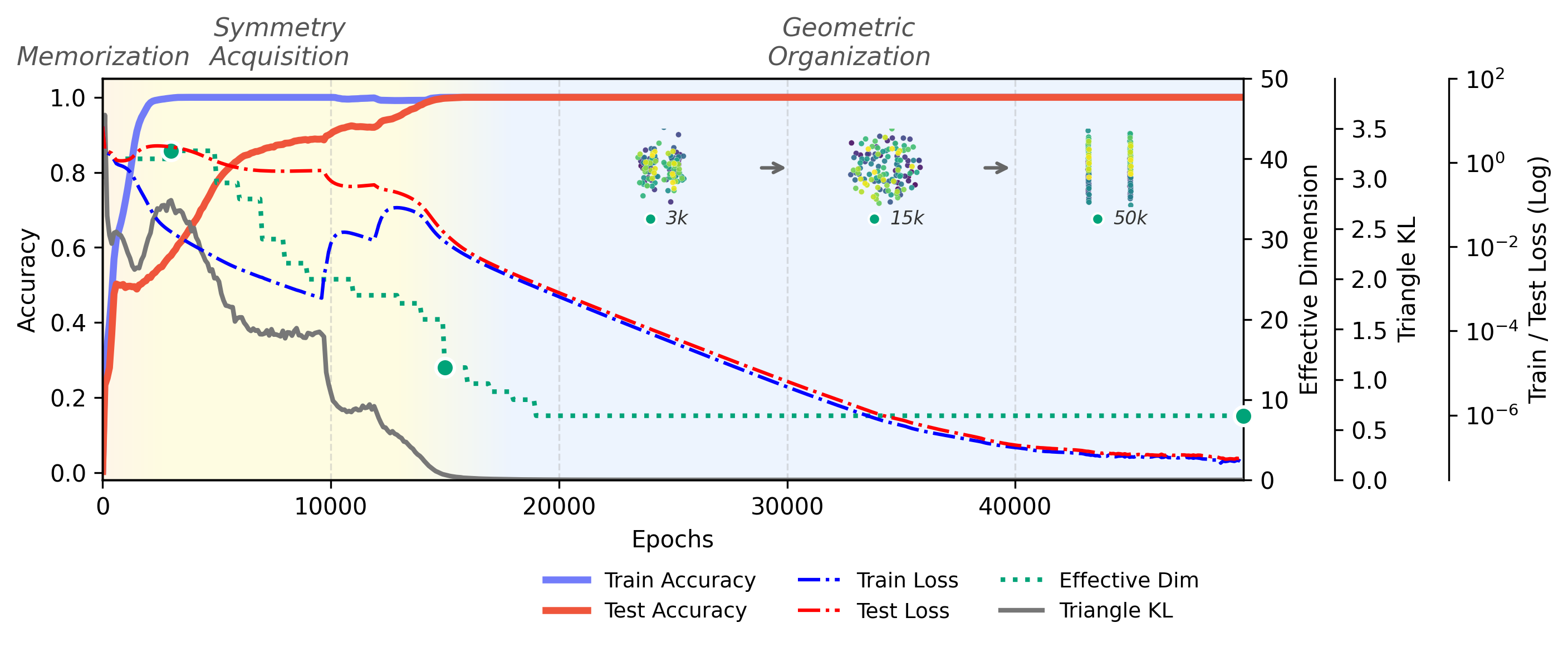}
    \caption{\textbf{Grokking Dynamics in Graph Metric Completion: Hypercube}}
\end{figure}

\subsubsection{2D Lattice}
\begin{figure}[H]
    \centering
     \includegraphics[width=0.97\linewidth]{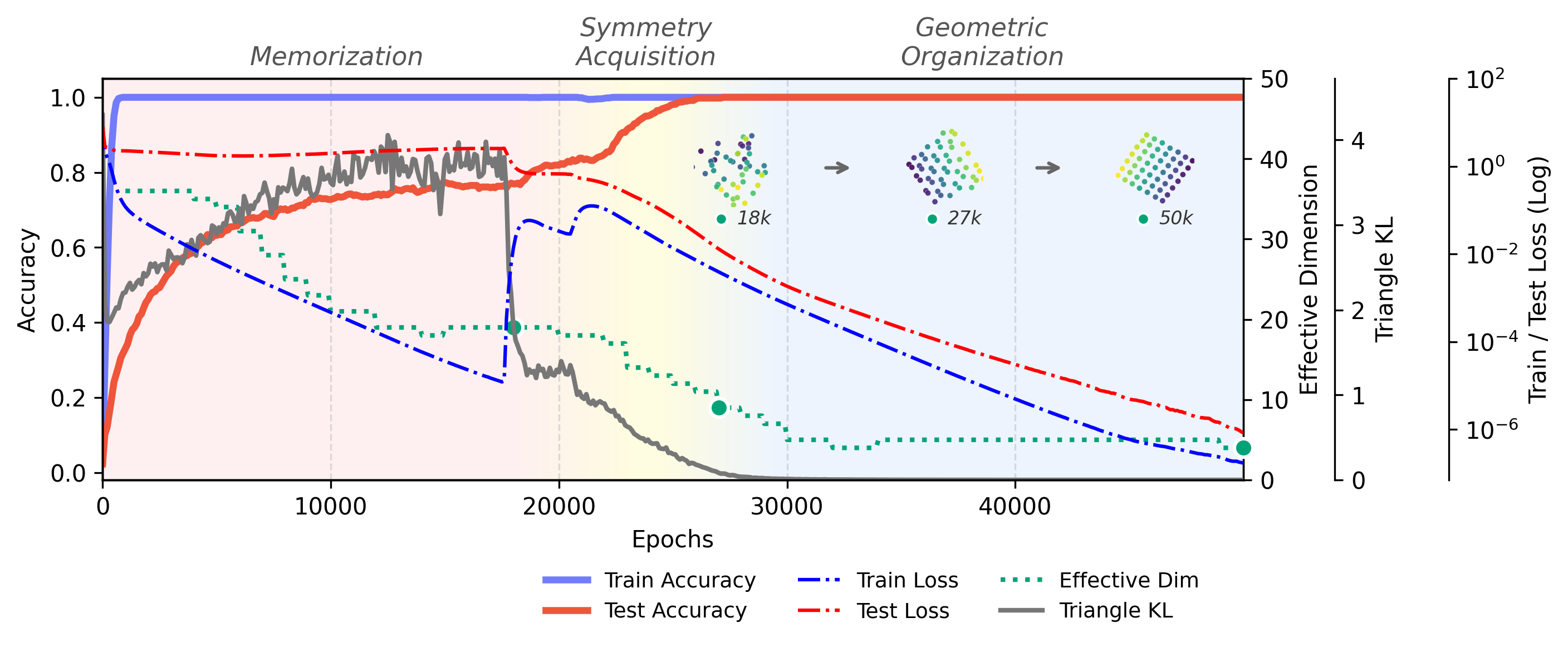}
    \caption{\textbf{Grokking Dynamics in Graph Metric Completion: Lattice}}
\end{figure}

\subsubsection{3D Lattice}
\begin{figure}[H]
    \centering
     \includegraphics[width=0.97\linewidth]{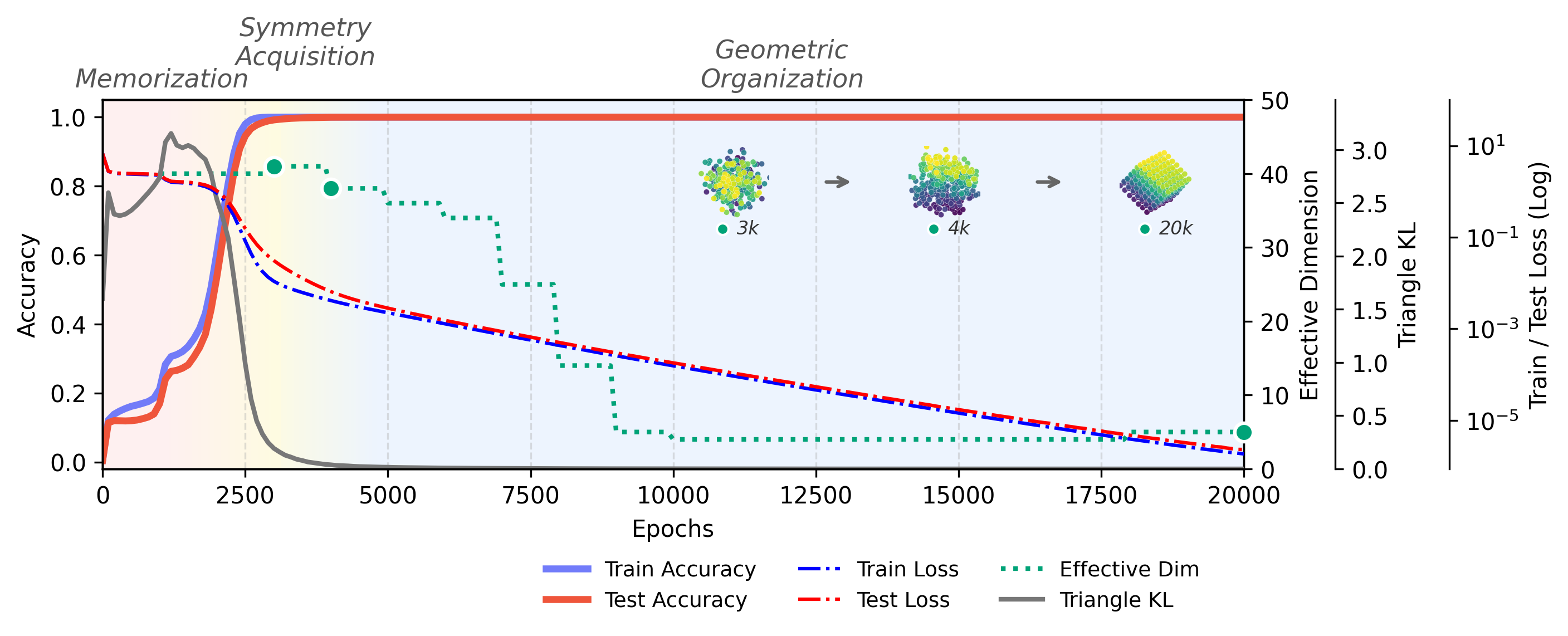}
    \caption{\textbf{Grokking Dynamics in Graph Metric Completion: 3D Lattice}}
\end{figure}

\newpage

\subsection{Comparison Tasks}

For the comparison task, we use the \textbf{Transitivity} property. We evaluate whether the model's pairwise comparisons are consistent with a global linear ordering, specifically testing if the learned representations adhere to a Bradley-Terry model, \textit{i.e.} the probability $P(a > b)$ is determined by the difference in latent scores, implying that the log-odds should be additive: $\ell_{ac} \approx \ell_{ab} + \ell_{bc}$.

In practice, we sample triples $(a, b, c)$ strictly ordered by their ground-truth rank such that $a < b < c$. We first extract the model's binary log-odds for each pair by taking the difference between the \textit{greater} and \textit{less} logits: $\ell = \text{logit}_{>} - \text{logit}_{<}$. After, we compare two distributions: the model's direct prediction $P(a > c) = \sigma(\ell_{ac})$ and the transitive prediction implied by the intermediate steps $Q(a > c) = \sigma(\ell_{ab} + \ell_{bc})$. The final metric is the Symmetric KL Divergence between these two distributions: $\mathcal{L}_{\text{trans}} = \frac{1}{2} \left( D_{\text{KL}}(P \parallel Q) + D_{\text{KL}}(Q \parallel P) \right)$.

\subsubsection{2D Comparison}
\begin{figure}[H]
    \centering
     \includegraphics[width=0.97\linewidth]{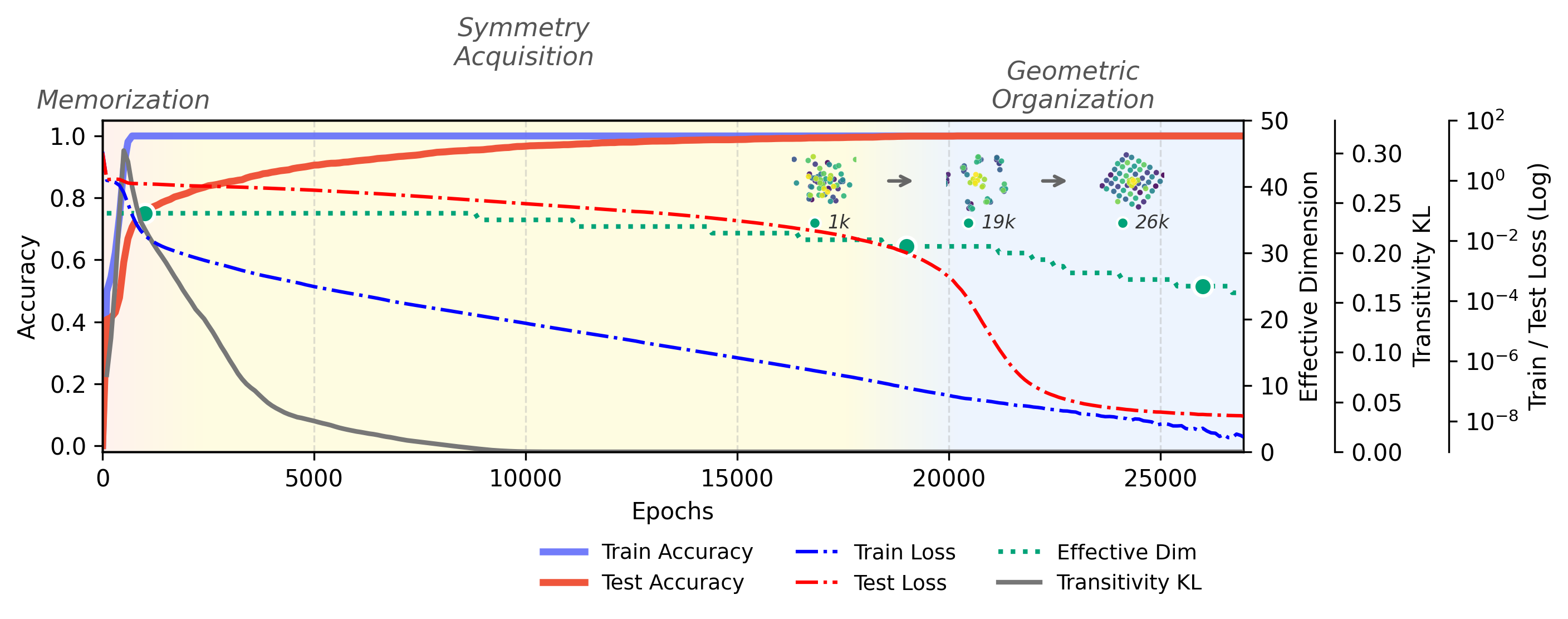}
    \caption{\textbf{Grokking Dynamics in 2D Comparison}}
\end{figure}

\subsubsection{3D Comparison}
\begin{figure}[H]
    \centering
     \includegraphics[width=0.97\linewidth]{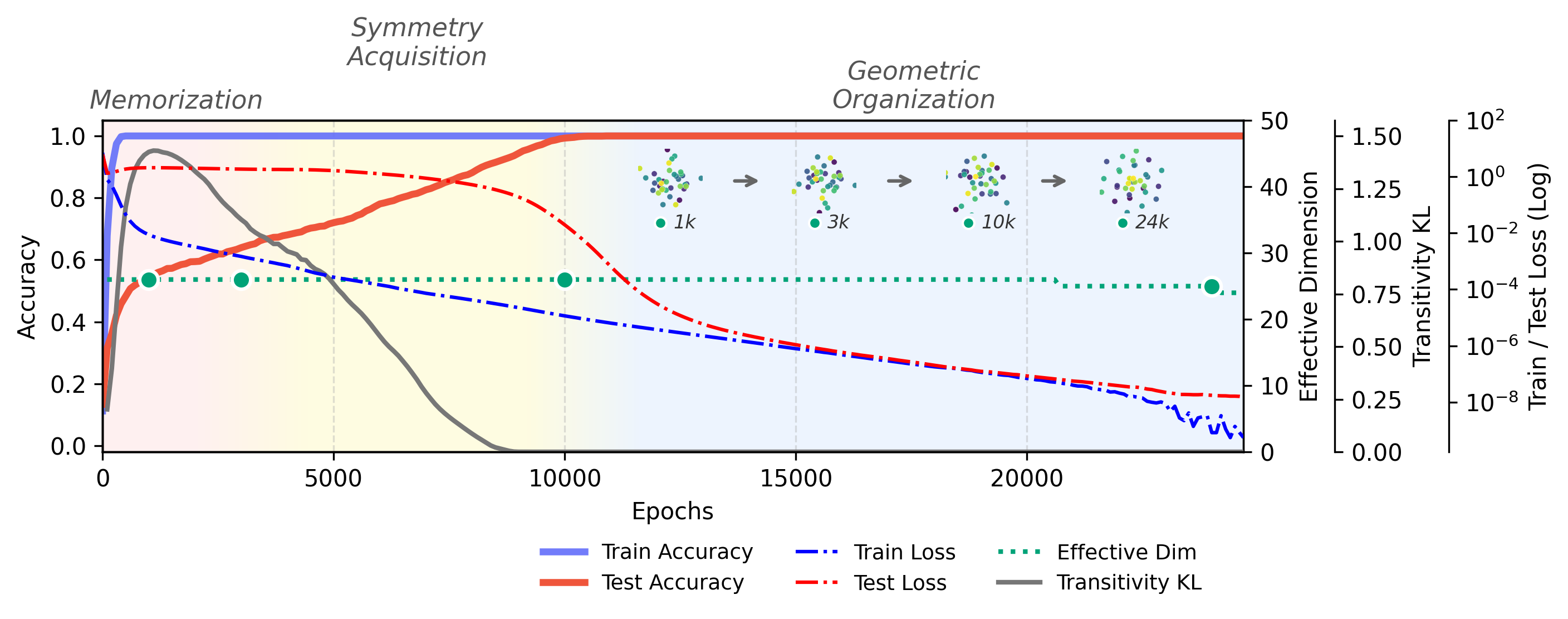}
    \caption{\textbf{Grokking Dynamics in 3D Comparison}}

\end{figure}
\newpage

\section{Characterizing Generalization via Intrinsic Symmetry}
\label{Appendix:sym_metric}

Below we plot the maximum \textit{instance-level} symmetry criterion achieved during training for each task. We observe a high correlation between low maximum instance-level criterion and successful generalization. These results suggest that the elimination of structural outliers represented by the maximum symmetry violation is a key marker of the phase transition from memorization to true algorithmic understanding. 

For each task, we track 20 independent seeds. In the scatter plots, each point represents the state of a run at a specific training epoch: \textbf{blue points} denote runs that have successfully generalized (reached 100\% accuracy) by that epoch, while \textbf{red markers} indicate runs that have not yet generalized. We observe a \textit{\textbf{soft threshold}} distinguishing these states, where generalizing runs consistently exhibit lower maximum symmetry violations compared to their non-generalizing counterparts.

\subsection{Modular Arithmetics}
\begin{figure}[H]
    \centering
    \includegraphics[height=0.40\textheight]{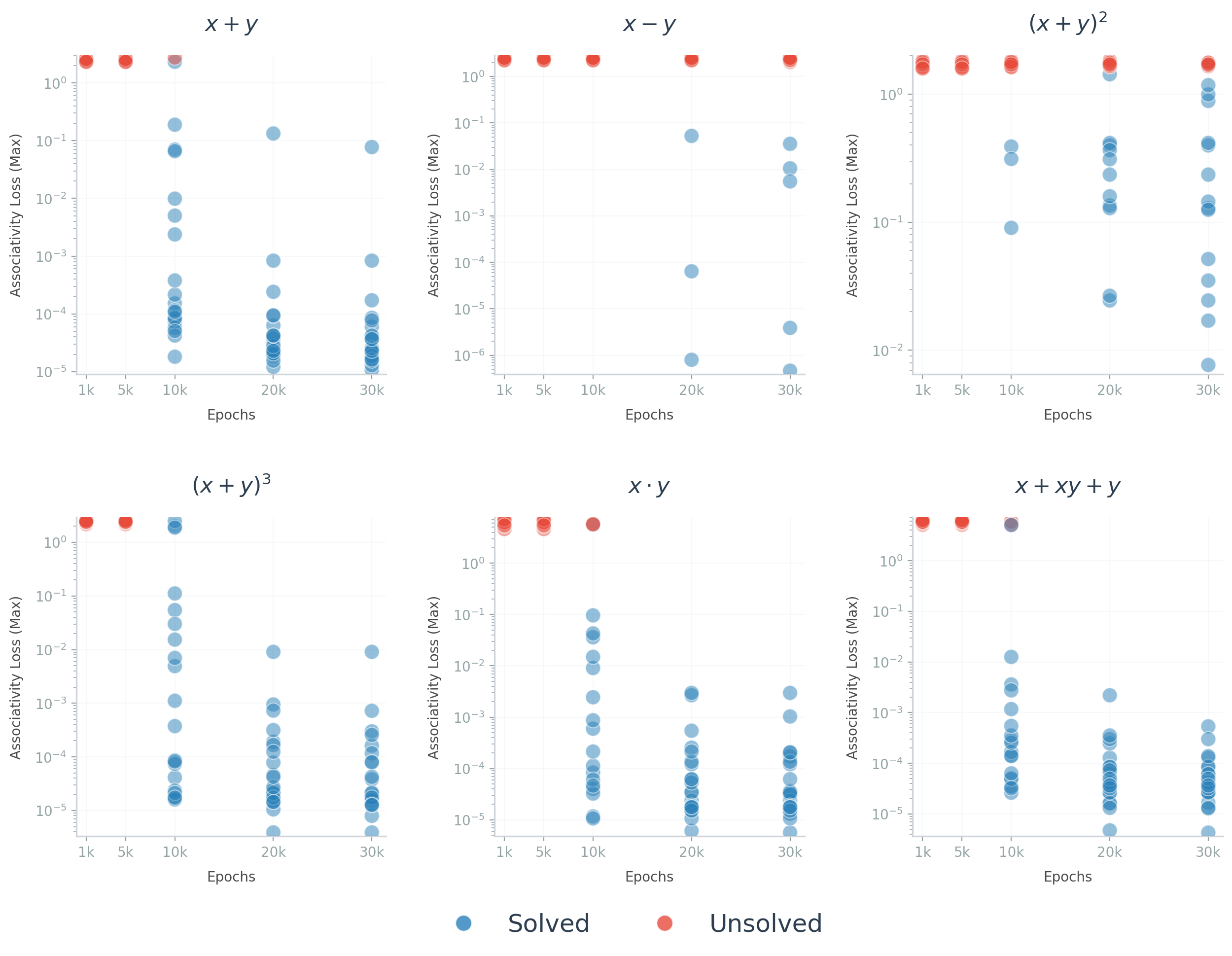}
\caption{\textbf{Maximum Symmetry Violation and Generalization Status for }modular arithmetic\textbf{ tasks.}}
\end{figure}

\newpage

\subsection{Graph Metric Completion}
\begin{figure}[H]
    \centering
    \includegraphics[height=0.40\textheight]{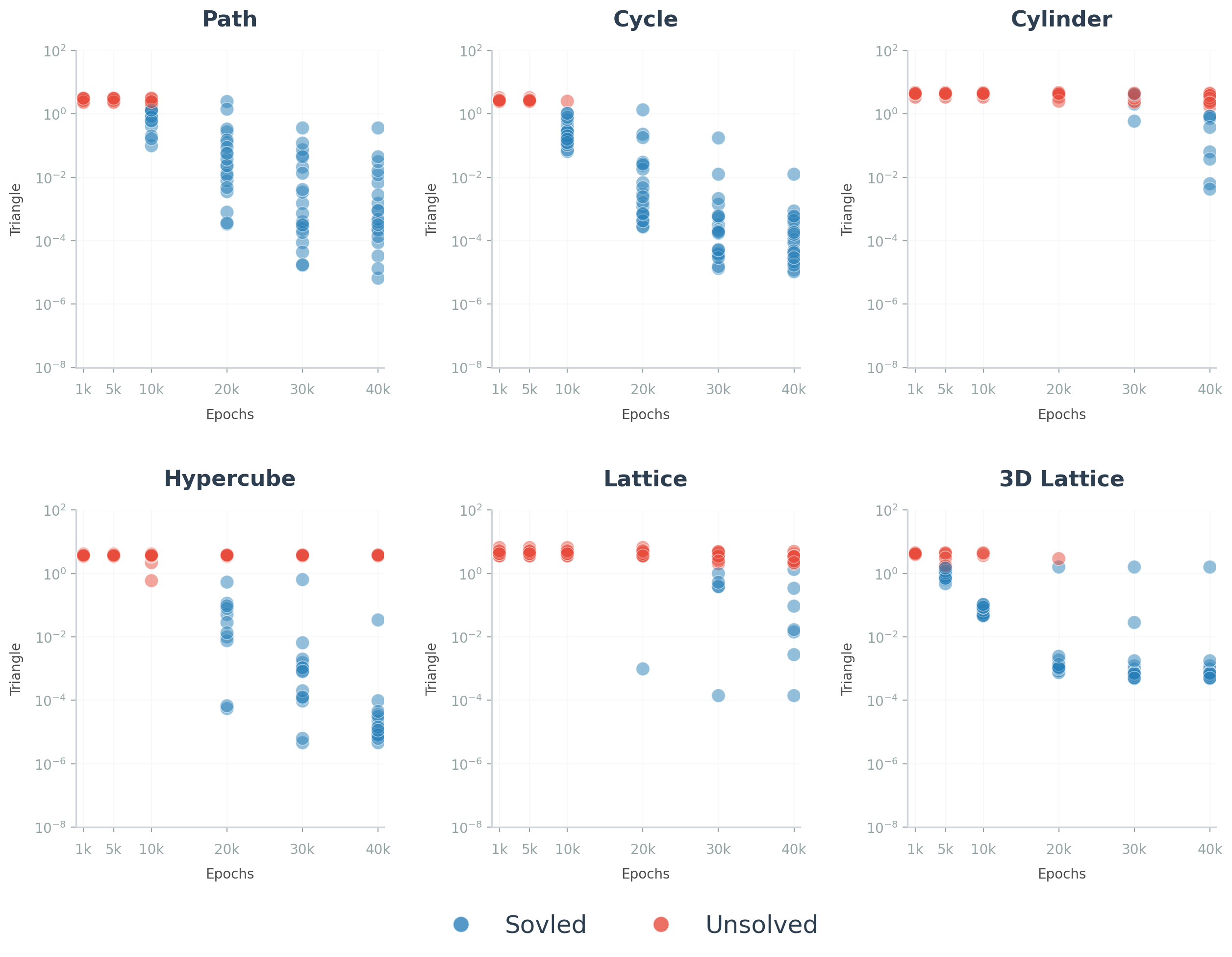}
    \caption{Symmetry Loss vs. Test Accuracy for Graph Metric Completion Tasks}
\end{figure}

\subsection{Comparison}
\begin{figure}[H]
    \centering
    \includegraphics[height=0.18\textheight]{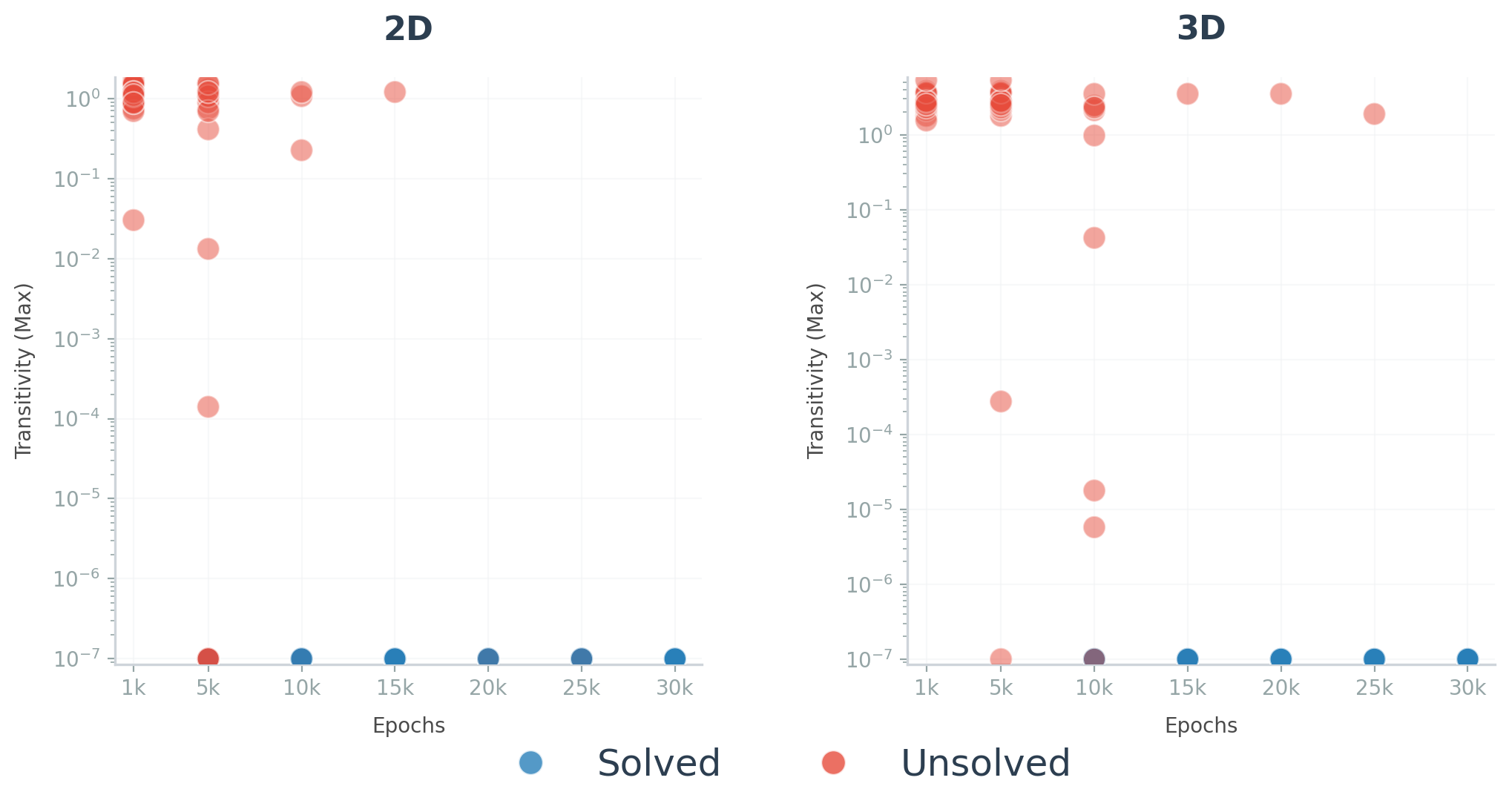}
    \caption{Symmetry Loss vs. Test Accuracy for Comparison Tasks}
\end{figure}

\newpage

\section{Analysis of Embedding Geometry}
\label{Appendix:A}

\subsection{Effective Dimensionality}
We quantify the effective dimensionality of the embedding space using a PCA-based criterion. We apply principal component analysis to the token embedding matrix and define the effective embedding dimension as the number of components whose explained variance exceeds 1\%. This serves as a simple and robust proxy for the dimensionality of the learned representations.

\subsection{Full Visualization}

Here, we visualize the geometric structure of the learned embeddings for each task using Principal Component Analysis (PCA). To fully capture the manifold's shape, we visualize the pairwise projections of the top principal components, plotting up to a maximum of 8 dimensions.

\subsection{Modular Arithmetic}

\subsubsection{Task 1 : Addition: $x+y \pmod{p}$}

\begin{figure}[H]
    \centering
    \includegraphics[height=0.41\textheight]{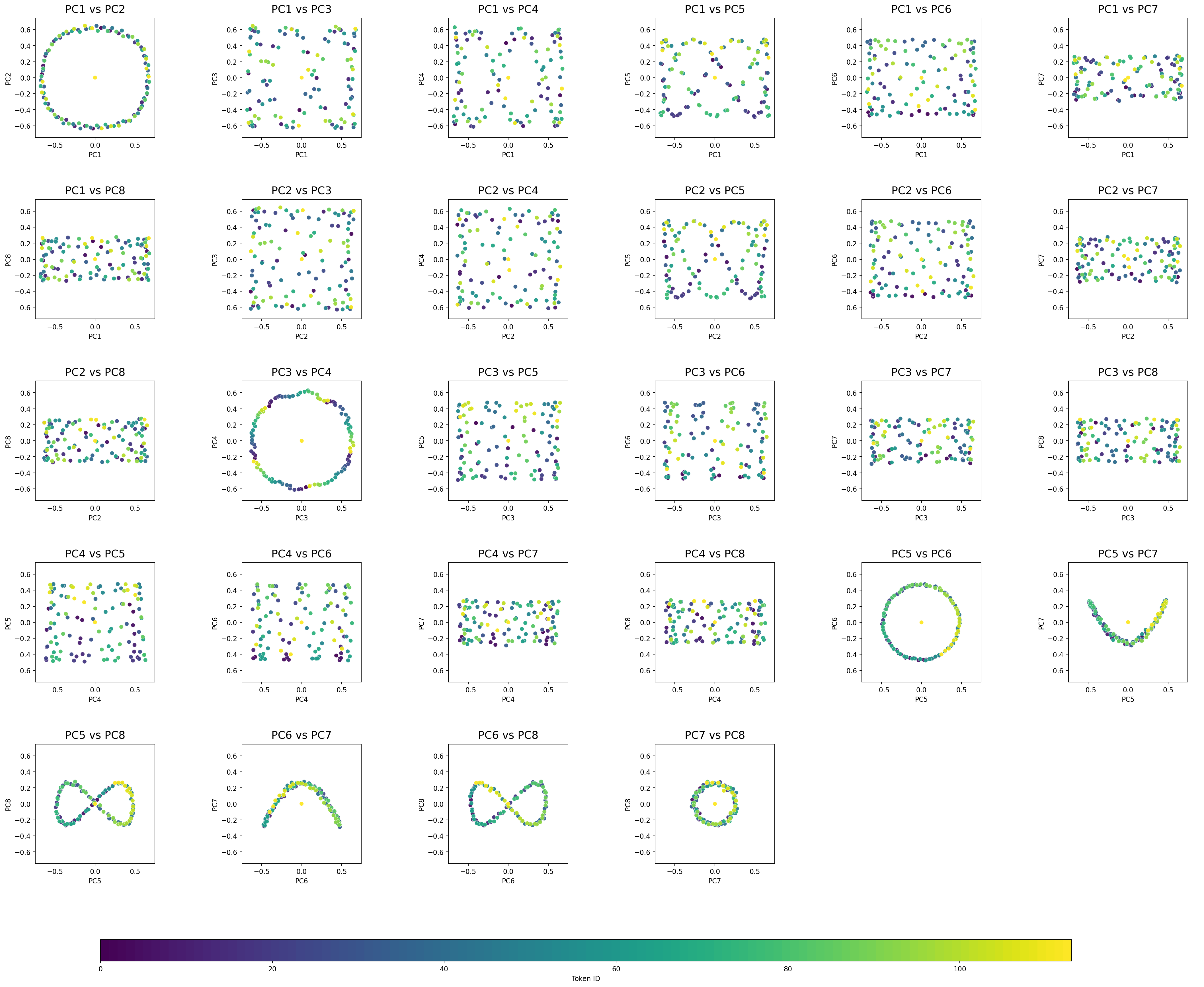}
    \caption{PCA projection of embeddings of model trained on modular addition task.}
\end{figure}

\newpage

\subsubsection{Task 2 : Subtraction: $x-y \pmod{p}$}
\begin{figure}[H]
    \centering
    \includegraphics[height=0.41\textheight]{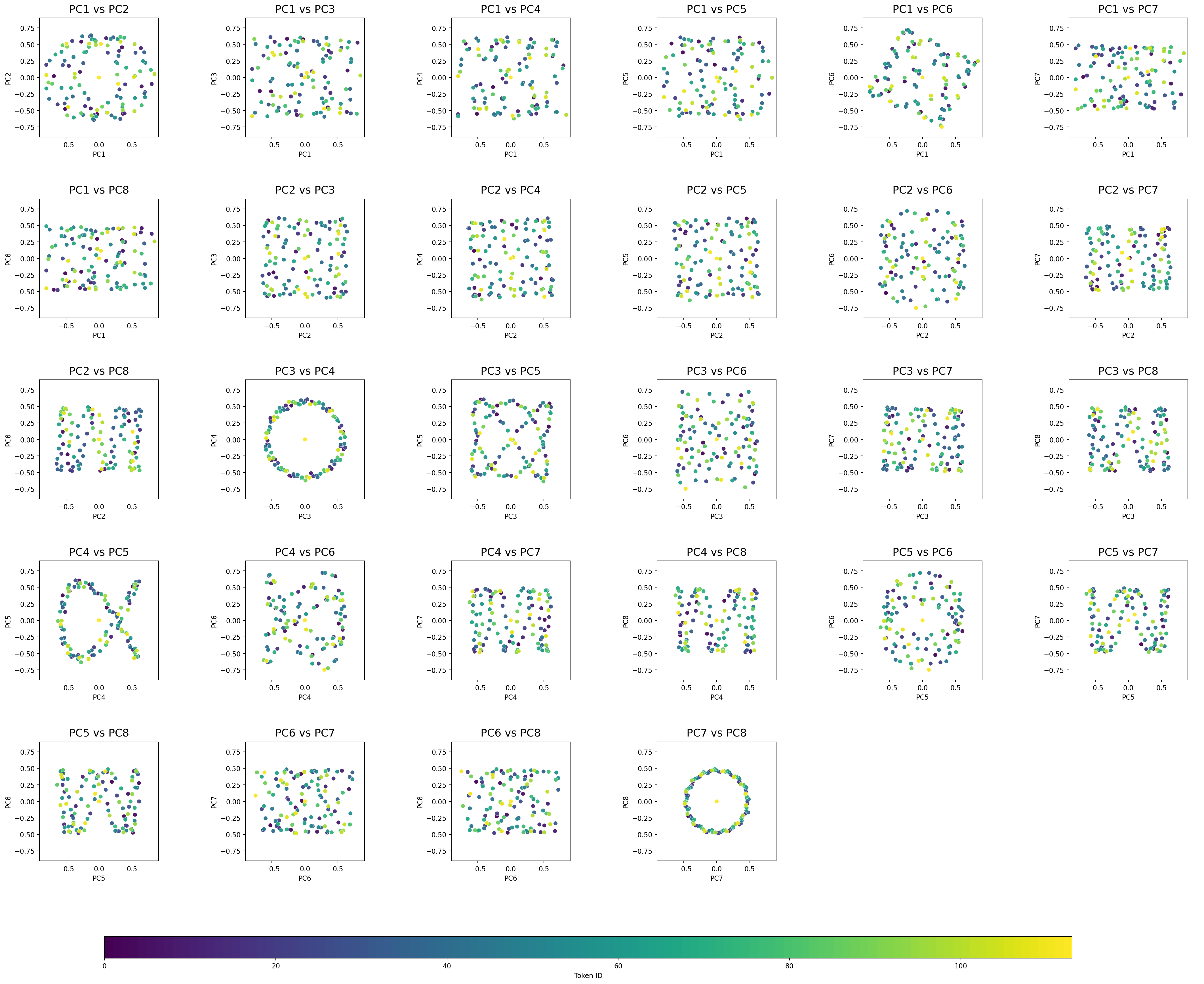}
    \caption{PCA projection of embeddings of model trained on modular subtraction task.}
\end{figure}

\subsubsection{Task 3 : $(x+y)^2 \pmod{p}$}
\begin{figure}[H]
    \centering
    \includegraphics[height=0.41\textheight]{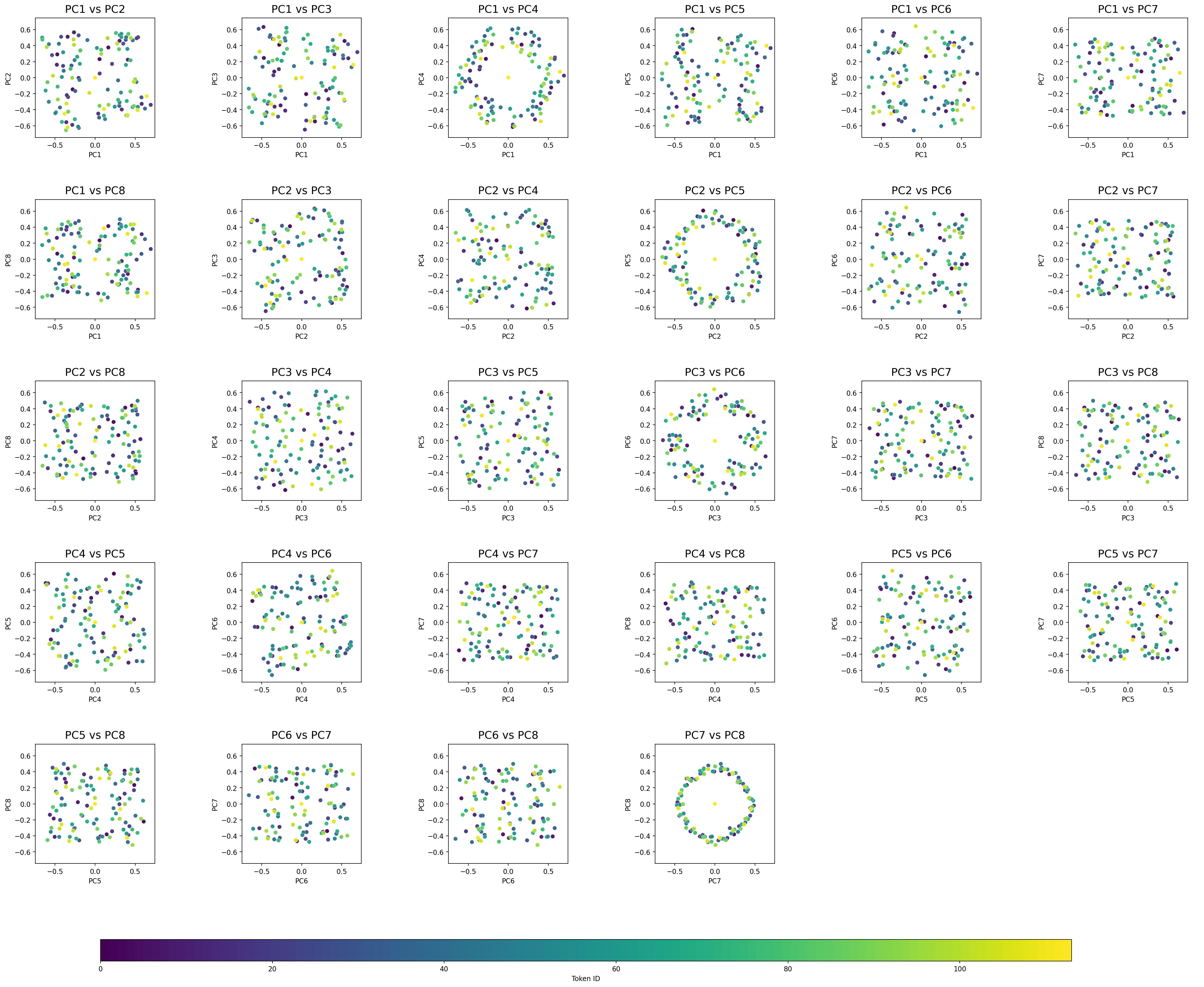}
    \caption{PCA projection of embeddings of model trained on Task 3.}
\end{figure}

\newpage

\subsubsection{Task 4 : $(x+y)^3 \pmod{p}$}
\begin{figure}[H]
    \centering
    \includegraphics[height=0.41\textheight]{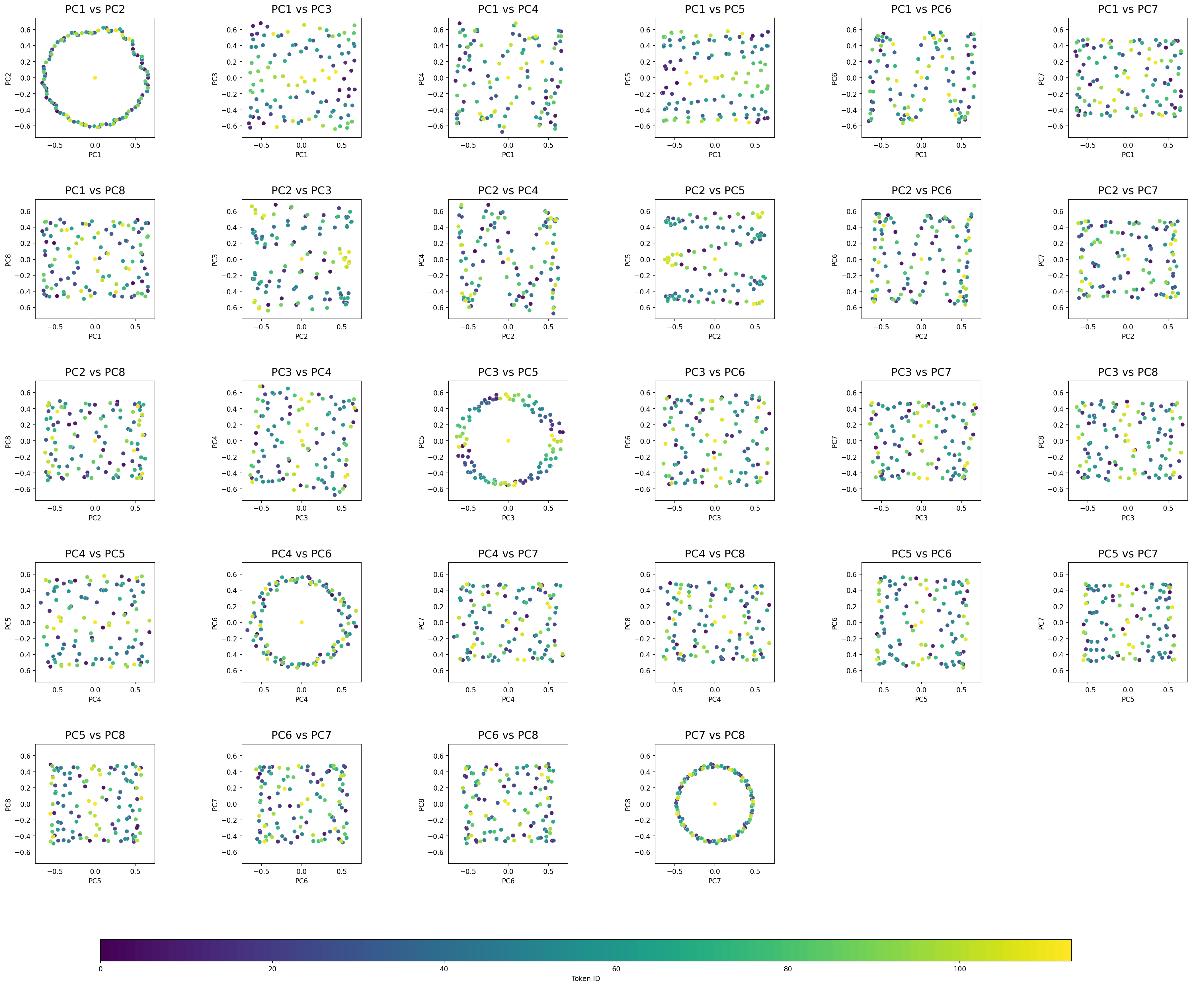}
    \caption{PCA projection of embeddings of model trained on Task 4.}
\end{figure}



\subsubsection{Task 5 : Multiplication: $ x \times y \pmod{p}$}
\begin{figure}[H]
    \centering
    \includegraphics[height=0.41\textheight]{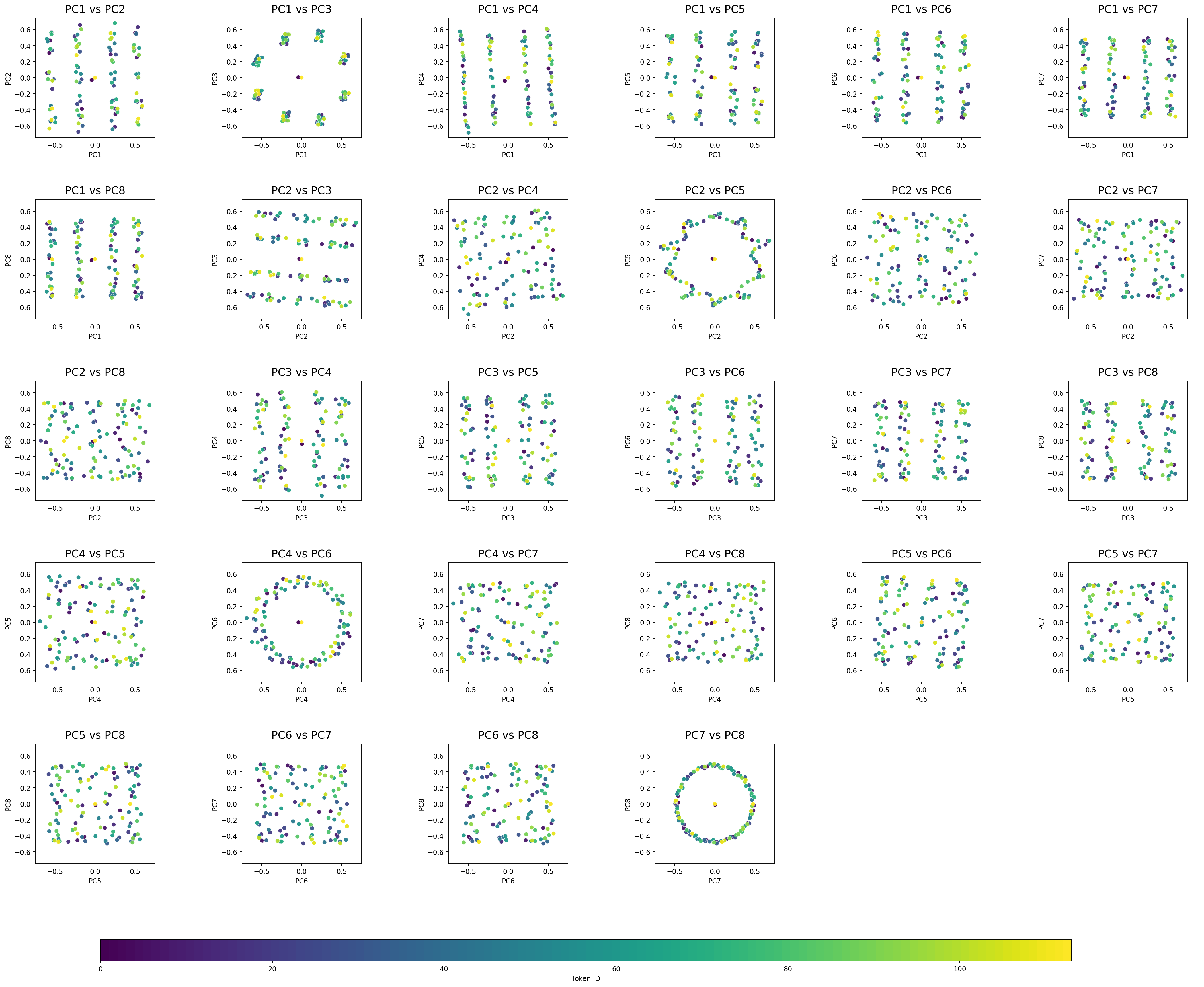}
    \caption{PCA projection of embeddings of model trained on modular multiplication task.}
\end{figure}

\newpage

\subsubsection{Task 6 : $ x+xy+y \pmod{p}$}
\begin{figure}[H]
    \centering
    \includegraphics[height=0.41\textheight]{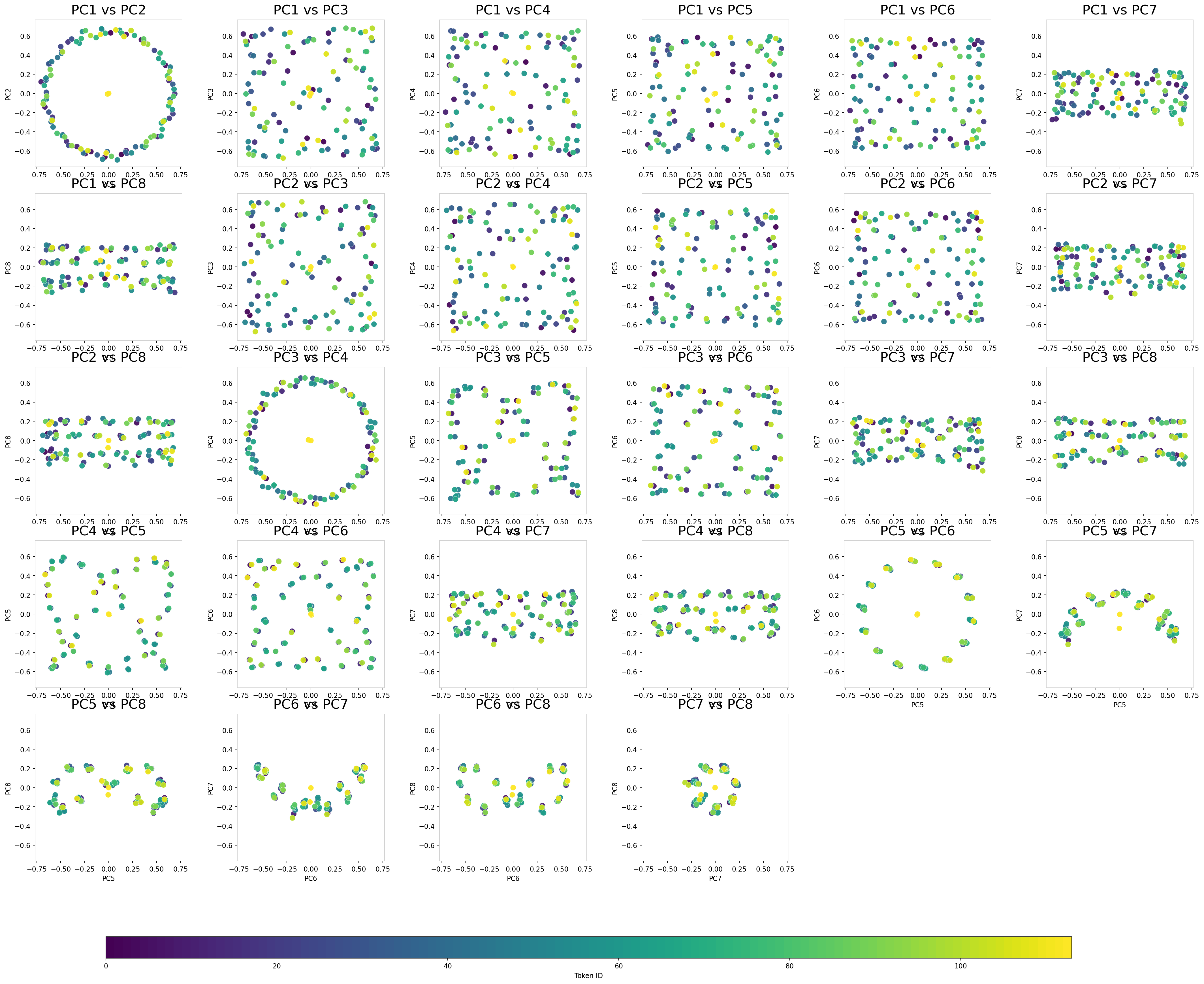}
    \caption{PCA projection of embeddings of model trained on Task 6.}
\end{figure}

\newpage

\subsection{graph metric completion task}

\subsubsection{Path}
\begin{figure}[H]
    \centering
    \includegraphics[height=0.39\textheight]{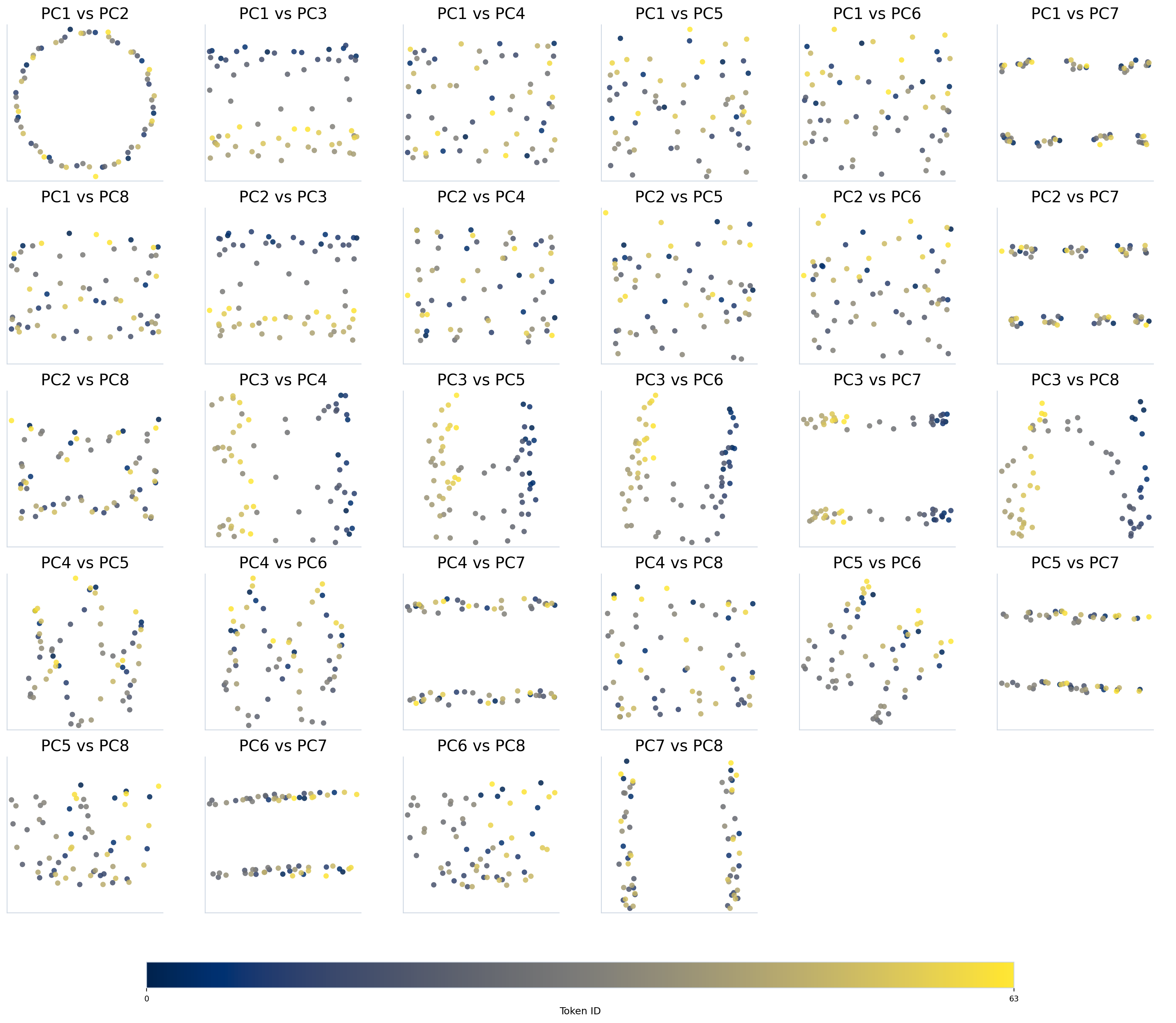}
    \caption{PCA projection of embeddings of model trained on the \textsc{Path} graph metric completion.}
\end{figure}

\subsubsection{Cycle}
\begin{figure}[H]
    \centering
    \includegraphics[height=0.39\textheight]{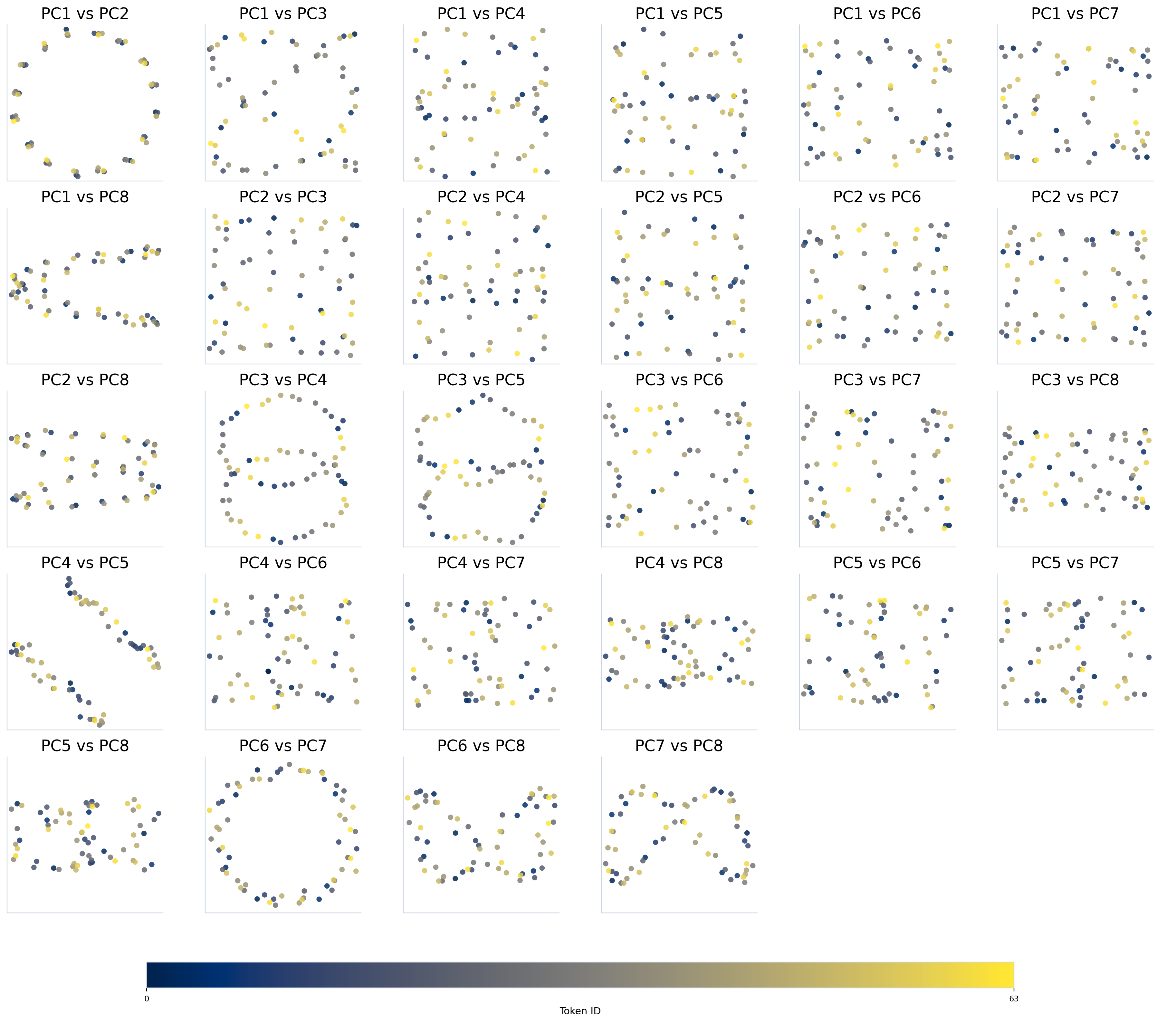}
    \caption{PCA projection of embeddings of model trained on the \textsc{Cycle} graph metric completion.}
\end{figure}

\newpage

\subsubsection{Cylinder}
\begin{figure}[H]
    \centering
    \includegraphics[height=0.39\textheight]{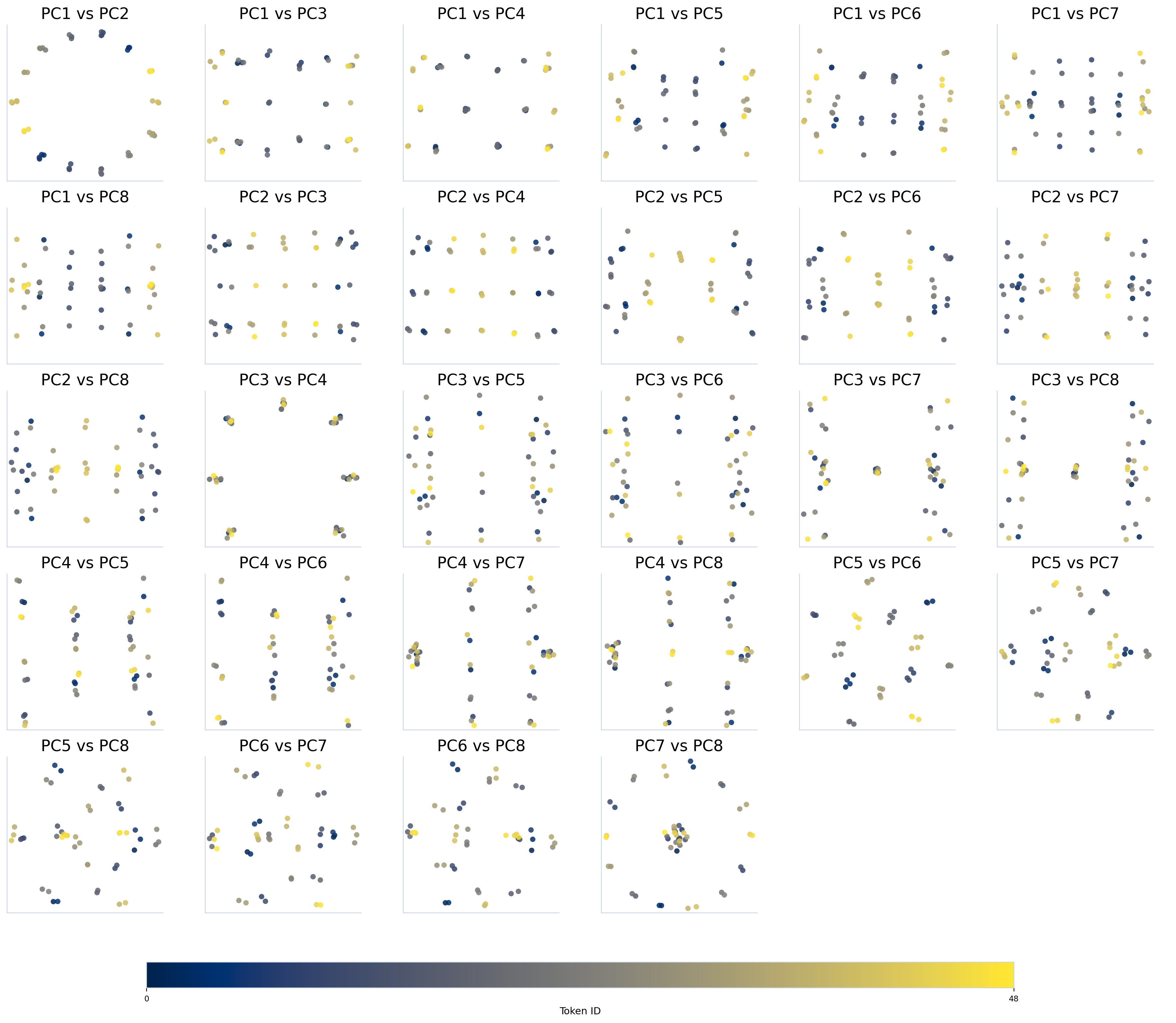}
    \caption{PCA projection of embeddings of model trained on the \textsc{Cylinder} graph metric completion.}
\end{figure}

\subsubsection{Hypercube}
\begin{figure}[H]
    \centering
    \includegraphics[height=0.39\textheight]{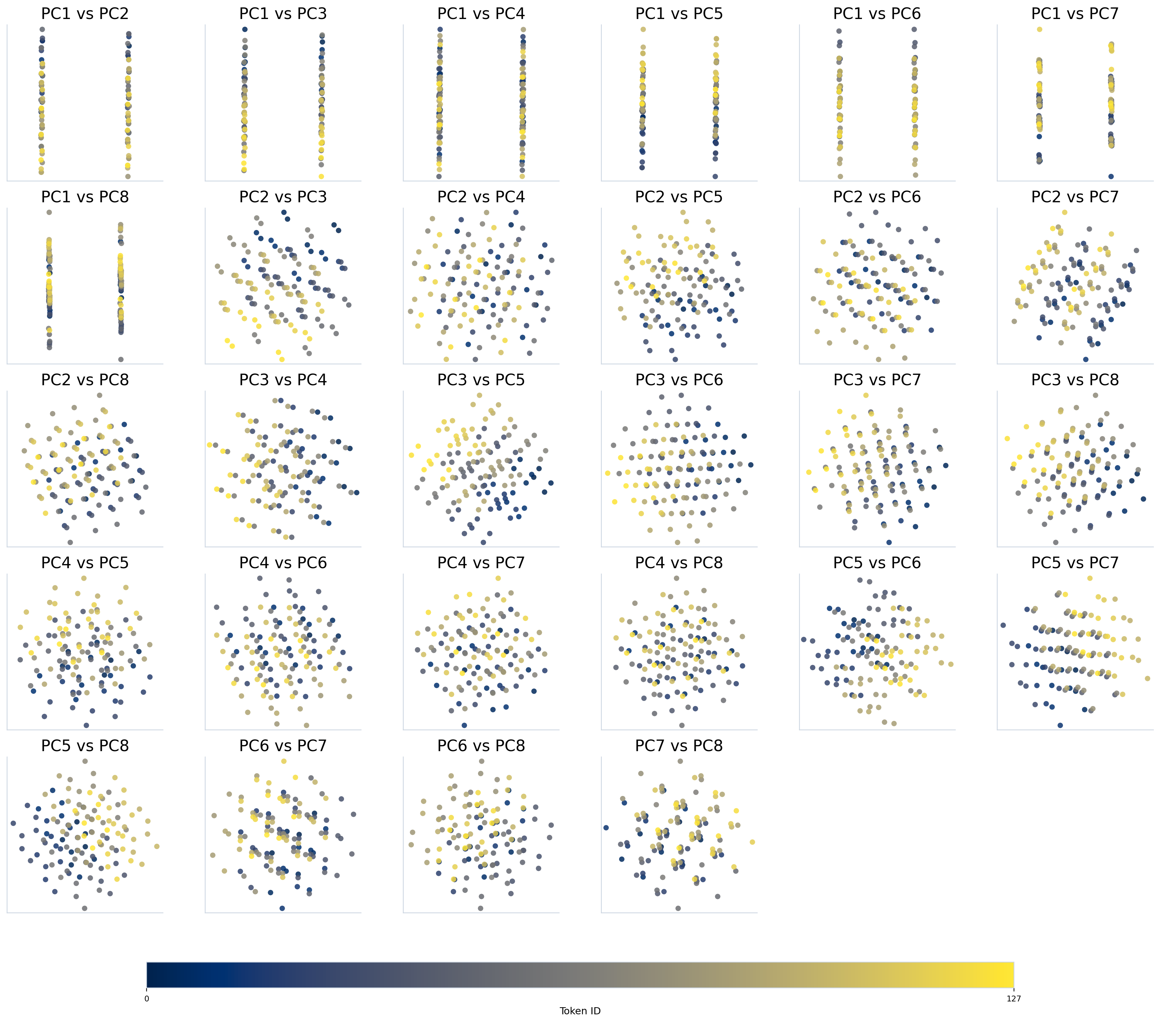}
    \caption{PCA projection of embeddings of model trained on the \textsc{Hypercube} graph metric completion.}
\end{figure}

\newpage

\subsubsection{2D Lattice}
\begin{figure}[H]
    \centering
    \includegraphics[height=0.41\textheight]{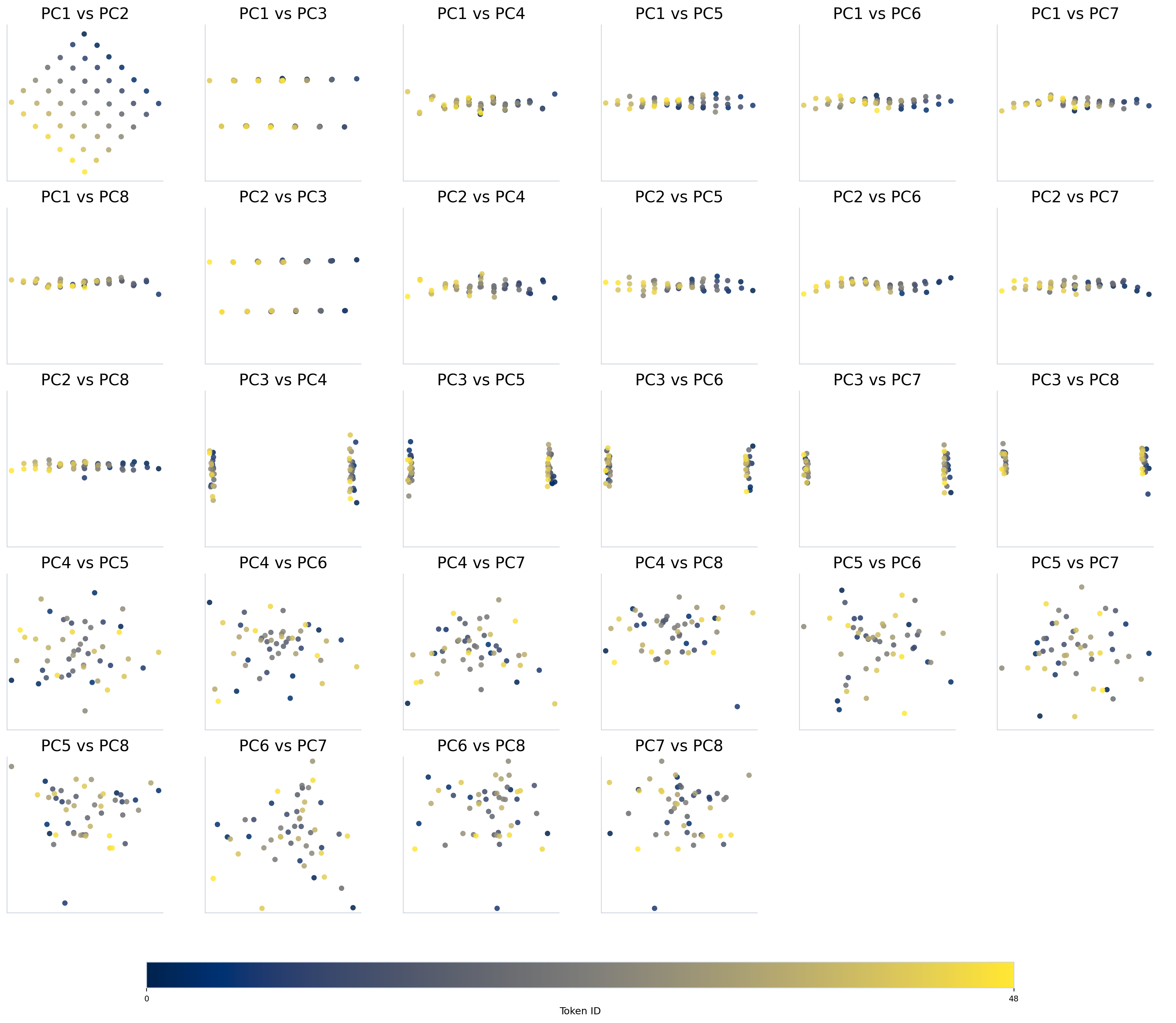}
    \caption{PCA projection of embeddings of model trained on the \textsc{2D Lattice} graph metric completion.}
\end{figure}

\subsubsection{3D Lattice}
\begin{figure}[H]
    \centering
    \includegraphics[height=0.41\textheight]{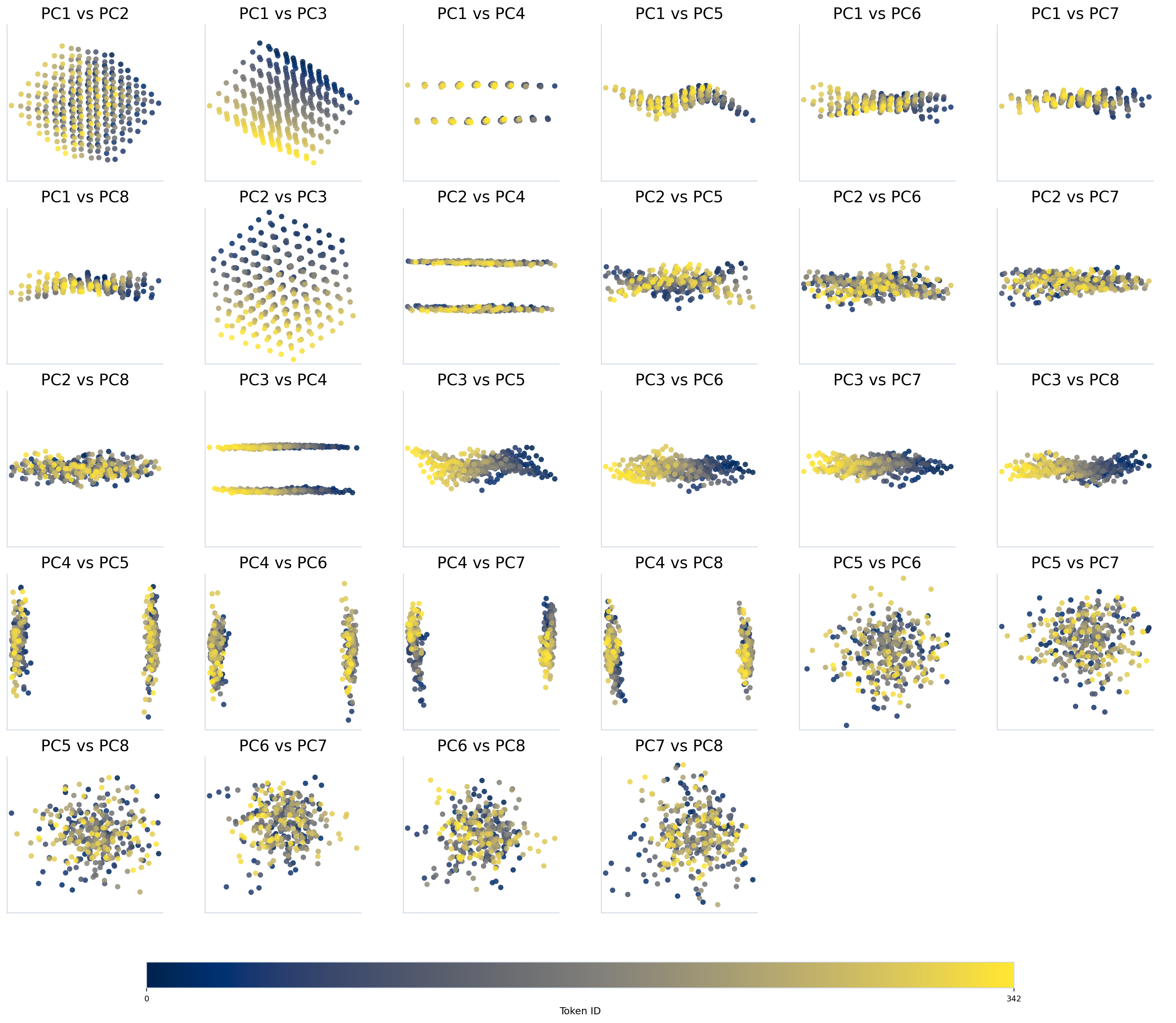}
    \caption{PCA projection of embeddings of model trained on the \textsc{3D Lattice} graph metric completion.}
\end{figure}

\subsection{Comparison}

\subsubsection{2D Comparison}
\begin{figure}[H]
    \centering
    \includegraphics[width=0.6\linewidth]{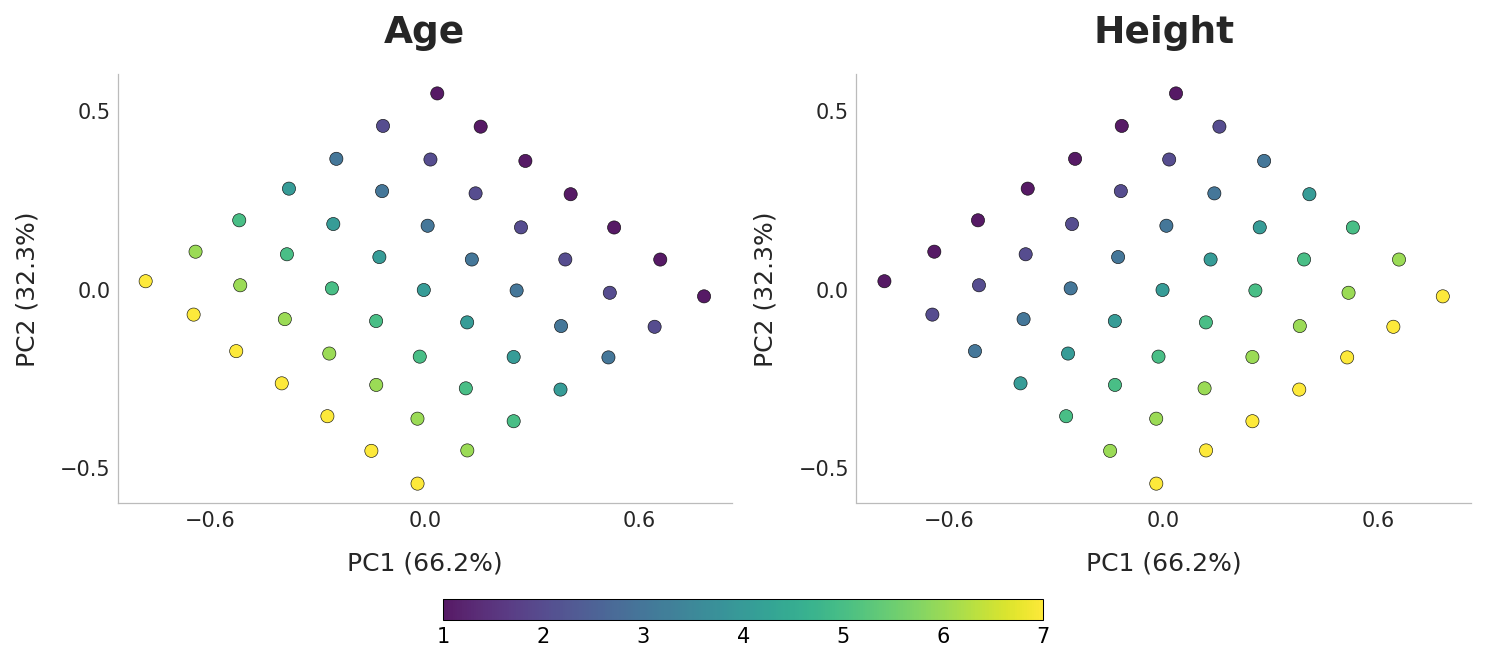}
    \caption{PCA projection of embeddings of model trained on the \textsc{2D Comparison}.}
\end{figure}

\subsubsection{3D Comparison}
\begin{figure}[H]
    \centering
    \includegraphics[width=0.7\linewidth]{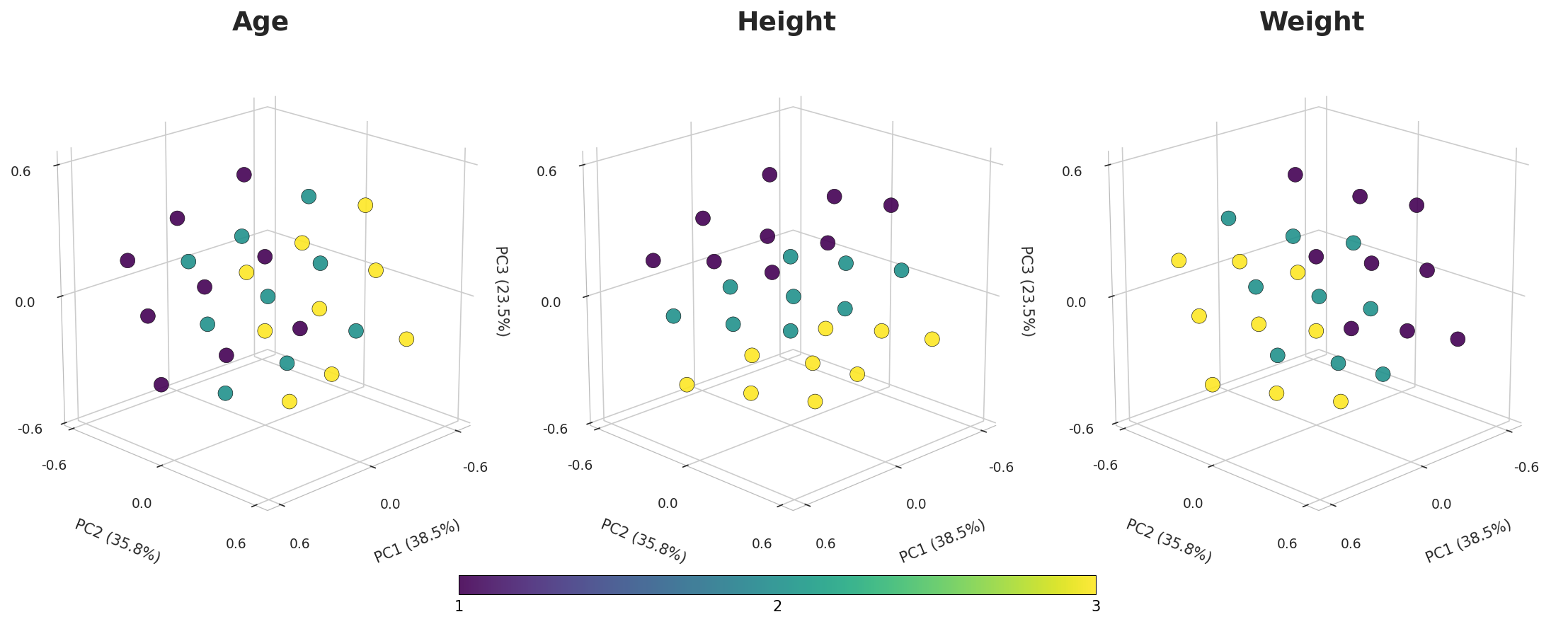}
    \caption{PCA projection of embeddings of model trained on the \textsc{3D Comparison}.}
\end{figure}

\newpage

\section{Accelerating Tasks using Symmetric Regularization}
\label{appendix:faster}
\subsection{Modular Arithmetic}

Note that Consistency refers to joint enforcement of commutativity and associativity.

\subsubsection{Addition: $x+y$}
\begin{figure}[H]
    \centering
    \includegraphics[width=0.82\linewidth]{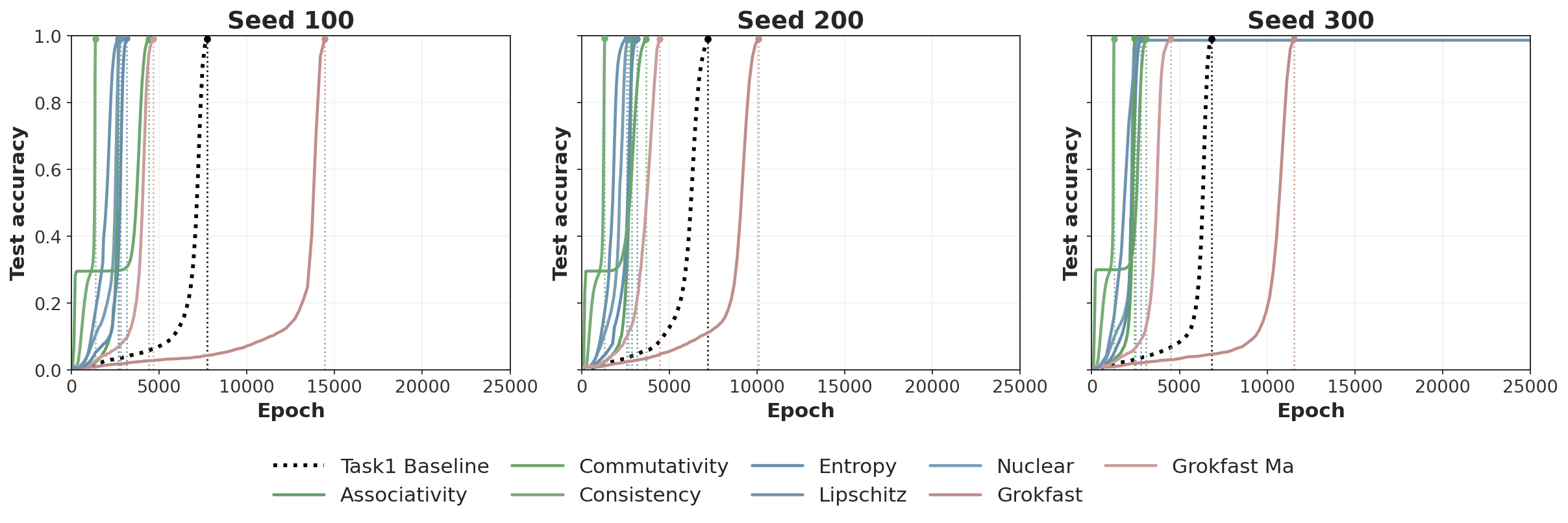}
    \caption{Convergence speed on modular addition when using various regularization.}
\end{figure}

\subsubsection{Subtraction: $x-y$}
\begin{figure}[H]
    \centering
    \includegraphics[width=0.85\linewidth]{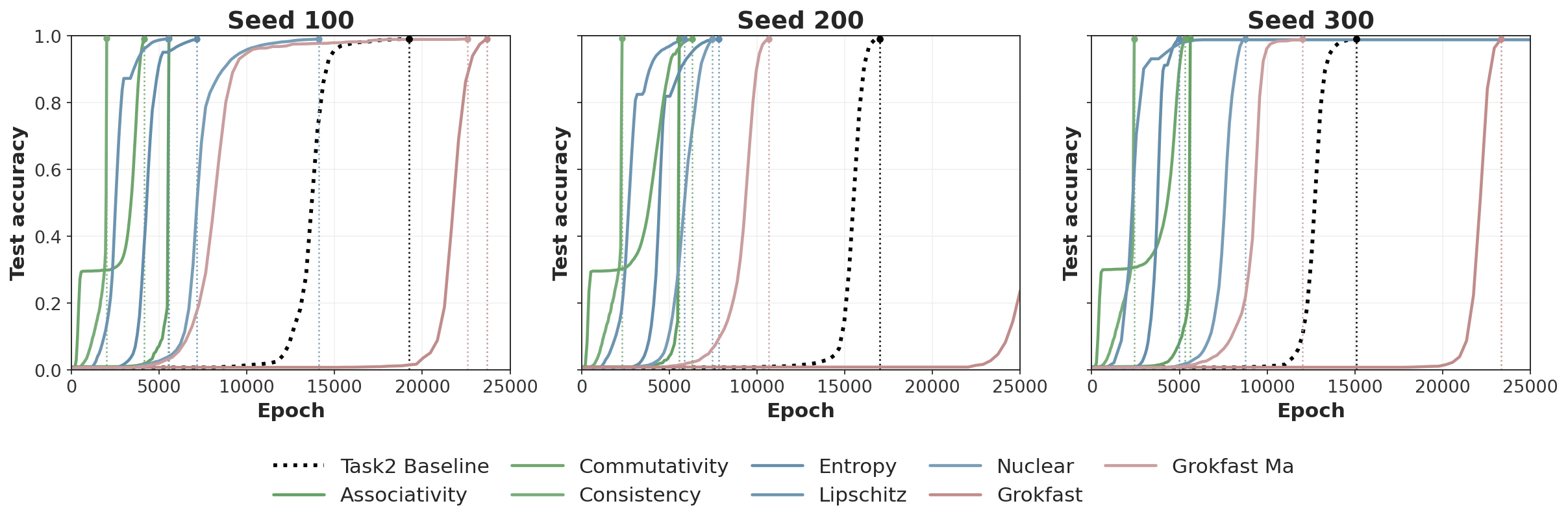}
    \caption{Convergence speed on modular subtraction when using various regularization.}
\end{figure}

\subsubsection{Squared sum: $(x+y)^2$}
\begin{figure}[H]
    \centering
    \includegraphics[width=0.82\linewidth]{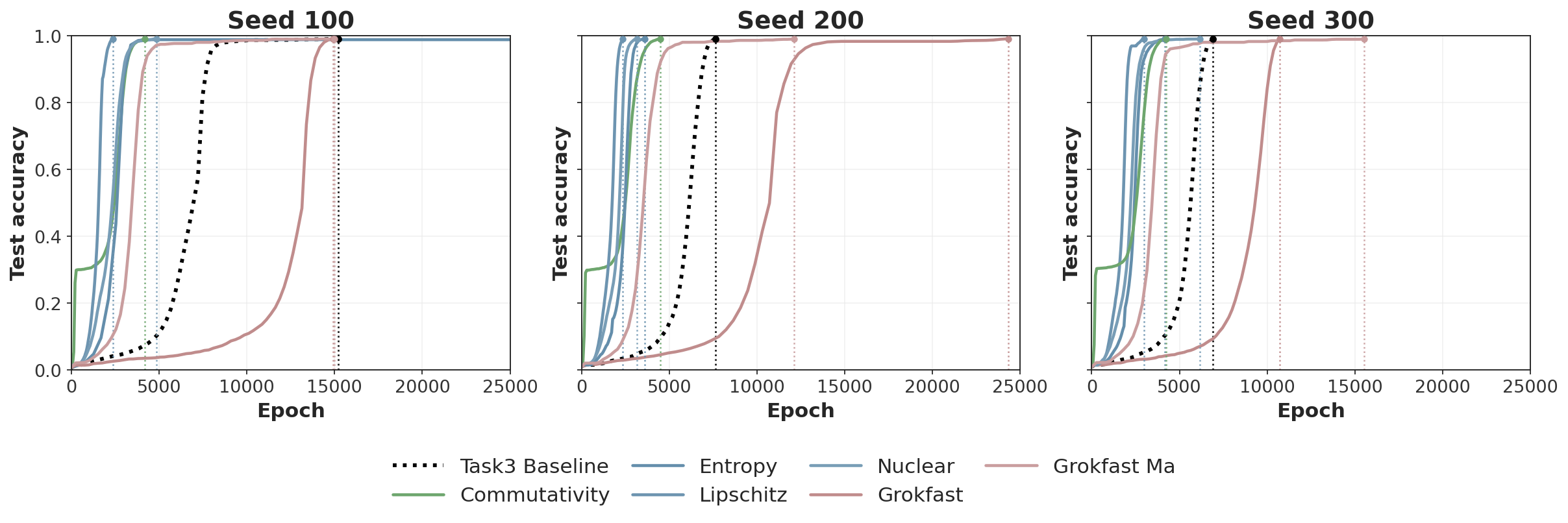}
    \caption{Convergence speed on modular squared sum when using square sum.}

\end{figure}

\newpage

\subsubsection{Cubed sum: $(x+y)^3$}
\begin{figure}[H]
    \centering
    \includegraphics[width=0.82\linewidth]{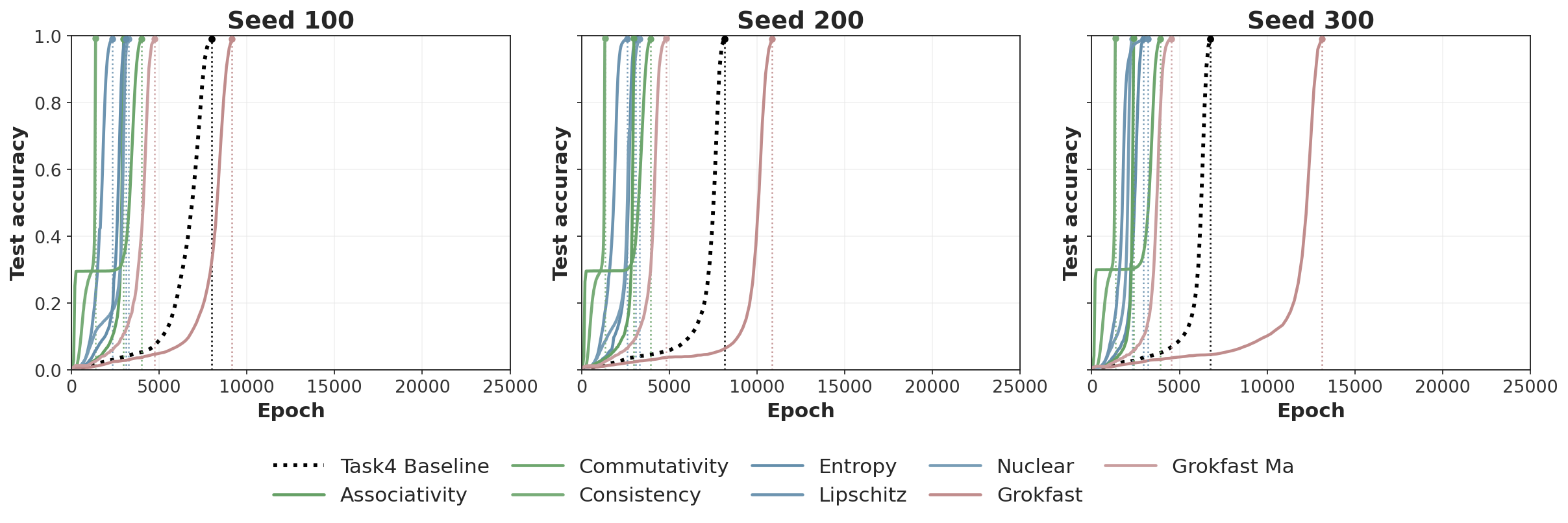}
    \caption{Convergence speed on modular cubed sum when using various regularization.}
\end{figure}


\subsubsection{Multiplication: $x\cdot y$}
\begin{figure}[H]
    \centering
    \includegraphics[width=0.9\linewidth]{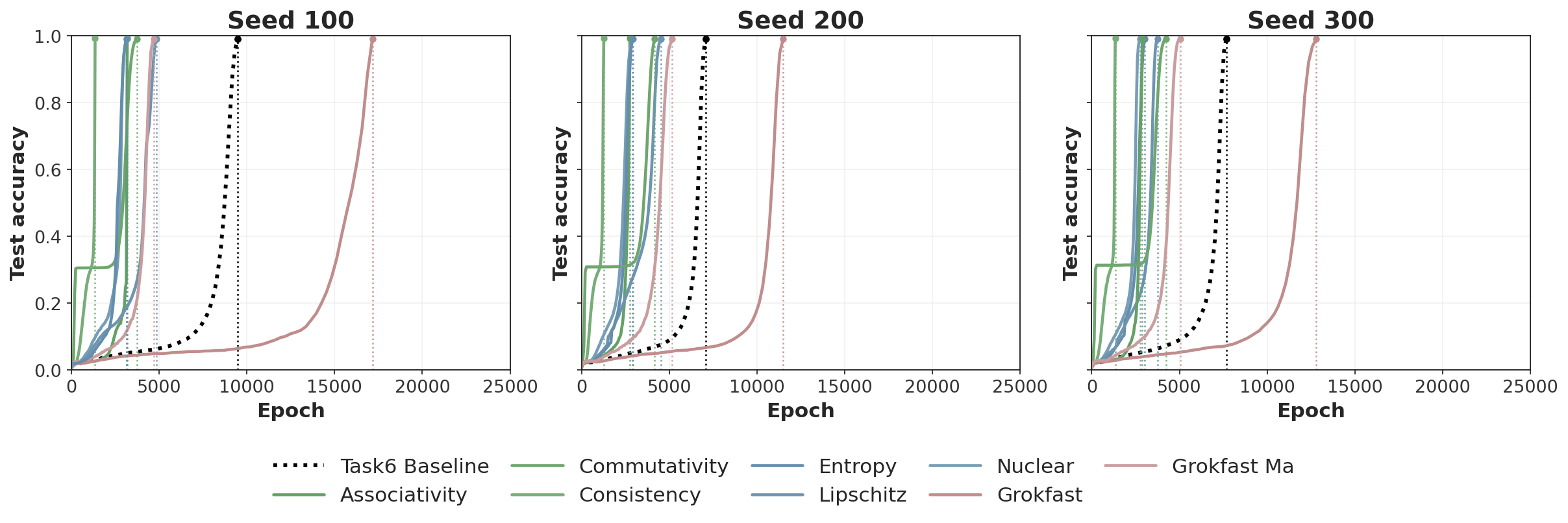}
    \caption{Convergence speed on modular multiplication when using various regularization.}
\end{figure}

\subsubsection{Affine--bilinear: $x + xy + y$}
\begin{figure}[H]
    \centering
    \includegraphics[width=0.85\linewidth]{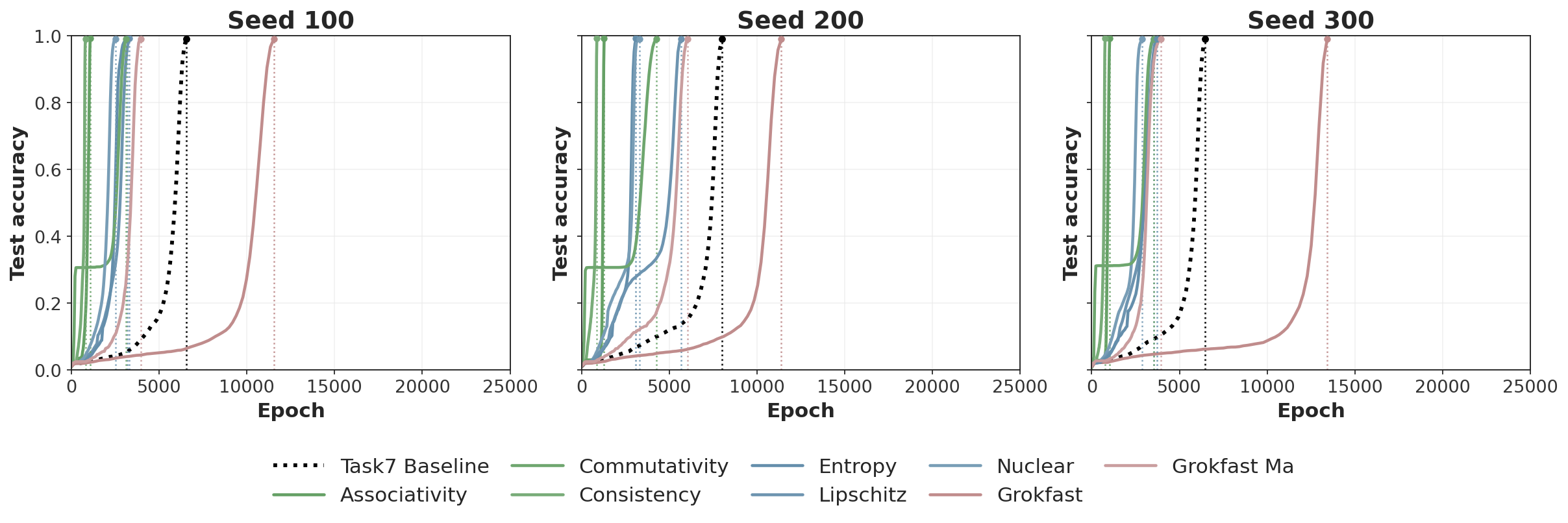}
    \caption{Convergence speed on modular affine--bilinear when using various regularization.}

\newpage

\end{figure}

\subsection{graph metric completion}

\subsubsection{Cycle}
\begin{figure}[H]
    \centering
    \includegraphics[width=0.85\linewidth]{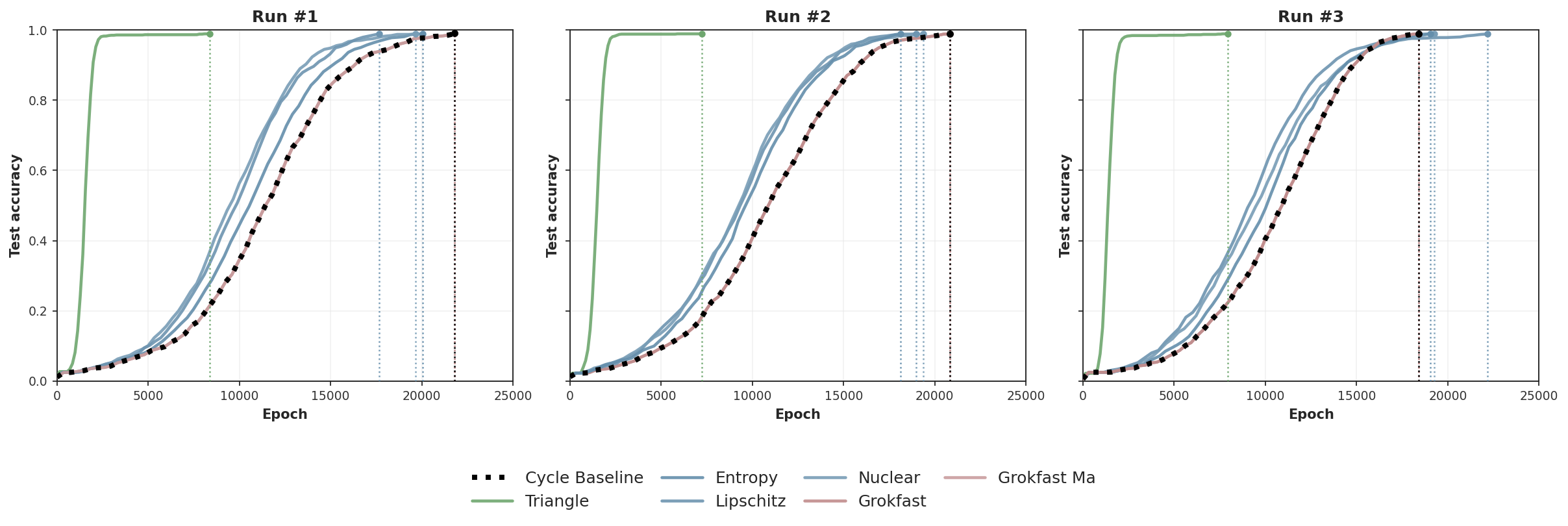}
    \caption{Convergence speed on \textbf{cycle} graph metric completion when using various regularization.}
\end{figure}

\subsubsection{Path}
\begin{figure}[H]
    \centering
    \includegraphics[width=0.85\linewidth]{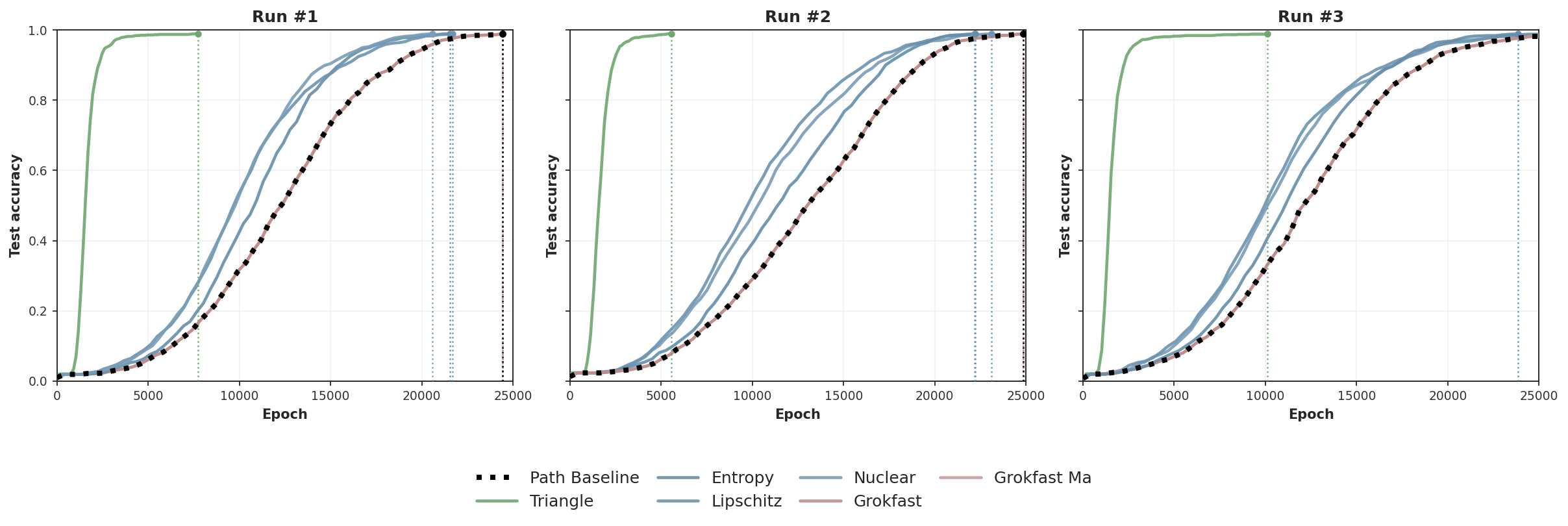}
    \caption{Convergence speed on \textbf{path} graph metric completion when using various regularization.}
\end{figure}

\subsubsection{Cylinder}
\begin{figure}[H]
    \centering
    \includegraphics[width=0.85\linewidth]{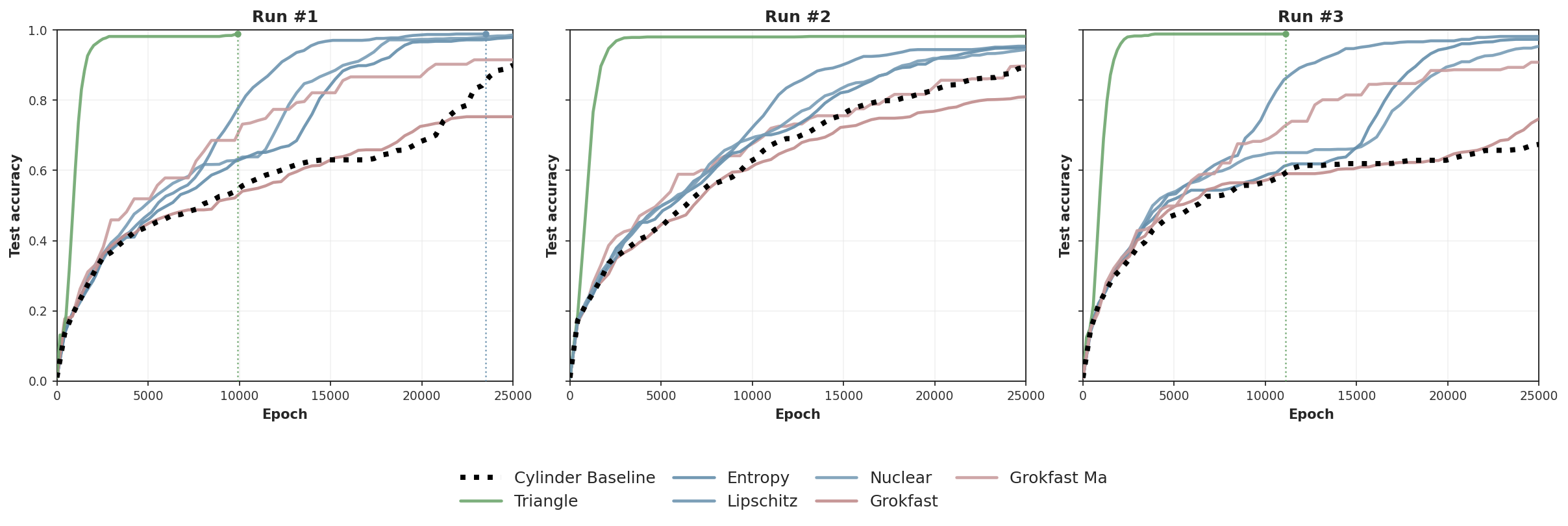}
    \caption{Convergence speed on \textbf{cylinder} graph metric completion when using various regularization.}
\end{figure}

\newpage

\subsubsection{Hypercube}
\begin{figure}[H]
    \centering
    \includegraphics[width=0.85\linewidth]{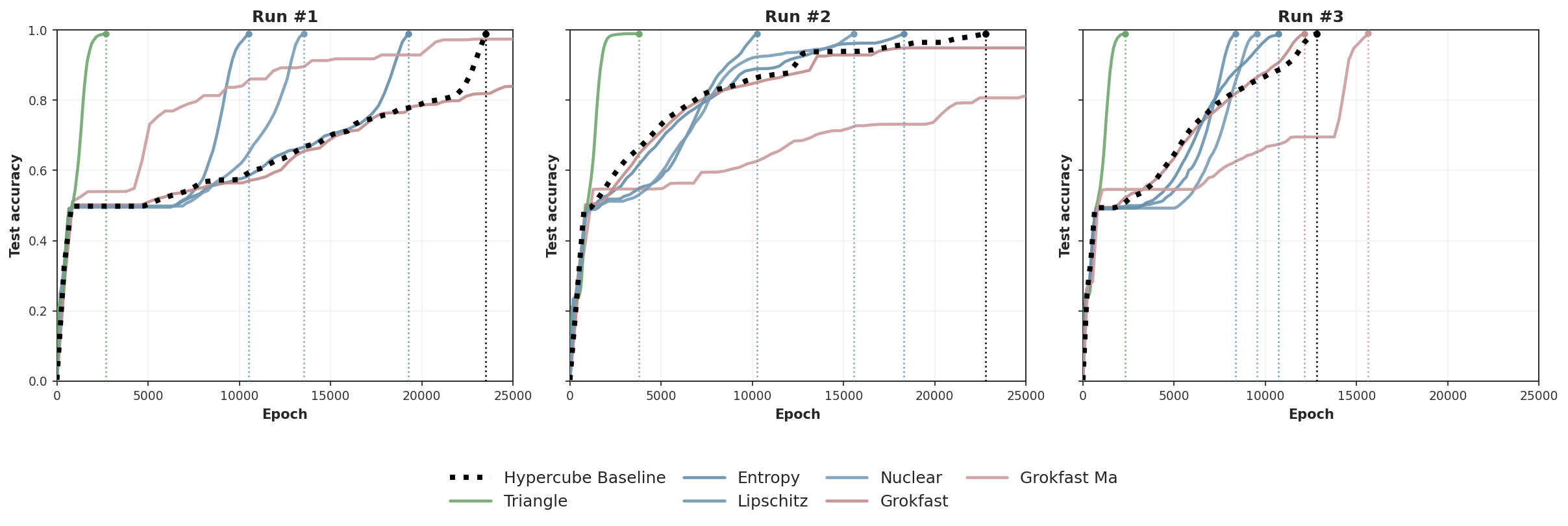}
    \caption{Convergence speed on \textbf{hypercube} graph metric completion when using various regularization.}

\end{figure}

\subsubsection{2D Lattice}
\begin{figure}[H]
    \centering
    \includegraphics[width=0.85\linewidth]{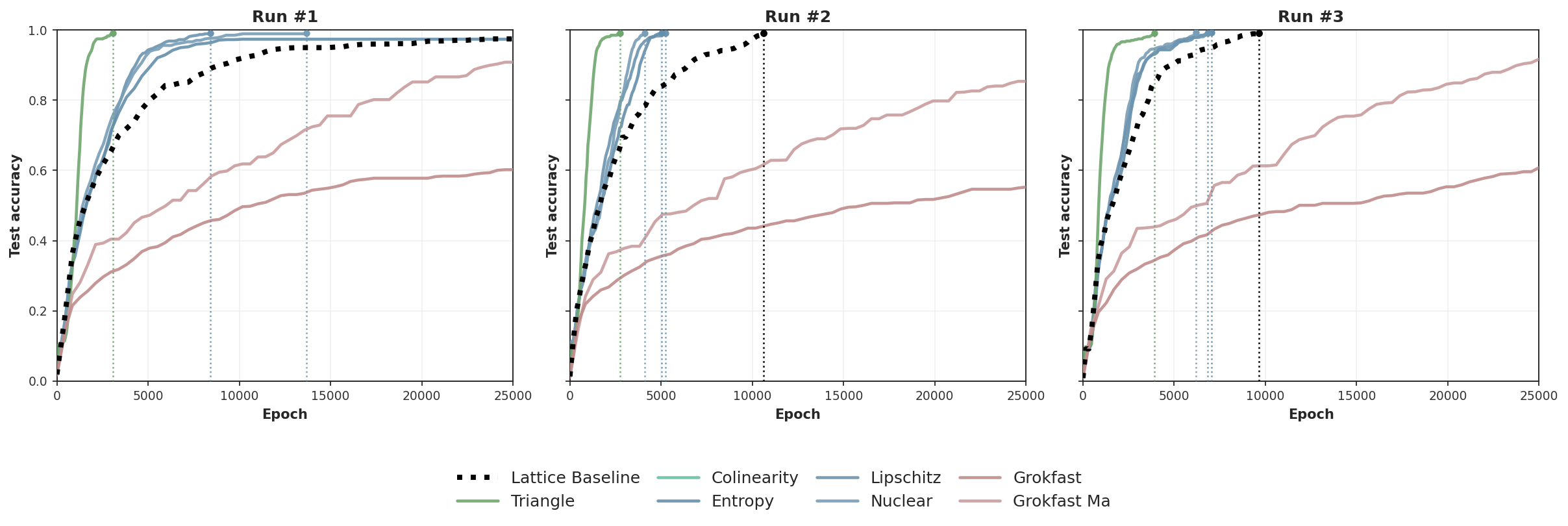}
    \caption{Convergence speed on \textbf{2d lattice} graph metric completion when using various regularization.}
\end{figure}

\subsubsection{3D Lattice}
\begin{figure}[H]
    \centering
    \includegraphics[width=0.85\linewidth]{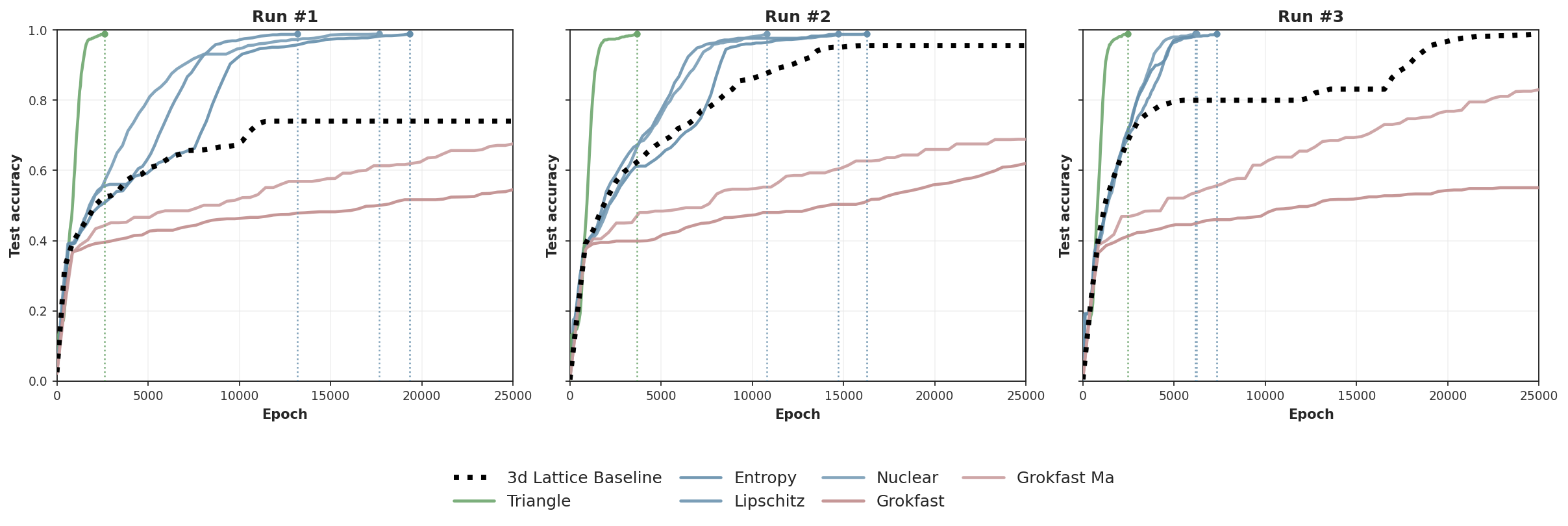}
    \caption{Convergence speed on \textbf{3d lattice} graph metric completion when using various regularization.}
\end{figure}

\newpage

\subsection{Comparison}
\subsubsection{2D Comparison}
\begin{figure}[H]
    \centering
    \includegraphics[width=0.85\linewidth]{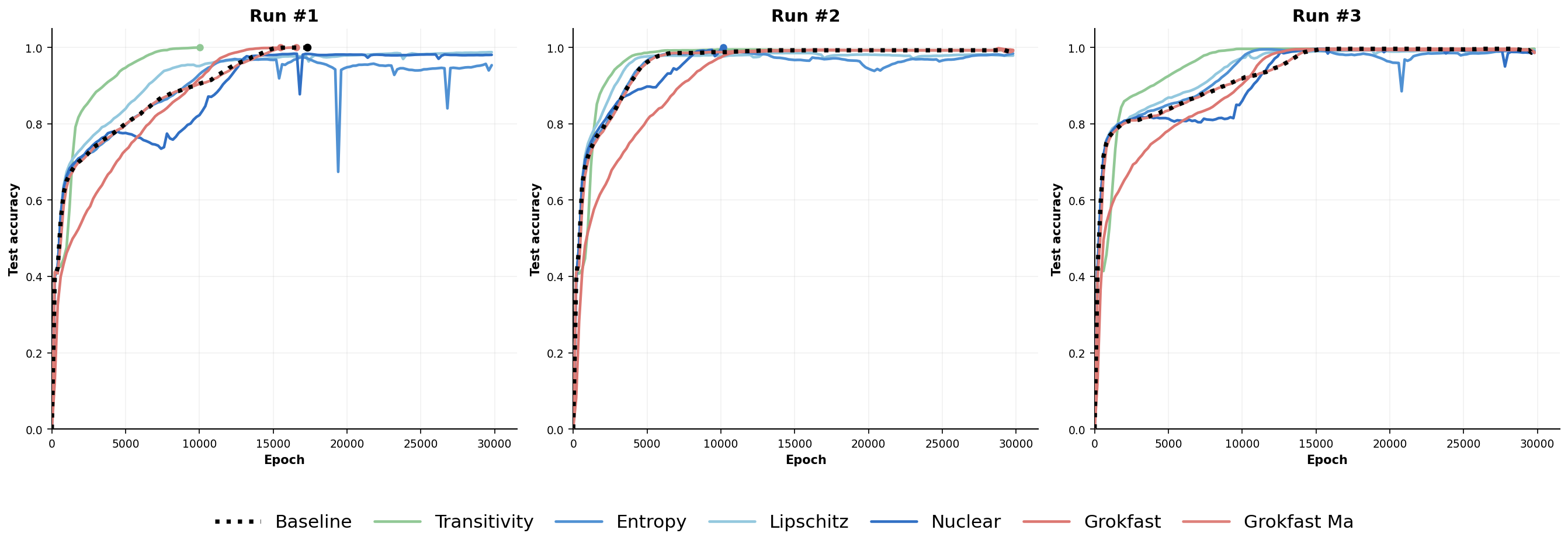}
    \caption{Convergence speed on \textbf{Comparison (2D)} when using various regularization.}
\end{figure}

\subsubsection{3D Comparison}
\begin{figure}[H]
    \centering
    \includegraphics[width=0.85\linewidth]{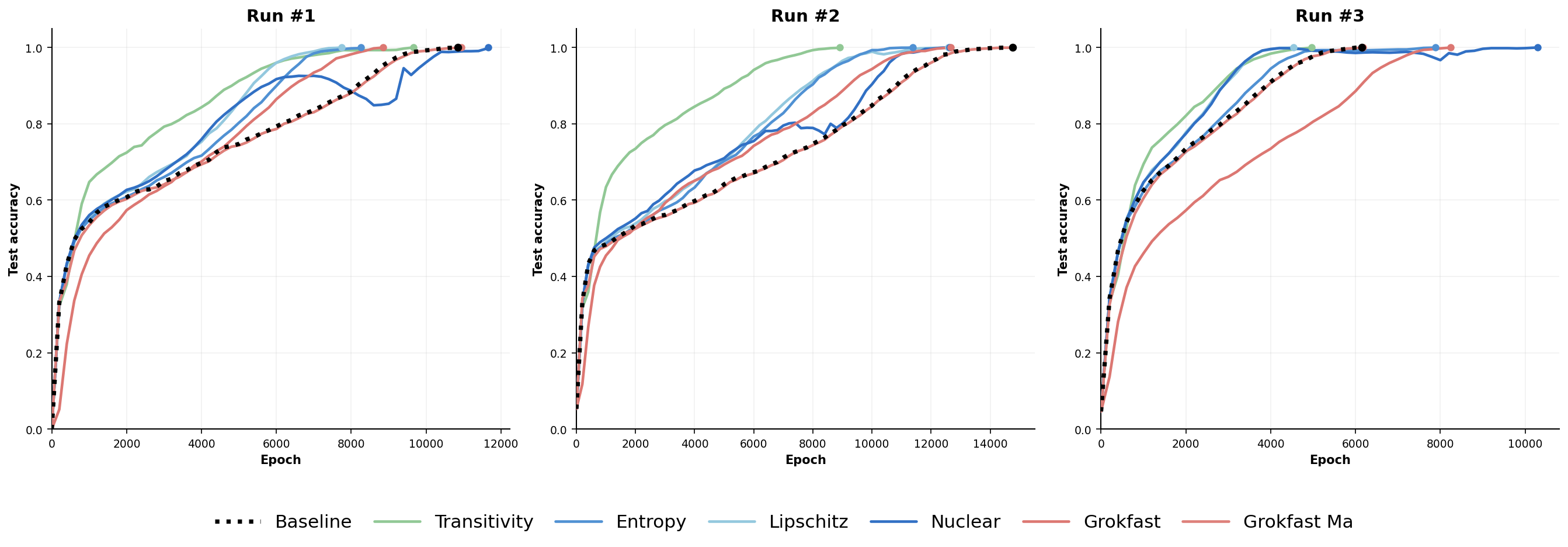}
    \caption{Convergence speed on \textbf{Comparison (3D)} when using various regularization.}
\end{figure}

\newpage
\section{Ablations on Alternative Explanations for Grokking}
\label{Appendix:alternatives}
We investigate other factors previously noted to affect grokking and its emergence timing, specifically \textbf{output rescaling} and \textbf{initialization scaling}. Following \tcb{(Kumar, 2023)}, we scale the network output by a factor $\alpha$ to modulate the transition between \textit{lazy} and \textit{rich} dynamics; And similar to \tcb{(Liu, 2023)}, we vary the initial weight norm of the embedding matrix by adjusting the standard deviation of a Gaussian distribution.  Both works demonstrate that these factors can push models into a regime that favors initial memorization, either through infinitesimal weight updates in the \textit{lazy} regime or by starting far from the generalizing weight norm. 

However, we show that our symmetry-enforcing objectives (commutativity + associativity losses) consistently accelerate generalization across various ranges of these factors. In this light, our work is complementary to previous interpretations: while scaling factors primarily modulate the \textit{laziness} or initial state of feature learning, intrinsic task symmetry provides the necessary \textbf{geometric direction} to move beyond memorization, which ultimately provides a structural explanation for generalization in algorithmic reasoning.

\begin{figure}[ht]
    \centering
    \begin{subfigure}[t]{0.45\linewidth}
        \centering
        \includegraphics[width=\linewidth]{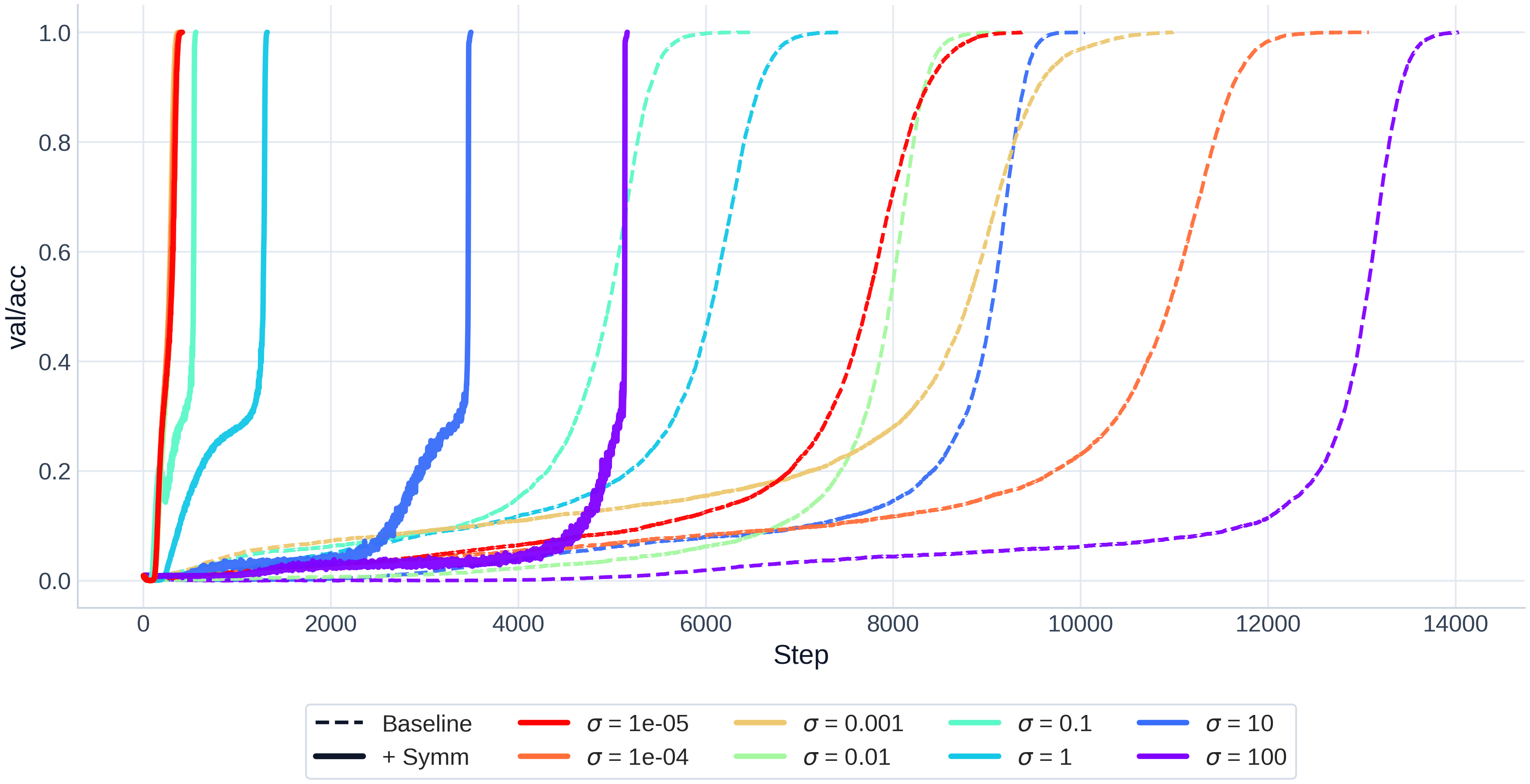}
        \caption{Embedding initialization}
        \label{fig:emb_init}
    \end{subfigure}
    \hspace{0.02\linewidth}    
    \begin{subfigure}[t]{0.45\linewidth}
        \centering
        \includegraphics[width=\linewidth]{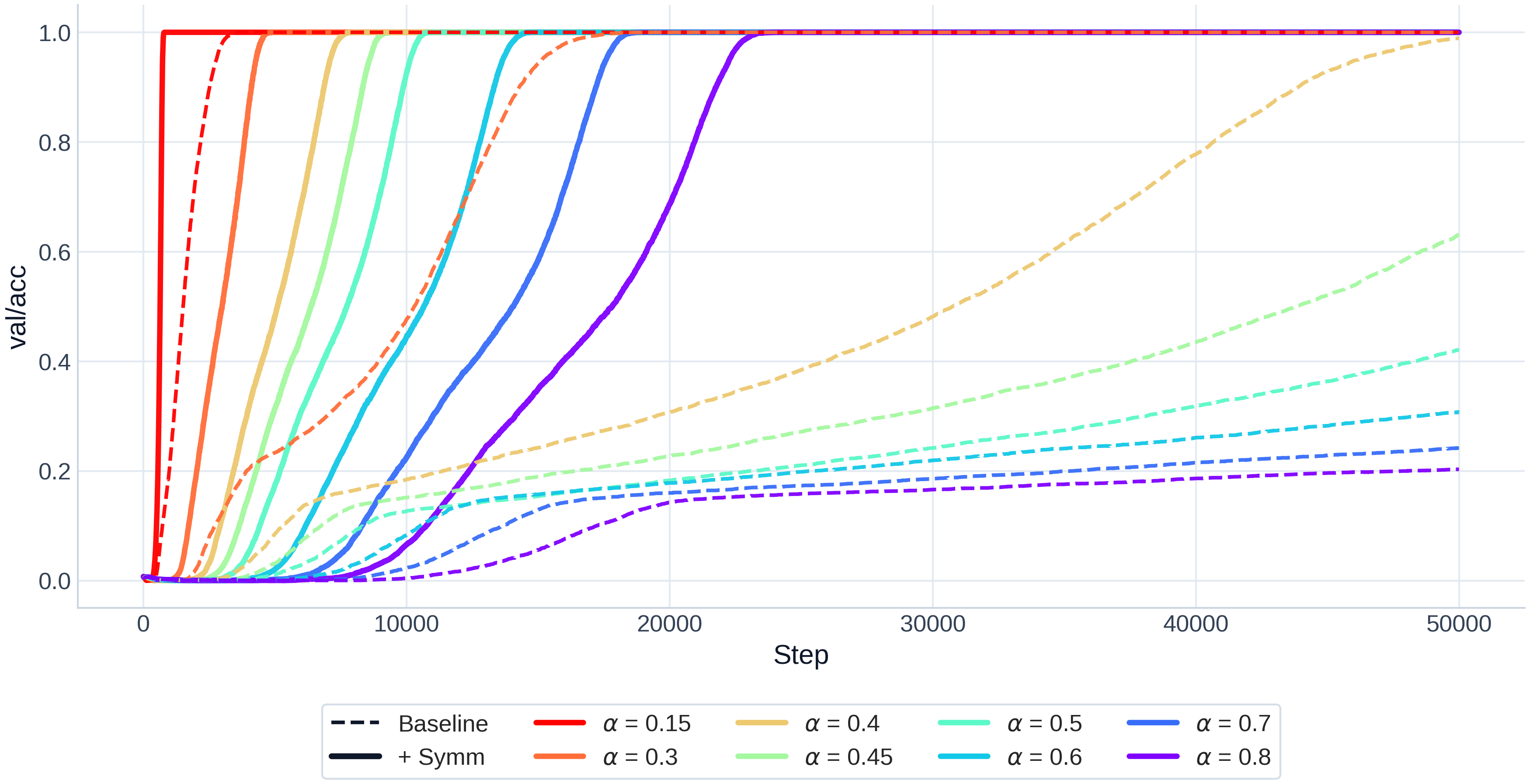}
        \caption{Output scaling}
        \label{fig:output_scale}
    \end{subfigure}
    \caption{Comparison of embedding initialization and output scaling on Modular Addition, 1-Layer Transformer}
    \label{fig:appendix_comparison}
\end{figure}

\newpage

\section{Training Configuration}
\label{app:config}
\textbf{Model Configuration.} We employ a standard single-layer ($L=1$) decoder-only Transformer with $H=4$ attention heads. The architecture uses an embedding dimension of $d_{\text{model}} = 128$, a feed-forward (MLP) dimension of $d_{\text{mlp}} = 512$, and a head dimension of $d_{\text{head}} = 32$. All models utilize ReLU activations and omit Layer Normalization.

\textbf{Training } All models are trained for a maximum of $100,000$ epochs using the AdamW optimizer. 

For each task, sepcifically, we employ following configurations:

\noindent \textbf{Modular Arithmetic.} Following standard protocols, we use modulus $p=113$, a training split of $r_{\text{train}}=0.3$, and a weight decay of $w = 1$.

\noindent \textbf{Graph Metric Completion.} We increase the training split to $r_{\text{train}}=0.8$ and weight decay to $w = 3$. We evaluate performance across six distinct graph structures:

\vspace{0.5em}
\noindent   
\renewcommand{\arraystretch}{1.2}
\begin{tabular}{lll}
(i) 2D Grid Lattices ($7 \times 7$) & (ii) 3D Grid Lattices ($7 \times 7 \times 7$) & (iii) Simple Cycles ($N=64$) \\
(iv) Path Graphs ($N=64$) & (v) Cylinders ($7 \times 7$ periodic) & (vi) 7D Hypercubes ($2^7$ nodes) \\
\end{tabular}
\vspace{0.5em}

\noindent \textbf{Comparison.} we use a reduced training split of $r_{\text{train}} = 0.2$ with weight decay $w = 1$. The dataset consists of entities arranged in either 2D ($7 \times 7$) or 3D ($3 \times 3 \times 3$) structures. Each entity is assigned a unique combination of attributes (height, weight, age), sampled such that values lie within a bounded range $[v_{\min}, v_{\max}]$.

\textbf{Acceleration.} For the acceleration experiments, we set the weighting coefficient to $\lambda = 1.0$. The total training objective is thus defined as the unweighted sum of the primary and auxiliary losses:

$$\mathcal{L}_{\text{total}} = \mathcal{L}_{\text{CE}} + \mathcal{L}_{\text{additional}}$$

%% file: appendix_theory.tex
\begin{table}[h]
\centering
\caption{Predicted embedding geometry and associated algorithm of various modular arithmetic tasks. Note that $q \in \mathbb{N}, \; \gcd(p,q)=1 $ and  $g$ is the generator of multiplicative group $\mathbb{Z}_p^\times$, $\ind_g(\cdot)$ denote the discrete logarithm with base $g$, and $\log(\cdot)$ denote the complex logarithm.  }

\small
\renewcommand{\arraystretch}{1.6} 
\begin{adjustbox}{max width=0.95\textwidth}
\begin{tabular}{@{} cccc @{}}
\toprule
\textbf{Task} & \textbf{Embedding function and associated algorithm} & \textbf{Embedding geometry} & \textbf{Notes}  \\
\midrule
$x+y=z\pmod p$
 & \ka{$\phi(x)=e^{ 2\pi i q x/p} $\\
      $\psi(x)=p \log(x) / (2 \pi i q)$\\
      $z = \psi(\phi(x)\phi(y))$}
 & \ka{ Circle (closed helix) \\
 $\angle(i,i+1) = \frac{2 \pi q}{p}$} & \ka{  $\phi(x)\phi(y)=\phi(z)$ }  \\
\grayrule
$x-y=z\pmod p$
 & \ka{$\phi(x)=e^{ 2\pi i q   x / p}$\\
      $\psi(x)=p \log(x) / (2 \pi i q)$\\
      $z= \phi(x)\phi(y)^{-1} $}
  & \ka{ Circle (closed helix) \\
 $\angle(i,i+1) = \frac{2 \pi q}{p}$}   & $\phi(x)\phi(y)^{-1}=\phi(z)$ \\
\grayrule
$(x+y)^{\alpha}=z$ $\pmod p$ $ \alpha \in \mathbb{N}$
 & \ka{$\phi(x)=e^{ 2\pi i q x /p}$\\
      $\psi(x)= h_\alpha (p \log (x)/(2\pi i q) )$\\
      $z= \phi(x)\phi(y)$}
  & \ka{ Circle (closed helix) \\
 $\angle(i,i+1) = \frac{2 \pi q}{p}$}  & \ka{ $h_\alpha: \mathbb{Z}_p^\times \to \mathbb{Z}_p^\times, $ \\
 $ h_\alpha(x) : x \mapsto x^\alpha \pmod{p}$ \\ $\phi(x)\phi(y)=\phi(z)$ } \\
\grayrule
$x \times y=z\pmod p$
 & \ka{
      $\phi(x)=e^{ 2 \pi i q \cdot \ind_g(x) /p}$\\
      $\psi(x)=g^{p \cdot \log(x) / (2 \pi i q )}$\\
      $z = \psi( \phi(x)\phi(y))$}
 & \ka{ Circle (closed helix) \\
 $\angle(g^i,g^{i+1}) = \frac{2 \pi q}{p}$}  & \ka{$\phi(x)\phi(y)=\phi(z)$} \\
\grayrule 
$x+xy+y=z\pmod p$
 & \ka{$\phi(x)=e^{ 2 \pi i q \cdot \ind_g(x+1) /p}$\\
      $\psi(x)=g^{p \cdot \log(x) / (2 \pi i q )}$\\
      $z = \psi( \phi(x)\phi(y)) - 1$}
 & \ka{ Circle (closed helix) \\
 $\angle(g^i - 1,g^{i+1} - 1) = \frac{2 \pi q}{p}$} & \ka{$\phi(x)\phi(y)=\phi(z)$} \\
\bottomrule
\end{tabular}
\end{adjustbox}
\label{tab:predicted_embedding}
\end{table}

\section{Symmetries in modular arithmetic}
\label{app:sym_modular}
We present the predicted embedding geometry of various modular arithmetic tasks. For modular addition, the exact algorithm using trigonometric identities has been reported in \citep{nanda2023progress,zhong2023clockpizzastoriesmechanistic,park2024accelerationgrokkinglearningarithmetic}.
If we consider the embedding function $\phi : \mathbb{Z}_p \to \mathbb{C}$ and outer function $\psi : \mathbb{C} \to \mathbb{Z}_p$, there is a simple identity to formulate modular addition.
\begin{align*}
\phi(x) &= e^{2\pi i qx/p}, \\[4pt]
\psi(x) &= \tfrac{p}{2\pi i q}\log(x),\qquad q\in\mathbb{Z}_p,\ \gcd(p,q)=1, \\[8pt]
z &= \psi(\phi(x)\phi(y)), \\[4pt]
z &= \psi\!\left(e^{2\pi i q(x+y)/p}\right) = x+y \pmod{p}.
\end{align*}

where $\log(\cdot)$ denotes complex logarithm. The embedding $\phi(\cdot)$ says that the embedding geometry $\{\phi(x) : x \in \mathbb{Z}_p \}$ have shape of the circle with the consistent angle between adjacent tokens being $ \frac{2 \pi q }{p}$. \\
For modular multiplication, $\mathbb{Z}_p ^\times=\mathbb{Z}_p \setminus \{0\}$ become cyclic group with some generator $g \in \mathbb{Z}_p^\times$ (\textit{i.e.} $\mathbb{Z}_p ^\times=\{ g^i : i=0,1,...,p-1\}$ ). Hence we have embedding function
$\phi : \mathbb{Z}_p^\times \to \mathbb{C}$ and outer function $\psi : \mathbb{C} \to \mathbb{Z}_p^*$, there is a simple identity to formulate modular addition.
\begin{align*}
\phi(x) &= e^{2\pi i q\,\ind_g(x)/p}, \\[4pt]
\psi(x) &= \tfrac{p}{2\pi i q}\log(x),\qquad q\in\mathbb{Z}_p,\ \gcd(p,q)=1, \\[8pt]
z &= \psi(\phi(x)\phi(y)), \\[4pt]
z &= \psi\!\left(e^{2\pi i q(\ind_g(x)+\ind_g(y))/p}\right) = x\times y \pmod{p}.
\end{align*}
where $\ind_g(\cdot)$ denotes the discrete logarithm with base $g$ in $\mathbb{Z}_p$. For $T_6: z= x+xy+y \pmod{p}$, we can rewrite it as $z+1=(x+1)(y+1) \pmod{p}$. Hence the embedding is similar to multiplication with translation by $1$.

Similarly, we present the predicted embedding geometry and algorithm of other various modular arithmetic tasks which are presented in Table \ref{tab:predicted_embedding} in Appendix \ref{app:sym_modular}.

We present the embedding geometry and proxy symmetries of various modular arithmetic in Tables \ref{tab:predicted_embedding} and \ref{tab:sym_modular}.

\section{Theoretical observation}
\label{app:theory}

\subsection{Theoretical observation on modular arithmetic}

First, we state the Pontryagin duality theorem in Lie groups. 

\begin{definition}[Pontryagin dual]
    For a locally compact abelian group $G$, the Pontryagin dual is the group $\hat{G}$ of continuous group homomorphisms from $G$ to the circle group $\mathbb{T}$, \textit{i.e.}
    \[ \hat{G} := \text{Hom}(G,\mathbb{T}).\]
\end{definition}

\begin{lemma}
    The Pontryagin dual of $\mathbb{T}$ is $\mathbb{Z}$ via the map
    \begin{align*}
        \Phi :  \mathbb{Z} &\cong \hat{\mathbb{T}} \\ 
        n &\mapsto \left( z \mapsto  e^{2 \pi i n z} \right).
    \end{align*}
\end{lemma}

\begin{lemma}
    The Pontryagin dual of $\mathbb{Z}$ is $\mathbb{T}$ via the map
    \begin{align*}
        \Phi :  \mathbb{T} &\cong \hat{\mathbb{Z}} \\ 
        e^{i \theta} &\mapsto \left( n \mapsto  e^{i n \theta} \right).
    \end{align*}
\end{lemma}

\begin{theorem}[Pontryagin duality \citep{pontrjagin1934theory,morris1977pontryagin,hewitt2013abstract}]
    Let $G$ be a locally compact abelian topological group. Then there is a canonical isomorphism such that 
    \[ G \cong \hat{\hat{G}}, \]
where $\hat{\hat{G}}$ is Pontryagin double dual of $G$. 
\label{thm:pontryagin}
\end{theorem}

If topological abelian group $G$ is one-dimensional, $G$ is isomorphic to the closed helix. 

\subsection{Proof of Proposition \ref{prop:abelian}}

\begin{proposition*}
Let $\mathcal{M}$ be an one-dimensional compact abelian topological group and $\mathbb{Z}_p$ is continuously embedded in $ \mathcal{M}$. Then $\mathcal{M}$ is isomorphic to $S^1$ embedded in high-dimensional torus $\mathbb{T}^D$ (closed helix) 
\[ \Phi: \mathcal{M} \cong  S^1 \hookrightarrow \mathbb{T}^D.\]
Hence $\Phi(\mathbb{Z}_p) \hookrightarrow \Phi(\mathcal{M})$ is also confined in the closed helix. 
\end{proposition*}
\begin{proof}
    
 Note that the one-dimensional only connected, compact, abelian topological group is the circle group $ \mathbb{T} \cong \mathbb{R}/\mathbb{Z} \cong [0,1] /\{0 \sim 1\}$. By Theorem \ref{thm:pontryagin}, we have 
\[ \Phi: \mathbb{T} \cong \hat{\hat{\mathbb{T}}} \cong \hat{\mathbb{Z} }\]
via the map 
\[ \Phi(x) =  \left( n \mapsto e^{2 \pi i n x} \right) , \quad n \in \mathbb{Z}  .  \]
Therefore, 

\[ \Phi(\mathbb{T}) =  \prod_{n \in \mathbb{Z}} n\mathbb{T} \subset \mathbb{T}^\infty, \]
where $n$-th each component denotes the rotated by $n$ times in infinite-dimensional torus. We note that $\Phi(\mathbb{T})$ can be embedded into some high-dimensional torus $\mathbb{T}^D$ 
\[ \Phi(\mathbb{T}) \cong (n_1\mathbb{T}, n_2\mathbb{T}, \dots , n_D\mathbb{T}) \subset \mathbb{T}^D, \]
if $\gcd(n_1, \dots , n_D) = 1$. Therefore, $\mathcal{M}$ is isomorphic to $S^1 \hookrightarrow \mathbb{T}^D $ in some high-dimensional torus $\mathbb{T}^D$. 

\end{proof}

\subsection{Theoretical observation on graph metric completion and comparison tasks}
\label{appendix:theory_graph}

In this section, we describe the embedding geometry in graph completion and comparisons task.

\subsection{Proof of \cref{prop:theory_graph}}
\begin{proposition*}
    Let $ \mathcal{X}=\{v_1 , \dots, v_n\}$ be a one-dimensional Path or Cycle graph (see Figure \ref{fig:graph_gallery}) with $n$ vertices ($n \ge 3$) and $\mathcal{Z} \in \{\mathbb{R}, \mathbb{T}\}$ be a embedding space. 
    Let $\Phi : G \to \mathcal{Z}$ be embedding with partial distance information
    \begin{align*}
        d(\Phi(v_{i}) ,\Phi(v_{i+1}))&=1, \; i=1,\dots n-1, \\
        d(\Phi(v_{1}) ,\Phi(v_{n}))&=1, \; \text{ if $\mathcal{X}$ is Cycle. } 
    \end{align*}
    If $\Phi(\cdot)$ preserve the triangular symmetry, then $\Phi(\mathcal{X})$ has its graph structure in embedding space  $\mathbb{R}$.
\end{proposition*}

\begin{proof}
{\textbf{Case 1:} $\mathcal{X}$ is Path.} \\

{\textbf{Case 1-1:} $\mathcal{Z}=\mathbb{R}$.} \\
    First suppose $n=3$ and $\mathcal{X}=\mathbb{R}$ with the standard metric. Let $v_1,v_2,v_3$ be the vertices of $\mathcal{X}$. 
    Consider the embedding $\Phi : \mathcal{X} \to \mathbb{R}$. Without loss of generality, let $\Phi(v_1)=0$.
    Since $\Phi(\cdot)$ preserves the triangular symmetry, we have $d(v_1,v_3)=d(v_1,v_2)+d(v_2,v_3)=2$.  
    Then $\Phi(v_2)=\pm 1$ because $d(\Phi(v_1),\Phi(v_2))=1$.  Without loss of generality, let $\Phi(v_1)=1$. Then we have  $\Phi(v_3) = 2$ because $d(\Phi(v_2),\Phi(v_3))=1, d(\Phi(v_1),\Phi(v_3))=2$. 
    Therefore, $v_1,v_2,v_3$ should be aligned in a order. \\
    For general $n \ge 4$, by applying the mathematical induction, we achieve the desired result. \\

{\textbf{Case 1-2:} $\mathcal{Z}=\mathbb{T}$.} \\
Let $\mathbb{T}=\{ e^{ 2 \pi i \theta / K }   : \theta \in \mathbb{R}\}$  with metric $d(e^{i\theta_1}, e^{i\theta_2} ) =  \min_{m \in \mathbb{Z}}| (\theta_1- \theta_2) - Km |$ for some large $K \gg n $. 
First suppose $n=3$. Let $v_1,v_2,v_3$ be the vertices of $\mathcal{X}$. 
Consider the embedding $\Phi : G \to \mathbb{T}$, and Let $\Phi(v_i)=e^{2 \pi i \theta_i/K} \in \mathbb{T}$.   Without loss of generality, let $\theta_1=0$.   Since $\Phi(\cdot)$ preserves the triangular symmetry, we have $d(v_1,v_3)=d(v_1,v_2)+d(v_2,v_3)=2$.

Then $\theta_2=\pm 1$ because $d(\Phi(v_1),\Phi(v_2))=1$.  Without loss of generality, let $\theta_2=1$. Then we have $\theta_3 = 2$ because $d(\Phi(v_2),\Phi(v_3))=1, d(\Phi(v_1),\Phi(v_3))=2$. Therefore, $v_1,v_2,v_3$ should be aligned in a order. For general $n \ge 4$, by applying the mathematical induction, we achieve the desired result.  
We also note that if $\mathbb{T}$ has different metric like pullback metric as
\begin{align*}
    d_k(e^{i\theta_1}, e^{i\theta_2} ) = \min_{m \in \mathbb{Z}} |  k (\theta_1 -\theta_2) - K m |, \quad \frac{ k}{K} \in \mathbb{Z}
\end{align*}
for some $k \in \mathbb{R}$, the we have different embeddings. 

{\textbf{Case 2:} $\mathcal{X}$ is Cycle.} \\
In this case, $\mathcal{Z}=\mathbb{R}$ is impossible. If  $\mathcal{Z}=\mathbb{T}$, let $\mathbb{T}=\{ e^{ 2 \pi i \theta / K }   : \theta \in \mathbb{R}\}$  with metric $d(e^{i\theta_1}, e^{i\theta_2} ) =  \min_{m \in \mathbb{Z}}| (\theta_1- \theta_2) - Km |$ for some large $K \gg n $. In Cycle graph, note that $d(v_1,v_3)=1$. Let $\Phi(v_i)=e^{2 \pi i \theta_i/K} \in \mathbb{T}$. 
Without loss of generality, assume $\theta_1=0$ and $\theta_2=1$. Then we have $\theta_2= -1$ from $d(v_1,v_3)=1$ and $\theta_2=2$ from $d(v_2,v_3)=1$. Hence $e^{2 \pi i (-1)/K}=e^{2 \pi i (2)/K}$ leading to $e^{2 \pi i (3)/K}=1$, $K \equiv 0  \text{ (mod 3)} $. Hence $\Phi(v_1), \Phi(v_2), \Phi(v_3)$ has cyclic shape in $\mathbb{T}$. For general $n \ge 4$, by applying the mathematical induction, we achieve the desired result. 

We also note that if $\mathbb{T}$ has different metric like pullback metric as
\begin{align*}
    d_k(e^{i\theta_1}, e^{i\theta_2} ) = \min_{m \in \mathbb{Z}} |  k (\theta_1 -\theta_2) - K m |, \quad \frac{ k}{K} \in \mathbb{Z}
\end{align*}
for some $k \in \mathbb{R}$, the we have different embeddings.

\end{proof}
\subsection{Theoretical observation on comparison task}
\subsection{Proof of \cref{prop:theory_comparison}}
\begin{proposition*}
    Let $\mathcal{X}=\{v_1 , \dots, v_n\}$ be a set of nodes in 1D comparison task (see \cref{subsec:comparision_task})  with $n \ge 3$ nodes. Let $\Phi : \mathcal{X} \to \mathbb{R}$ be embedding with partial relation information
    \[ \Phi(v_{i}) < \Phi(v_{i+1}), \; i=1,\dots n-1. \]
    If $\Phi(\cdot)$ preserve the transitivity, then $\Phi(\mathcal{X})$ has a grid structure in embedding space $\mathbb{R}$. 
\end{proposition*}
\begin{proof}
    Suppose $n=3$. Let $v_1,v_2,v_3$ be nodes of $\mathcal{X}$ with $v_1<v_2<v_3$. Consider the embedding $\Phi : \mathcal{X} \to \mathbb{R}$. Without loss of generality, let $\Phi(v_1)=0$. Since $\Phi(\cdot)$ preserves the transitivity, we have $\Phi(v_1)<\Phi(v_3)$ from $\Phi(v_1)<\Phi(v_2)$ and $\Phi(v_2)<\Phi(v_3)$. Hence we have $\Phi(v_1)<\Phi(v_2)<\Phi(v_3)$.
    For general $n \ge 4$, by applying the mathematical induction, we achieve the desired result.
\end{proof}